\documentclass[12pt]{article}

\usepackage{macros}

\makeatletter
\def\blfootnote{\gdef\@thefnmark{}\@footnotetext}
\makeatother

\begin{document}

\title{Robust Estimation under the Wasserstein Distance}
\author{Sloan Nietert\thanks{Cornell University}, Rachel Cummings\thanks{Columbia University}, and Ziv Goldfeld\footnotemark[1]}

\blfootnote{A preliminary version of this work appeared at AISTATS 2022 under the title ``Outlier-Robust Optimal Transport: Duality, Structure, and Statistical Analysis.'' The current text presents new results for robust distribution estimation and a more general duality theorem, among other refinements.}

\maketitle

\maketitle

\begin{abstract}%
We study the problem of robust distribution estimation under the Wasserstein distance, a popular discrepancy measure between probability distributions rooted in optimal transport (OT) theory. Given $n$ samples from an unknown distribution $\mu$, of which $\eps n$ are adversarially corrupted, we seek an estimate for $\mu$ with minimal Wasserstein error. To address this task, we draw upon two frameworks from OT and robust statistics: partial OT (POT) and minimum distance estimation (MDE). We prove new structural properties for POT and use them to show that MDE under a partial Wasserstein distance achieves the minimax-optimal robust estimation risk in many settings. Along the way, we derive a novel dual form for POT that adds a sup-norm penalty to the classic Kantorovich dual for standard OT. Since the popular Wasserstein generative adversarial network (WGAN) framework implements Wasserstein MDE via Kantorovich duality, our penalized dual enables large-scale generative modeling with contaminated datasets via an elementary modification to WGAN. Numerical experiments demonstrating the efficacy of our approach in mitigating the impact of adversarial corruptions are provided.

\end{abstract}
\vspace{-1mm}
\section{Introduction}

Given i.i.d.\ samples $X_1, \dots, X_n$ from an unknown probability measure $\mu$, the empirical measure $\hat{\mu}_n = \frac{1}{n}\sum_{i=1}^n \delta_{X_i}$ is a simple and ubiquitous estimator for $\mu$. However, real-world data is often subject to measurement errors, model misspecification, and even malicious outliers. In this work, we seek to recover $\mu$ from corrupted samples $\tilde{X}_1, \dots, \tilde{X}_n$ under the strong $\eps$-contamination model, namely, where an $\eps$-fraction of the clean samples may be arbitrarily modified. This model is equivalently described by the total variation (TV) constraint $\|\tilde{\mu}_n - \hat{mu}_n\|_\tv \leq \eps$, where $\tilde{\mu}_n = \frac{1}{n}\sum_{i=1}^n \delta_{\tilde{X}_i}$ is empirical measure of the corrupted data. %
Working within a metric space $(\cX,\mathsf{d})$, we measure the quality of an estimate $\nu$ for $\mu$ using OT. Specifically, we use the $p$-Wasserstein distance $\Wp(\mu,\nu)$, defined via
\vspace{-2mm}
\begin{equation}
\label{eq:Wp-intro}
\Wp(\mu,\nu)\coloneqq\left(\inf_{\pi\in\Pi(\mu,\nu)}\int_{\cX\times\cX}\mathsf{d}(x,y)^p\dd\pi(x,y)\right)^\frac{1}{p} = \inf_{\pi \in \Pi(\mu,\nu)} \|\mathsf{d}\|_{L^p(\pi)},
\vspace{-1mm}
\end{equation}
where $\Pi(\mu,\nu)$ is the set of couplings for $\mu$ and $\nu$. This is a natural choice for non-parametric distribution estimation, since $\Wp$ metrizes the weak topology and tolerates support mismatch \citep{boissard2015wasserstein,fournier2015rate,singh2018minimax,lei2020convergence}. Unfortunately, due to the strict marginal constraints in \eqref{eq:Wp-intro}, OT values can be arbitrarily perturbed via small TV perturbations of the input measures; in particular, $\Wp(\mu,(1-\eps)\mu + \eps \delta_{x}) \to \infty$ as $\mathsf{d}(x,x_0) \to \infty$, for any $x_0 \in \cX$.

\subsection{Our Contributions}
To address the sensitivity of $\Wp$ to outliers, we consider the POT problem
\begin{equation}
\label{eq:POT-intro}
    \RWp(\mu,\nu) \!\defeq\! \inf_{\pi \in \Pi_\eps(\mu,\nu)} \!\|\mathsf{d}\|_{L^p(\pi)}, \:\: \Pi_\eps(\mu,\nu)\!\defeq\!\left\{\pi \!\in\! \cM_\plus(\cX^2)\,\middle|\, \begin{aligned}&\pi(\cdot\!\times\!\cX)\leq \mu,\pi(\cX\!\times\! \cdot)\leq \nu\\[-0.25em]&\pi(\cX\!\times\! \cX)=1-\eps\end{aligned}\right\},
\end{equation}
where the inequalities are understood setwise.\footnote{Setwise inequalities between measures reduce to pointwise inequalities between densities when they exist.} 
Equipped with this robust Wasserstein distance,\footnote{Despite calling $\RWp$ a distance, note that it is no longer a proper metric, e.g., $\RWp(\mu,\nu)=0$ for $\mu\neq\nu$ so long as $\|\mu - \nu\|_\tv \leq \eps$.} we employ minimum distance estimation (MDE) under $\RWp$, seeking an estimator $\mathsf{T}(\tilde{\mu}_n)$ such that 
\begin{equation*}
    \mathsf{T}(\tilde{\mu}_n) \in \argmin_{\nu \in \cG} \RWp(\nu,\tilde{\mu}_n),
\end{equation*}
where $\cG$ is a distribution family known to contain $\mu$ (generally encoding tail bounds). %
We prove that this procedure achieves the minimax optimal estimation error under $\Wp$ in the population limit, i.e., as $n \to \infty$. In the finite-sample setting, our bounds are shown to be optimal up to the sub-optimality of $\hat{\mu}_n$ for distribution learning without corruptions, and this gap is negligible for the main distribution families considered.\footnote{In an independent manuscript (submitted after the conference version of this work was published and concurrently with the extended version), \cite{ma2023inference} explore MDE under the robust Wasserstein distance of \cite{mukherjee2021}. Their approach leads to a distinct dual, with a sup-norm constraint compared to our sup-norm penalty, and they show convergence of their finite-sample MDE towards MDE w.r.t.\ the population distribution. However, they do not provide robust estimation guarantees for their estimator under the Wasserstein distance, which is the focus of this work.}

Our analysis hinges on several new results for POT, which may be of independent interest. This includes an approximate triangle inequality and the following equivalent primal formulation,
\begin{equation}
\label{eq:TV-robustified-Wp}
    \RWp(\mu,\nu) = \inf_{\substack{\mu':\|\mu' - \mu\|_\tv \leq \eps}} \Wp(\mu',\nu),
\end{equation}
directly connecting $\RWp$ to the robust estimation task at hand. Moreover, this reformulation enables deriving a penalized dual problem which, when $p=1$, simplifies to (see \cref{thm:RWp-dual} for the general case):
\begin{equation*}
    \RWone(\mu,\nu) = \sup_{\varphi: \|\varphi\|_{\Lip} \leq 1} \E_\mu[\varphi] - \E_\nu[\varphi] - 2\eps\|\varphi\|_\infty.
\end{equation*}
When $\eps = 0$, this recovers the classic Kantorovich $\Wone$ dual that underlies the popular WGAN approach to generative modeling \citep{arjovsky_wgan_2017}. In particular, WGAN is a min-max implementation of MDE under $\Wone$ that features the Kantorovich dual objective. By adding our simple $\ell_\infty$ penalty, we arrive at a practical implementation of MDE under $\RWone$ which enables robust generative modeling. We provide numerical experiments on image data highlighting the potency of the proposed approach for mitigating the effect of outliers in generative modeling from contaminated data. Comparisons to the robust WGANs from \cite{balaji2020} and \cite{staerman21} are also provided, demonstrating the superiority of our method. 

The rest of the paper is organized as follows. \cref{sec:prelims} opens with a discussion of preliminaries, followed by a summary of our structural results for POT in \cref{sec:RWp}. \cref{sec:robustness} provides robust estimation guarantees for MDE under $\RWp$, as well as for $\RWp$ itself as an estimate of $\Wp$. \cref{sec:experiments} applies our duality theory to build a robust WGAN and provides empirical results stemming from this approach, as well as comparisons to competing methods. %

\subsection{Related Work}
\label{subsec:related_work}

\cite{balaji2020} define a robust OT problem mirroring \eqref{eq:TV-robustified-Wp} but with general $f$-divergences instead TV, primarily focusing on the $\chi^2$ divergence. They arrive at a more complex dual form (derived by invoking standard Kantorovich duality on the Wasserstein distance inside the infimum) that requires optimization over a significantly larger domain. 
In \cite{mukherjee2021}, robustness w.r.t.\ the TV distance is integrated via a regularization term in the objective. This leads to a simple modified primal problem but the corresponding dual requires optimizing over two potentials, even when $p=1$. Additionally, \cite{le2021} and \cite{nath2020} consider robustness via Kullback-Leibler (KL) divergence and integral probability metric regularization terms, respectively. The former focuses on Sinkhorn-based algorithms, while the latter introduces a dual form that is distinct from ours and less compatible with existing duality-based OT computational methods. \cite{staerman21} present a median-of-means adaptation of the dual Kantorovich problem to tackle robust estimation when the contamination fraction $\eps$ vanishes as $n \to \infty$.

The robust OT literature is intimately related to unbalanced OT theory, which addresses transport problems between measures of different mass \citep{piccoli2014, chizat2018, liero2018, schmitzer2019, hanin1992}. These formulations are reminiscent of the problem \eqref{eq:TV-robustified-Wp} but with regularizers added to the objective (KL being the most studied) rather than incorporated as constraints. Sinkhorn-based primal algorithms \citep{chizat2018scaling} are the standard approach to computation, and these have been extended to large-scale machine learning via minibatch methods \citep{fatras21a}. 
\cite{sonthalia2020dual} explore a distinct version of unbalanced OT derived via dual regularization, but the regularized dual involves two potentials even when $p=1$.
Unbalanced OT is sometimes called relaxed or semi-relaxed OT, depending on whether one or both marginal constraints are relaxed \citep{blondel2018smooth, fukunaga2021}. The former work \citep{blondel2018smooth} features a regularized dual, but the regularizer penalizes violations of the standard Kantorovich dual constraint, which is not relaxed in our approach.

The POT problem \eqref{eq:POT-intro} itself has also been previously studied. \cite{caffarelli2010} consider a different parameterization of the problem, arriving at a distinct dual. In particular, even when $p=1$, their dual involves two potential functions. Their work reduces the POT problem to an OT problem with an augmented cost, as discussed in Section~\ref{ssec:POT}. \cite{figalli2010} restricts attention to quadratic costs with no discussion of duality. \cite{alvarez2011uniqueness} find sufficient conditions for the optimal partial coupling to be unique and induced by a deterministic transport map. \citet{del2013rates} study statistical convergence rates under POT, generally finding them to be comparable to standard OT. More recently, \citet{chapel2020} applied POT to positive-unlabeled learning, but dual-based algorithms were not considered.

There is a long history of learning in the presence of corruptions, dating back to the robust statistics community in the 1960s  \citep{huber64}. Over the years, many robust and sample-efficient estimators have been designed, particularly for mean and scale parameters; see, e.g., \cite{ronchetti2009robust} for a comprehensive survey and \cite{donoho88} for an influential population-limit approach based on MDE. In the theoretical computer science community, much work has focused on achieving optimal rates for robust mean and moment estimation with computationally efficient estimators \citep{cheng2019high, diakonikolas2019recent}. Recent statistics work, building off of \cite{donoho88}, has developed a unified framework for robust estimation based on MDE, which generally achieves optimal population-limit and near-optimal finite-sample guarantees for tasks including mean and covariance estimation \citep{zhu2019resilience}. Their analysis employs a quantity dubbed ``generalized resilience'' which is also essential to our work (once adapted to the OT setting).
Our MDE approach follows their ``weaken-the-distance'' principle for robust statistics, but with a distinct notion of ``weakening'' that lends itself better to OT. \citet{zhu2019resilience} also consider robust estimation where perturbations are bounded under $\Wone$ instead of $\|\cdot\|_\tv$, but they do not consider a $\Wone$ objective. 
In follow-up work by a subset of the authors \citep{nietert2024distribution}, spectral methods from \citet{diakonikolas2016} are adapted to obtain a computationally efficient and minimax optimal procedure for estimation under $\Wone$ when $\mu$ has bounded covariance, under a stronger adversarial model. However, this approach does not translate to the WGAN setting motivating our work.

\vspace{-3mm}
\section{Preliminaries}
\label{sec:prelims}

Let $(\cX,\mathsf{d})$ be a complete, separable metric space, and denote the diameter of a set $A \subset \cX$ by $\diam(A) \coloneqq \sup_{x,y \in A} \mathsf{d}(x,y)$. Write $\Lip_1(\cX)$ for the family of real 1-Lipschitz functions on $\cX$. Take $C_b(\cX)$ as the set of continuous, bounded real functions on $\cX$, and let $\cM(\cX)$ denote the set of signed Radon measures on $\cX$ equipped with the TV norm $\|\mu\|_\tv \coloneqq \frac{1}{2}|\mu|(\cX)$. Let $\cM_\plus(\cX)$ denote the space of finite, positive Radon measures on $\cX$. For $\mu,\nu \in \cM_\plus(\cX)$ and $p \in [1,\infty]$, we consider the standard $L^p(\mu)$ space with norm $\|f\|_{L^p(\mu)} = \left(\int |f|^p \dd \mu \right)^{1/p}$ and write $\mu \leq \nu$ when $\mu(B) \leq \nu(B)$ for every Borel set $B \subseteq \cX$. Denote the shared mass of $\mu$ and $\nu$ by $\mu \land \nu \coloneqq \mu - (\nu - \mu)^-$, where $\kappa = \kappa^+ - \kappa^-$ is the Jordan decomposition of $\kappa \in \cM(\cX)$, and write $\mu \lor \nu \defeq \mu + \nu - \mu \land \nu$. For $A \subseteq \cX$ and $\mu \in \cM_\plus(\cX)$, define $\mu|_A = \mu(\cdot \cap A)$.

Let $\cP(\cX) \subset M_\plus(\cX)$ denote the space of probability measures on $\cX$, and take $\cP_p(\cX) \coloneqq \{ \mu \in \cP(\cX) : \int \mathsf{d}(x,x_0)^p \dd \mu(x) < \infty \text{ for some $x_0 \in \cX$}\}$ to be those with bounded $p$th moment. 
Given $\mu,\nu \in \cP(\cX)$, let $\Pi(\mu,\nu)$ denote the set of their couplings, i.e., $\pi \in \cP(\cX \times \cX)$ such that $\pi(B \times \cX) = \mu(B)$ and $\pi(\cX \times B) = \nu(B)$, for every Borel set $B$. When $\cX = \R^d$, we write the covariance matrix for $\mu \in \cP_2(\cX)$ as $\Sigma_\mu \coloneqq \E[(X - \E[X])(X - \E[X])^\top]$ where $X \sim \mu$. We also use $\E_\mu[f(X)]$ to denote the expectation of the random variable $f(X)$ when $X \sim \mu$
and write $f_\sharp \mu$ for the probability law of $f(X)$. We write $a \lor b \coloneqq \max\{a,b\}$, $a \land b \coloneqq \min \{a,b\}$, and use $\lesssim, \gtrsim, \asymp$ to denote inequalities/equality up to absolute constants. We sometimes write $a = O(b)$ and $b = \Omega(a)$ for $a \lesssim b$.
\vspace{-3mm}

\subsection{(Partial) Optimal Transport}
\label{ssec:POT}
For a lower semi-continuous cost function $c:\cX^2 \to \R$, the \emph{optimal transport cost} between $\mu,\nu \in \cM_\plus(\cX)$ with $\mu(\cX) = \nu(\cX)$ is defined by
\begin{equation}
\label{eq:OT-def}
    \mathsf{OT}_c(\mu,\nu) \coloneqq \inf_{\pi \in \Pi(\mu,\nu)} \int_{\cX^2} c \dd \pi,
\end{equation}
where $\Pi(\mu,\nu) \coloneqq \{ \pi \in \cM_\plus(\cX^2) : \pi(\cdot \times \cX) = \mu, \pi(\cX \times \cdot) = \nu \}$.
For $p \in [1,\infty)$, the $p$-Wasserstein distance $\Wp(\mu,\nu)$ as defined in \eqref{eq:Wp-intro} coincides with $\mathsf{OT}_{c}(\mu,\nu)^{1/p}$ for the cost $c(x,y) \coloneqq \mathsf{d}(x,y)^p$.
Some basic properties of $\Wp$ are (see, e.g., \citealp{villani2003,santambrogio2015}): (i) the $\inf$ is attained in the definition of $\Wp$, i.e., there exists a coupling $\pi^\star \in \Pi(\mu,\nu)$ such that $\Wp^{p}(\mu,\nu) = \int_{\R^d \times \R^d} |x-y|^pd\pi^\star(x,y)$;
(ii) $\Wp$ is a metric on $\cP_p(\cX)$; and
(iii) convergence in $\Wp$ is equivalent to weak convergence plus convergence of $p$th moments: $\Wp(\mu_{n},\mu) \to 0$ if and only if $\mu_{n} \stackrel{w}{\to} \mu$ and $\int |x|^{p} d \mu_{n}(x) \to \int |x|^{p} d \mu(x)$. 

Wasserstein distances adhere to a dual formulation, which we summarize below. For a function $\varphi: \cX \to [-\infty,\infty)$ and symmetric cost function $c: \cX^2 \to \R$, the \textit{$c$-transform} of $\varphi$ is
\[
\varphi^c(y) \coloneqq \inf_{x \in \cX}\big [ c(x,y) -\varphi(x) \big], \quad y \in \R^d.
\]
A function $\varphi: \cX \to [-\infty,\infty)$, not identically $-\infty$, is called \textit{$c$-concave} if $\varphi = \psi^c$ for some function $\psi: \cX \to [-\infty,\infty)$. For the $\Wone$ cost $c(x,y) = \mathsf{d}(x,y)$, this occurs if and only if $\varphi \in \mathrm{Lip}_1(\cX)$. With this notation, the following duality holds \citep[Theorem 1.3]{villani2003}, 
\begin{equation}
\mathsf{OT}_c(\mu,\nu) = \!\!\sup_{\substack{\varphi,\psi \in C_b(\cX)\\\varphi(x) + \psi(y) \leq c(x,y) \, \forall x,y \in \cX}} \!\left [ \int_{\cX} \!\varphi \dd\mu + \int_{\cX} \!\psi \dd\nu \right ]= \sup_{\varphi \in C_b(\cX)} \left [ \int_{\cX} \varphi \dd\mu + \int_{\cX} \varphi^c \dd\nu \right ],\label{eq:Wp-duality}
\end{equation}
and one may restrict the right-hand supremum to $c$-concave $\varphi$. \smallskip

We define the partial OT cost $\OT^\eps(\mu,\nu)$ as the solution to \eqref{eq:OT-def} when the set of feasible couplings $\Pi(\mu,\nu)$ is replaced with the set of partial couplings $\Pi_\eps(\mu,\nu)$, as defined in \eqref{eq:POT-intro}. We recover $\RWp$ by taking $c = \mathsf{d}^p$. Computation for POT is generally addressed via the following characterization of POT as an instance of standard OT over an augmented space.

\begin{proposition}[POT as augmented standard OT, \citealp{caffarelli2010}]
\label{prop:RWp-augmented-Wp}
Let $\bar{x}\notin\cX$ be a dummy point, define the disjoint union $\bar{\cX}=\cX \sqcup \{\bar{x}\}$, and suppose that the cost $c$ is non-negative. Define the augmented cost function $\bar{c}:\bar{\cX}^2 \to [0,\infty]$ by
\begin{equation*}
    \bar{c}(x,y) = \begin{cases}
        c(x,y), &x,y \in \cX\\
        0, &x \in \cX, y = \bar{x}\ \mbox{or}\ y \in \cX, x = \bar{x}\\
        \infty. &x = y = \bar{x}
    \end{cases}.
\end{equation*}
Then, for all $\mu,\nu \in \cP(\cX)$, we have $\OT^{\eps}(\mu,\nu) = \mathsf{OT}_{\bar{c}}(\mu + \eps \delta_{\bar{x}}, \nu + \eps \delta_{\bar{x}})$.
\end{proposition}

One may thus compute $\RWp$ by applying methods for standard OT to this augmented space, e.g., linear programming for an exact solution or Sinkhorn methods for an approximate solution (see \cite{peyre2019} for a thorough discussion of computational OT)
. In these cases, $\bar{c}(\bar{x},\bar{x})$ can be taken to be a sufficiently large but finite constant.

\subsection{Minimum Distance Estimation}

Given a distribution family $\cG \subseteq \cP(\cX)$ and a statistical distance $\mathsf{D}:\cP(\cX)^2 \to [0,\infty]$, we define the minimum distance functional $\mathsf{T}_{[\cG,\mathsf{D}]} : \cP(\cX) \to \cP(\cX)$ by the relation
\begin{equation*}
\mathsf{T}_{[\cG,\mathsf{D}]}(\mu) \in \argmin_{\nu \in \cG} \mathsf{D}(\mu,\nu),
\end{equation*}
choosing the representative arbitrarily and assuming no issues with existence. More generally, for any $\delta > 0$, we set $\mathsf{T}_{[\cG,\mathsf{D}]}^\delta(\mu)$ to any distribution $\kappa \in \cG$ such that $\mathsf{D}(\mu,\kappa) \leq \inf_{\nu \in \cG} \mathsf{D}(\mu,\nu) + \delta$.
One can view such $\nu$ as a $\delta$-approximate projection of $\mu$ onto the family $\cG$ under $\mathsf{D}$. We recall next how MDE can be applied in the contexts of OT and robust statistics.

\vspace{-1mm}
\paragraph{WGAN as MDE under $\Wone$.}

A remarkable application of OT to machine learning is generative modeling via the WGAN framework, which can be understood as MDE under $\Wone$. Given a distribution family $\cG$ and the $n$-sample empirical measure $\hat{\mu}_n$ of an unknown distribution $\mu$, the $\Wone$ minimum distance estimate for $\mu$ given $\hat{\mu}_n$ is
\[
    \mathsf{T}_{[\cG,\Wone]}(\hat{\mu}_n) = \argmin_{\nu \in \cG} \Wone(\hat{\mu}_n,\nu)    = \argmin_{\nu \in \cG} \sup_{f \in \mathrm{Lip}_1(\cX)} \int_\cX f \dd(\hat{\mu}_n - \nu).
\]%
In practice, $\cG$ is realized as the pushforward of a simple reference distribution (e.g., standard normal or uniform on a ball) through a parameterized family of neural networks. The resulting estimate is called a generative model, since a minimizing parameter allows one to generate high quality samples from a distribution close to $\mu$, given that $\cG$ is a sufficiently rich family. The min-max objective for this problem lends itself well to implementation via alternating gradient-based optimization, and the resulting WGAN framework has supported breakthrough work in generative modeling \citep{arjovsky_wgan_2017, gulrajani2017improved, tolstikhin2018wasserstein}. MDE under the Wasserstein metric has also been of general interest in the statistics community due to its generalization guarantees and lack of parametric assumptions \citep{bassetti2006kantorovich,bernton2017inference,boissard2015wasserstein}.

\paragraph{Robust statistics and MDE under $\|\cdot\|_\tv$.} 
In robust statistics, one seeks to estimate some property of a distribution $\mu$ from contaminated data, generally under the requirement that $\mu \in \cG$ for some class $\cG \subseteq \cP(\cX)$ encoding distributional assumptions. A common approach for quantifying the error is using a statistical distance $\mathsf{D}:\cP(\cX)^2 \to [0,\infty]$. 
In the population-limit version of such tasks, sampling issues are ignored, and we suppose that one observes a contaminated distribution $\tilde{\mu}$ such that $\|\tilde{\mu} - \mu\|_\tv \leq \eps$. The information-theoretic complexity of robust estimation in this setting is characterized by the \emph{modulus of continuity}
\begin{equation*}
    \mathfrak{m}_\mathsf{D}(\cG,\|\cdot\|_\tv,\eps) \coloneqq \sup_{\substack{\mu,\nu \in \cG \\ \|\mu - \nu\|_\tv \leq \eps}} \mathsf{D}(\mu,\nu),
\end{equation*}
and the associated guarantees are achieved via MDE under the TV norm. This is summarized in the following lemma due to \citet{donoho88}.

\begin{proposition}[Robustness of MDE, \citealp{donoho88}]
Fix $\cG \subseteq \cP(\cX)$, let $\mu \in \cG$, and take any $\tilde{\mu} \in \cP(\cX)$ with $\|\tilde{\mu} - \mu\|_\tv \leq \eps$. Then, for any $\nu \in \cG$ such that $\|\nu - \tilde{\mu}\|_\tv \leq \eps$, including the minimum distance estimate $\mathsf{T}_{[\cG,\|\cdot\|_\tv]}(\tilde{\mu})$, we have
\begin{equation}
\label{eq:modulus-risk-bound}
    \mathsf{D}(\nu,\mu) \leq \mathfrak{m}_\mathsf{D}(\cG,\|\cdot\|_\tv,2\eps).
\end{equation}
\end{proposition}

The MDE solution satisfies $\|\mathsf{T}_{[\cG,\|\cdot\|_\tv]}(\tilde{\mu}) - \tilde{\mu}\|_\tv \leq \eps$ because $\|\mu - \tilde{\mu}\|_\tv \leq \eps$ and $\mu \in \cG$, while inequality \eqref{eq:modulus-risk-bound} holds by the definition of the modulus since $\mu,\nu \in \cG$ and $\|\mu - \nu\|_\tv \leq \|\mu - \tilde{\mu}\|_\tv + \|\tilde{\mu} - \nu\|_\tv \leq 2\eps$. In typical cases of interest, it is easy to show that no procedure can obtain estimation error substantially less than $\mathfrak{m}_\mathsf{D}(\cG,\|\cdot\|_\tv,\eps)$ for all $\mu \in \cG$, so this characterization is essentially tight.

\paragraph{Generalized resilience.}
Our robustness analysis for OT distances relies on the notion of (generalized) resilience, which was originally proposed for robust mean estimation \citep{steinhardt2018resilience, zhu2019resilience}. Specifically, given a statistical distance $\mathsf{D}:\cP(\cX)^2 \to [0,\infty]$, we say that $\mu \in \cP(\cX)$ is \emph{$(\rho,\eps)$-resilient w.r.t.\ $\mathsf{D}$} if $\mathsf{D}(\mu,\nu) \leq \rho$ for all distributions $\nu$ with $\nu \leq \frac{1}{1-\eps} \mu$. Resilience is a primary sufficient condition for robust estimation, since it implies immediate bounds on the relevant modulus of continuity.

\begin{lemma}[Bounded modulus under resilience, \citealp{zhu2019resilience}]
\label{lem:bounded-modulus-generalized-resilience}
Let $\mathsf{D}:\cP(\cX)^2 \to [0,\infty]$ satisfy the triangle inequality, and write $\cG$ for the family distributions which are $(\rho,2\eps)$-resilient w.r.t. $\mathsf{D}$. Then we have
\begin{equation*}
    \mathfrak{m}_\mathsf{D}(\cG,\|\cdot\|_\tv,\eps) \leq 2\rho.
\end{equation*}
\end{lemma}

\begin{example}[Robust mean estimation]
In robust mean estimation, the error is quantified via $\mathsf{D}_\mathrm{mean}(\mu,\nu) = \|\E_\mu[X] - \E_\nu[X]\|$, i.e., we seek an estimate for the mean of an unknown distribution under the $\ell_2$ norm. A distribution $\mu$ is $(\rho,\eps)$-resilient w.r.t.\ $\mathsf{D}_\mathrm{mean}$ %
if $\|\E_\mu[X] - \E_\nu[X]\| \leq \rho$ for all distributions $\nu \leq \frac{1}{1-\eps}\mu$. For example, if $\mu$ has variance $\sigma$, one can show that $\mu$ is $\bigl(O(\sigma\sqrt{\eps}),2\eps\bigr)$-resilient in mean; thus its mean can be recovered under TV $\eps$-corruptions up to error $O(\sigma\sqrt{\eps})$. The MDE procedure achieving this risk bound in the population-limit is closely related to efficient finite-sample algorithms based on iterative filtering and spectral reweighing \citep{diakonikolas2016, hopkins2020robust}.
\end{example}

\section{Structure of POT}%
\label{sec:RWp}

Our approach to robust estimation builds upon the robust Wasserstein distance $\RWp$, defined by a POT problem where a $(1-\eps)$-fraction of mass is transported. Here, we establish several important properties and equivalences for POT, which may be of independent interest. Throughout, we fix a cost $c$ of the form $c(x,y) = h(d(x,y))$, where $h:\R_{\geq 0} \to \R_{\geq 0}$ is continuous and strictly increasing with $h(0) = 0$. In particular, this captures the $p$-Wasserstein cost $c = \mathsf{d}^p$. Proofs for this section are deferred to Section~\ref{prf:RWp}.
\smallskip

To begin, we provide three alternative reformulations of $\ROT$, which was originally defined in Section~\ref{ssec:POT} as a minimization problem over partial couplings. These reformulations more evidently connect POT to outlier robustness, and each of them will be used throughout our theory. After stating the result, we discuss and interpret each representation.

\begin{proposition}[Equivalent reformulations]
\label{prop:alternative-primal-probs}
For $\mu,\nu \in \cP(\cX)$ and $\eps \in [0,1]$, we have
\begin{equation*}
    \ROT(\mu,\nu) \!=\!\! \min_{\substack{\mu_\minusb,\nu_\minusb \in \cM_\plusb(\cX)\\ \mu_\minusb \leq \mu,\, \nu_\minusb \leq \nu\\ \mu_\minusb(\cX) = \nu_\minusb(\cX) \geq 1-\eps}}\!\!\!\! \OT(\mu_\minus,\nu_\minus)\!=\!\min_{\substack{\mu_\plusb,\nu_\plusb \in \cM_\plusb(\cX)\\\mu_\plusb \geq \mu,\, \nu_\plusb \geq \nu \\ \mu_\plusb(\cX) = \nu_\plusb(\cX) \leq 1+\eps}} \!\!\!\!\OT(\mu_\plus,\nu_\plus) \!=\! \min_{\substack{\mu' \in \cP(\cX)\\\|\mu' - \mu\|_\tv \leq \eps}} \!\!\OT(\mu',\nu).
\end{equation*}
For the first reformulation, there are minimizers $\mu_\minus \leq \mu$ and $\nu_\minus \leq \nu$ such that $\mu_\minus,\nu_\minus \geq \mu \land \nu$. For the second, there are maximizers $\mu_\plus \geq \mu$ and $\nu_\plus \geq \nu$ such that $\mu_\plus,\nu_\plus \leq \mu \lor \nu$.
\end{proposition}

\begin{figure}[t]
\begin{minipage}[c]{0.47\textwidth}
\centering
\scalebox{0.68}{\includegraphics{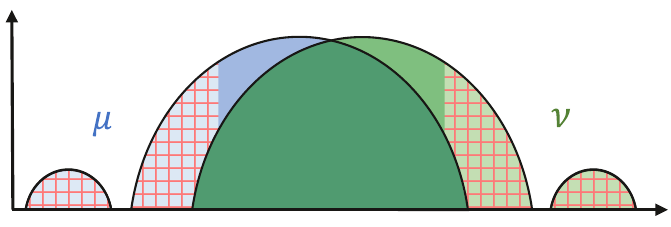}}
\end{minipage}\hfill
\begin{minipage}[c]{0.47\textwidth}
\caption{Visualization of Proposition~\ref{prop:alternative-primal-probs}. The gridded light blue and green regions each have mass $\eps$, respectively, and are removed to obtain optimal $\mu_\minus$ and $\nu_\minus$ for $\RWone$. No mass need be removed from the dark region designating $\mu \land \nu$.} \label{fig:primal-perturbations}
\end{minipage}
\end{figure}

The first formulation measures $\ROT$ as the smallest transport cost that can be attained by omitting an $\eps$-fraction of mass from each of $\mu$ and $\nu$; the removed mass can be interpreted as the identified outliers. 
The existence of minimizers bounded from below by $\mu \land \nu$, depicted in \cref{fig:primal-perturbations}, was previously known for $\RWone$ and $\Wtwo^\eps$ (see, e.g., \citealp{figalli2010}), but our result requires a more general argument. The second equality shows that $\ROT$ can be equivalently formulated in terms of mass addition, which is key to several of our proofs. This result relies on the geometric symmetry of the POT problem; instead of removing mass from one measure, we may equivalently add it to the other.
The third TV formulation follows since a TV perturbation can be split into separate mass addition and mass removal steps.\smallskip

We now establish basic structural properties of $\ROT$, identifying extreme values, noting its monotonicity in the robustness radius, and establishing an approximate triangle inequality.

\begin{proposition}[Basic properties]
\label{prop:RWp-structural-properties}
For $\mu,\nu,\kappa\in\cP(\cX)$ and $\eps_1,\eps_2 \in [0,1]$, we have
\begin{enumerate}
\item \emph{Monotonicity in the robustness parameter:} if $0 < \eps_1 \leq \eps_2 < \|\mu - \nu\|_\tv$, then
\begin{equation*}
0 < \OT^{\eps_2}(\mu,\nu) \leq \tfrac{1-\eps_2}{1-\eps_1} \OT^{\eps_1}(\mu,\nu) \leq \OT^{\eps_1}(\mu,\nu) < \infty,
\end{equation*}
with $\OT^0(\mu,\nu) = \OT(\mu,\nu)$ and $\OT^{\|\mu - \nu\|_\tv}(\mu,\nu) = 0$;
\item \emph{Approximate triangle inequality:} recalling that $c(x,y) = h(d(x,y))$, if $h$ and $\log \circ h \circ \exp$ are convex and $\eps_1 + \eps_2 \leq 1$, then
\begin{equation*}
    h^{-1}\bigl(\OT^{\eps_1 + \eps_2}(\mu,\nu)\bigr) \leq h^{-1}\bigl(\OT^{\eps_1}(\mu,\kappa)\bigr) + h^{-1}\bigl(\OT^{\eps_2}(\kappa,\nu)\bigr).
\end{equation*}
In particular, $\Wp^{\eps_1 + \eps_2}(\mu,\nu) \leq \Wp^{\eps_1}(\mu,\kappa) + \Wp^{\eps_2}(\kappa,\nu)$.
\end{enumerate}
\end{proposition}

The monotonicity statement is a simple observation, while the proof of the second claim combines Proposition~\ref{prop:alternative-primal-probs} and standard triangle inequalities for OT and $\|\cdot\|_{\tv}$.\smallskip

Although computation for POT is well-understood via a reduction to standard OT (see Proposition~\ref{prop:RWp-augmented-Wp}), this reduction uses a non-standard cost function, inhibiting some dual-based OT applications like WGAN. Fortunately, we prove that $\ROT$ admits a Kantorovich-type dual in terms of the base cost $c$.
This result is key to our experiments in \cref{sec:experiments}.

\vspace{-1mm}
\begin{restatable}[Dual form]{theorem}{RWpdual}
\label{thm:RWp-dual}
For $0<\eps\leq 1$ and $\mu,\nu \in \cP(\cX)$, we have 
\begin{align}
\vspace{-1mm}
\label{eq:RWp-dual}
    \ROT(\mu,\nu) &= \sup_{\substack{\varphi \in C_b(\cX)}} \int_\cX \varphi \dd \mu + \int_\cX \varphi^c \dd \nu -  2 \eps \|\varphi\|_\infty\\
    &= \sup_{\substack{\varphi \in C_b(\cX)}} \int_\cX \varphi \dd \mu + \int_\cX \varphi^c \dd \nu - \eps \left(\sup_{x \in \cX} \varphi(x) - \inf_{x \in \cX} \varphi(x)\right),\notag
\vspace{-1mm}
\end{align}
and the suprema are achieved by $\varphi \in C_b(\cX)$ such that $\varphi = (\varphi^c)^c$. If $\varphi\in C_b(\cX)$ maximizes \eqref{eq:RWp-dual} and $\mu_\minus,\nu_\minus\!\in\!\cM_\plus(\cX)$ are minimizers for the mass removal reformulation in Proposition~\ref{prop:alternative-primal-probs}, then $\supp(\mu-\mu_\minus) \subseteq \argmax(\varphi)$ and $\supp(\nu-\nu_\minus) \subseteq \argmin(\varphi)$.
\end{restatable}
\vspace{-1mm}

\begin{figure}[t]
\begin{center}

\begin{tikzpicture}

\definecolor{lightgreen}{HTML}{7FBF7F}
\definecolor{lightergreen}{HTML}{C5E0B4}
\definecolor{textgreen}{HTML}{548235}
\definecolor{lightblue}{HTML}{A1B8E1}
\definecolor{textblue}{HTML}{4472C4}
\definecolor{lighterblue}{HTML}{DEEBF7}
\definecolor{lightorange}{HTML}{FF9900}

\def\leftOutlierMass{(0.25,0) rectangle +(0.45,0.5) }
\def\rightOutlierMass{(5.65,0) rectangle +(0.45,0.5)}

\draw[fill=lighterblue] \leftOutlierMass node[above, xshift=-2.2mm, yshift=1mm, color=textblue]{\Large $\mu$};
\pattern[pattern=grid, pattern color=red] \leftOutlierMass;
\draw[fill=lightblue] (0.95,0) rectangle +(1.9,2.63);
\draw[fill=lightgreen] (3.5,0) rectangle +(1.9,2.63);
\draw[fill=lightergreen] \rightOutlierMass node[above, xshift=-2.2mm, yshift=1mm, color=textgreen]{\Large $\nu$};
\pattern[pattern=grid, pattern color=red] \rightOutlierMass;

\draw[-{Triangle[width=1.25mm, length=1.5mm]}, line width=1pt] (0,0) -- (6.5,0);
\draw[-{Triangle[width=1.25mm, length=1.5mm]}, line width=1pt] (0,0) -- (0,2.75);

\draw[color=lightorange, line width=4pt] (0.1,2.63) -- (0.93,2.63)  -- node[above, yshift=1.25mm]{\Large $\varphi$} (5.4,0) -- (6.27,0);

\draw (0,-0.17) -- (0,-0.17);

\end{tikzpicture}\:\:\:\:
\includegraphics[width=0.45\linewidth]{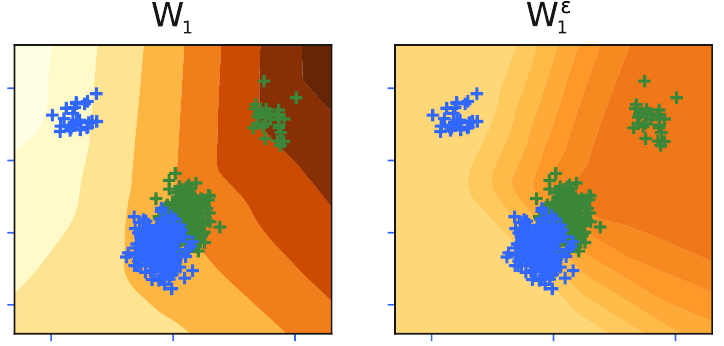}
\end{center}
\vspace{-5mm}
\caption{Optimal potentials: (left) 1D densities plotted with their optimal potential for the $\RWone$ dual problem; (right) contour plots for optimal dual potentials to $\Wone$ and $\RWone$ between 2D Gaussian mixtures. Observe how optimal potentials for the robust dual are flat over outlier mass.}\label{fig:optimal-potentials}
\end{figure}
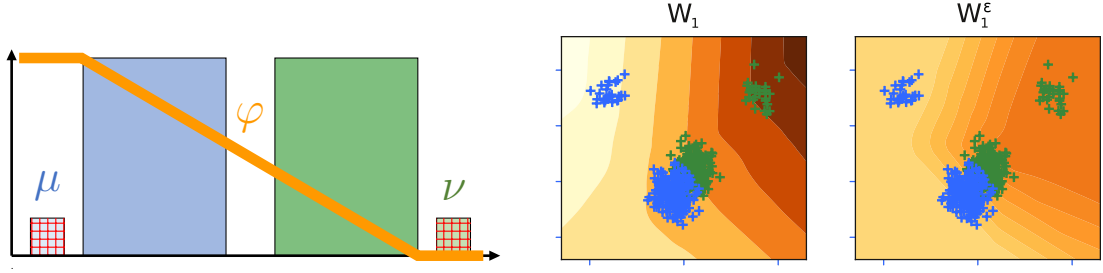

This new formulation differs from the classic dual \eqref{eq:Wp-duality} by a sup-norm penalty on the potential function, its lack of moment requirements on $\mu$ and $\nu$, and the existence of continuous, bounded maximizers. Recall that when $c = \mathsf{d}$, $\varphi = (\varphi^c)^c$ exactly when $\varphi$ is 1-Lipschitz, with $\varphi^c = -\varphi$. Moreover, the extremal level sets of any maximizing dual potential encode the location of mass removed in the primal problem, as depicted in \cref{fig:optimal-potentials}.\medskip

\begin{proof}[Proof sketch]
The ``$\geq$'' direction of \eqref{eq:RWp-dual} is straightforward to derive. Starting from the mass addition formulation of $\ROT$ (see Proposition~\ref{prop:alternative-primal-probs}), we bound
\begin{align*}
\label{eq:duality-leq}
    \ROT(\mu,\nu) &= \inf_{\substack{\alpha,\beta \in \cM_\plusb(\cX)\\\alpha(\cX) = \beta(\cX) \leq \eps}} \sup_{\varphi \in C_b(\cX)} \int_{\cX}\varphi \dd (\mu + \alpha) + \int_\cX \varphi^c \dd (\nu + \beta)\\
    &\geq \sup_{\varphi \in C_b(\cX)}  \inf_{\substack{\alpha,\beta \in \cM_\plusb(\cX)\\\alpha(\cX) = \beta(\cX) \leq \eps}} \int_{\cX}\varphi \dd (\mu + \alpha) + \int_\cX \varphi^c \dd (\nu + \beta)\\
    &= \sup_{\varphi \in C_b(\cX)} \int_\cX \varphi \dd \mu + \int_\cX \varphi^c \dd \nu - \eps\left(\sup_{x \in \cX} \varphi(x) - \inf_{x \in \cX} \varphi(x)\right).
\end{align*}
By subtracting a constant from $\varphi$, we can ensure that the final term takes its desired form of $2\|\varphi\|_\infty$ without modifying the objective value. For compact $\cX$, one can argue via Sion's minimax theorem that the inequality above is actually an equality; for general $\cX$, however, the situation is more subtle. Fortunately, we can sidestep any functional analysis by applying standard Kantorovich duality to the augmented problem from Proposition~\ref{prop:RWp-augmented-Wp}. The full proof in Section~\ref{prf:RWp-dual} proceeds with a careful analysis of $\bar{c}$-concave functions under the augmented cost $\bar{c}$ to match our penalized objective. Existence of maximizers follows by a compactness and semi-continuity argument under an appropriate weak topology. 
\end{proof}

\begin{remark}[TV as a dual norm]
\label{rem:dual-norm}
Recall that $\|\cdot\|_\tv$ is the dual norm corresponding to the Banach space of measurable functions on $\cX$ equipped with $\|\cdot\|_\infty$. An inspection of the proof of \cref{thm:RWp-dual} reveals that our penalty scales with $\|\cdot\|_\infty$ precisely for this reason.
\end{remark}
\vspace{-1mm}

Lastly, we describe an alternative dual form which ties robust OT to loss trimming---a popular practical tool for robustifying estimation algorithms when $\mu$ and $\nu$ have finite support \citep{shen19}.
\vspace{-1mm}

\begin{restatable}[Loss trimming dual]{proposition}{losstrimming}
\label{prop:loss-trimming}
If $\mu,\nu \in \cP(\cX)$ are uniform distributions over $n$ points each and $\eps \in [0,1]$ is a multiple of $1/n$, then 
\begin{align*}
\vspace{-1mm}
\ROT(\mu,\nu)=\sup_{\varphi \in C_b(\cX)}\left( \min_{\substack{S \subseteq \supp(\mu)\\|S| = (1 - \eps)n}} \frac{1}{n}\sum_{x \in S} \varphi(x) + \min_{\substack{T \subseteq \supp(\nu)\\|T| = (1 - \eps)n}} \frac{1}{n}\sum_{y \in T} \varphi^c(y) \right).
\end{align*}
\end{restatable}
\vspace{-1mm}

The inner minimization problems above clip out the $\eps n$ fraction of samples whose potential evaluations are largest. This is similar to how standard loss trimming clips out a fraction of samples that contribute most to the considered training loss.

\section{Robust Estimation under \texorpdfstring{$\bm{\Wp}$}{Wp}}
\label{sec:robustness}

We now formalize the problem of robust distribution estimation under $\Wp$. Let $X_1, \dots, X_n$ be i.i.d.\ samples from an unknown distribution $\mu \in \cP(\cX)$ and write $\hat{\mu}_n = \frac{1}{n} \sum_{i=1}^n \delta_{X_i}$ for their empirical measure. Under the strong $\eps$-contamination model, an adversary observes these samples and produces corrupted samples $\tilde{X}_1, \dots, \tilde{X}_n$ with empirical measure $\tilde{\mu}_n$, such that the contamination fraction is at most $\eps$, i.e., $\frac{1}{n} \sum_{i=1}^n \mathds{1}{\bigl\{\tilde{X}_i \neq X_i\bigr\}}\leq \eps$. 
We seek an estimate $\mathsf{T}(\tilde{\mu}_n) \in \cP(\cX)$ for $\mu$ that minimizes the $p$-Wasserstein error $\Wp\bigl(\mathsf{T}(\tilde{\mu}_n),\mu\bigr)$.
\vspace{-1mm}

\paragraph{Error and risk.}
Formally, this corruption process may be any measurable transformation of the clean samples and an independent (but otherwise arbitrary) source of randomness. Write $\cM^\mathrm{adv}_n(\mu,\eps)$ for the family of all joint distributions $\tilde{\PP}_n \in \cP(\cX^n)$ of corrupted data $\{\tilde{X}_i\}_{i=1}^n$ that can be obtained this way. Note that the modified samples, of which there are at most $\eps n$, may be highly correlated. Fixing any estimator $\mathsf{T}$ that depends only on the corrupted data (and possibly independent randomness), we define its $n$-sample \emph{robust estimation risk}~by
\vspace{-2mm}
\begin{equation*}
    R_{p,n}(\mathsf{T},\mu,\eps) \coloneqq  \sup_{\tilde{\PP}_n \in \cM_n^\mathrm{adv}(\mu,\eps)} \E_{\tilde{\PP}_n}\bigl[\Wp\bigl(\mathsf{T}(\tilde{\mu}_n),\mu\bigr)\bigr].
\vspace{-1mm}
\end{equation*}
By definition, the error incurred by $\mathsf{T}$ under $\eps$-contamination is bounded by $R_{p,n}$ in expectation. 
No estimator can obtain finite error for all distributions, so we often require that $\mu$ belongs to a family $\cG \subseteq \cP(\cX)$ encoding distributional assumptions. 
In this setting, we define the worst-case risk over $\cG$ and the corresponding \emph{minimax risk}, respectively, by
\[
R_{p,n}(\mathsf{T},\cG,\eps) \coloneqq \smash{\sup_{\mu \in \cG} R_{p,n}(\mathsf{T},\mu,\eps)} \quad \mbox{and} \quad R_{p,n}(\cG,\eps) \coloneqq \inf_{\mathsf{T}} R_{p,n}(\mathsf{T},\cG,\eps).
\]
\vspace{-6mm}

\begin{figure}

\begin{minipage}[c]{0.45\textwidth}
\centering
\begin{tikzpicture}
\definecolor{lightgreen}{HTML}{7FBF7F}
\fill[lightgreen] (1,0) ellipse (2 and 1);

\node at (1.75,-0.4) {\LARGE $\mathcal{G}$};
\coordinate[label = {[label distance=0.075cm]225:$\mu$}] (mu) at (0,0);
\coordinate[label = {[label distance=0.075cm]135:$\hat{\mu}_n$}] (mu_hat) at (0,2.5);
\coordinate[label = {[label distance=0.075cm]45:$\tilde{\mu}_n$}] (mu_tilde) at (3,2.5);
\coordinate[label = {[label distance=0.1cm]0:$\mathsf{T}(\tilde{\mu}_n)$}] (MDE) at (2.3,0.75);

\fill[black] (mu) circle (0.1);
\fill[black] (mu_hat) circle (0.1);
\fill[black] (mu_tilde) circle (0.1);
\fill[black] (MDE) circle (0.1);

\draw (mu) -- node[left]{\scriptsize $\Wp \leq \delta_n$} (mu_hat);
\draw (mu_hat) -- node[above]{\scriptsize $\|\cdot\|_\tv \leq \eps$} (mu_tilde);
\draw[dashed] (mu) -- node[above, sloped]{\scriptsize $\RWp \leq \delta_n$} (mu_tilde);
\draw[dashed] (mu_tilde) -- node[right]{\scriptsize $\:\RWp \leq \delta_n$} (MDE);
\draw[dashed] (mu) -- node[below, sloped]{\scriptsize $\Wp^{2\eps} \leq 2\delta_n$}  (MDE);
\end{tikzpicture}
\end{minipage}%
\begin{minipage}[c]{0.5\textwidth}
\vspace*{5mm}
\caption{Visual depiction of MDE under $\RWp$ and its analysis. Solid lines represent statistical distance bounds given by the problem formulation, and dotted lines represent bounds deduced by our choice of estimator and the approximate triangle inequality for $\RWp$. We abbreviate $\delta_n = \Wp(\hat{\mu}_n,\mu)$.}
\label{fig:RWp-MDE}
\end{minipage}
\end{figure}
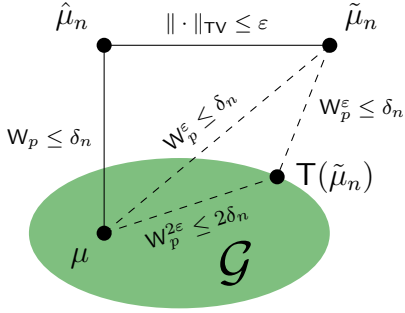

\paragraph{The estimator.}

To solve this robust estimation task, we perform MDE using $\RWp$, projecting $\tilde{\mu}_n$ onto the clean family $\cG$ under the robust distance. That is, we consider the estimator\footnote{As mentioned before, existence and representative selection are inconsequential; approximate minimizers will suffice.}
\vspace{-1mm}
\begin{equation*}
    \mathsf{T}_{[\cG,\RWp]}(\tilde{\mu}_n) \in \argmin_{\nu \in \cG} \RWp(\tilde{\mu}_n,\nu),
\end{equation*}
as illustrated in Figure~\ref{fig:RWp-MDE}. To study the performance of MDE, we begin by analyzing a population-limit version of the problem where measures are observed directly rather than via sampling. In \cref{ssec:generalized-resilience}, we introduce a resilience-type condition which characterizes population-limit minimax risk for a variety of standard distribution classes, and we show that this risk is attained by MDE under $\RWp$. In \cref{ssec:finite-sample}, we return to the finite-sample regime and prove that MDE under $\RWp$ still achieves near-optimal risk bounds. Finally, in \cref{ssec:distance-estimation}, we consider $\RWp$ as a stand-alone estimator for $\Wp$ and quantify its performance. Proofs for this section are deferred to \cref{subsec:prfs-robustness}.

\begin{remark}[One-sided robust distance]
\label{rem:one-sided-distance}
A natural alternative to $\RWp(\tilde{\mu}_n,\nu)$ above is the one-sided robust distance $\RWp(\tilde{\mu}_n\|\nu) \defeq \smash{\inf_{\mu' \in \cP(\cX): \mu' \leq \frac{1}{1-\eps}\tilde{\mu}_n} \Wp(\mu',\nu)}$. This variant seems well-suited to the additive $\eps$-corruption model, i.e., when $\hat{\mu}_n \leq \frac{1}{1-\eps}\tilde{\mu}_n$. Nevertheless, we show in Appendix~\ref{app:asymmetric-results} that MDE under the one-sided distance achieves matching guarantees to those proved in this section, even under strong $\eps$-contamination. It also admits a slightly simpler dual form, which we employ for the generative modeling experiments in \cref{sec:experiments}. Presently, we stick to the symmetric robust distance since its analysis is cleaner.
\end{remark}
\vspace{-3mm}

\subsection{Population-Limit Guarantees and Resilience}
\label{ssec:generalized-resilience}

Before analyzing finite-sample risk, we characterize a baseline risk that is unavoidable even with unlimited samples. Formally, given $\mu$ that belongs to a clean family $\cG \subseteq \cP(\cX)$, we consider the task of estimating $\mu$ under $\Wp$ from a contaminated distribution $\tilde{\mu}$ with $\|\tilde{\mu} - \mu\|_\tv \leq \eps$. Given a map $\mathsf{T}: \cP(\cX) \to \cP(\cX)$, we define its \emph{population-limit robust estimation error} as
\vspace{-1mm}
\begin{equation*}
    R_{p,\infty}(\mathsf{T},\mu,\eps) \coloneqq \sup_{\substack{\tilde{\mu} \in \cP(\cX): \|\tilde{\mu} - \mu\|_\tv \leq \eps}} \Wp\bigl(\mathsf{T}(\tilde{\mu}),\mu\bigr).
\vspace{-1mm}
\end{equation*}
As before, the minimax risk is $R_{p,\infty}(\cG,\eps) \coloneqq \inf_{\mathsf{T}} \sup_{\mu \in \cG} R_{p,\infty}(\mathsf{T},\mu,\eps)$, and MDE corresponds to the estimator $\mathsf{T} = \mathsf{T}_{[\cG,\RWp]}$. In this regime, however, the procedure simplifies considerably; since $\inf_{\nu \in \cG}\RWp(\tilde{\mu},\nu) \leq \RWp(\tilde{\mu},\mu) = 0$, the minimum distance estimate satisfies $\RWp(\mathsf{T}(\tilde{\mu}),\tilde{\mu}) = 0$, or equivalently $\|\mathsf{T}(\tilde{\mu}) - \tilde{\mu}\|_\tv \leq \eps$. Consequently, Lemma~\ref{lem:bounded-modulus-generalized-resilience} implies that this estimator matches the performance of MDE under the TV norm, i.e., $\mathsf{T} = \mathsf{T}_{[\cG,{\|\cdot\|_\tv}]}$, which recovers the standard approach to population-limit robust statistics outlined in \cref{sec:prelims}.

\medskip
To quantify the risk, we first note that if $\cG$ is unrestricted, $R_{p,\infty}(\cG,\eps)$ may be unbounded. Therefore, we next introduce a natural distributional assumption under which accurate estimation is possible. The definition essentially  instantiates generalized resilience \citep{zhu2019resilience} to our setting.

\vspace{-1mm}

\begin{definition}[Resilience w.r.t.\ $\Wp$]
Let $0 \leq \eps < 1$ and $\rho \geq 0$. A distribution $\mu \in \cP(\cX)$ 
is called $(\rho,\eps)$-\emph{resilient} w.r.t.\ $\Wp$ if $\Wp(\mu,\mu') \leq \rho$ for all $\mu' \in \cP(\cX)$ such that $\mu' \leq \frac{1}{1-\eps}\mu$. 
Write $\cW_p(\rho,\eps)$ for the family of all such distributions.
\end{definition}

\vspace{-1mm}

That is, $\mu$ belongs to $\cW_p(\rho,\eps)$ if deleting an $\eps$-fraction of mass from $\mu$ and renormalizing (equivalently, conditioning on an event with probability $1-\eps$) leads to change at most $\rho$ in $\Wp$. In particular, this implies that $\mu$ has finite $p$th moments; otherwise, such change could be infinite. Lemma~\ref{lem:bounded-modulus-generalized-resilience} then gives the following.

\vspace{-1mm}

\begin{proposition}[Bounded risk under resilience] 
\label{prop:minimax-risk}
Let $0 \leq \eps < 1/2$ and $\cG \subseteq \cW_p(\rho,2\eps)$. Then
$R_{p,\infty}(\cG,\eps) \leq 2 \rho$, and this risk is achieved by the estimators $\mathsf{T}_{[\cG,\RWp]}$ and $ \mathsf{T}_{[\cG,\|\cdot\|_\tv]}$.
\end{proposition}

\vspace{-1mm}

While resilience w.r.t.\ $\Wp$ is a high-level condition, we can precisely quantify it for many families of interest. Moreover, the resulting upper risk bounds from Proposition~\ref{prop:minimax-risk} are often tight. In particular, for $\sigma \geq 0$ and $q \geq 1$, we consider the class $\cG_q(\sigma)$ of distributions $\mu \in \cP(\cX)$ such that $\E_\mu[\mathsf{d}(X,x_0)^q] \leq \sigma^q$ for some $x_0 \in \cX$, i.e., those with centered absolute $q$th moments bounded by $\sigma$.
When $\cX = \R^d$, we further consider the following standard families of sub-Gaussian distributions and those with a bounded covariance matrix, namely:
\begin{equation}
\label{eq:moment-bound-classes}
\cG_\mathrm{subG}(\sigma) \coloneqq \left\{ \mu \in \cP(\R^d) : \E_\mu\left[ \exp(|\langle\theta,X - \E_\mu[X]\rangle|^2/\sigma^2)\right] \leq 2  \:\: \forall \theta \in \unitsph \right\},
\end{equation}
\begin{equation*}
    \quad \cG_\mathrm{cov}(\sigma) \coloneqq \left\{ \mu \in \cP(\R^d) : \Sigma_\mu \preceq \sigma^2 \I_d \right\}.
\end{equation*}
We now state tight population-limit minimax risk bounds for these concrete classes.
\vspace{-1mm}

\begin{theorem}[Population-limit minimax risk]
\label{thm:concrete-minimax-risk-bounds}
For $0 \leq \eps \leq 0.49$,\footnote{The stated risk bounds hold for any $\eps$ bounded away from 1/2, with inverse dependence on $1-2\eps$.} we have
\vspace{-1mm}
\begin{align*}
    R_{p,\infty}(\cG,\eps) \asymp \begin{cases}
        \sigma \eps^{\frac{1}{p}-\frac{1}{q}}, & \text{\!\!if $\cG = \cG_q(\sigma)$ for $q\!\geq\!p$ and $\,\exists\, x,y\in\cX\!:\! \mathsf{d}(x,y)\!=\!\sigma \eps^{-\frac 1q}$}\\
        \sigma \sqrt{d + p + \log\left(\frac{1}{\eps}\right)}\,\eps^{\frac 1p}, & \text{\!\!if $\cX = \R^d$ and $\cG = \cG_\mathrm{subG}(\sigma)$}\\
        \sigma \sqrt{d}\, \eps^{\frac{1}{p}-\frac{1}{2}}, & \text{\!\!if $\cX = \R^d$, $p \leq 2$, and $\cG = \cG_\mathrm{cov}(\sigma)$}\\
    \end{cases}
\vspace{-1mm}
\end{align*}
and these risks are achieved by the estimators $\mathsf{T}_{[\cG,\RWp]}$ and $ \mathsf{T}_{[\cG,\|\cdot\|_\tv]}$.
\end{theorem}

\vspace{-1mm}

For the upper bounds, we show in \cref{prf:concrete-minimax-risk-bounds} that resilience w.r.t.\ $\Wp$ is implied by mean resilience of the $p$th power of the metric $\mathsf{d}$ and then apply Proposition~\ref{prop:minimax-risk}. This yields the risk bound for $\cG_q(\sigma)$ and further implies the two latter bounds by observing that $\cG_\mathrm{subG}(\sigma) \subseteq \cG_{p \lor \log(1/\eps)}\bigl(\sqrt{d + p \lor \log(1/\eps)}\bigr)$ and $\cG_\mathrm{cov}(\sigma) \subseteq \cG_{2}\bigl(\sqrt{d}\sigma\bigr)$.
For lower bounds, we use that $R_{p,\infty}(\cG,\eps) \geq \Wp(\alpha,\beta)/2$ for any $\alpha,\beta \in \cG$ with $\|\alpha - \beta\|_\tv \leq \eps$. Indeed, even if one knows that $\mu \in \{\alpha,\beta\} \subset \cG$, there is no way to distinguish between these cases if $\tilde{\mu} = \alpha$. The best one can do in this case is to estimate their Wasserstein barycenter, achieving error $\Wp(\alpha,\beta)/2$.
When $\cG = \cG_q(\sigma)$, for example, consider $\mu = \delta_{x}, \nu = (1-\eps)\delta_x + \eps \delta_{y}$ for some $x,y \in \cX$. Then we have $\E_\mu[\mathsf{d}(X,x)^q] = 0$, $\E_\nu[\mathsf{d}(X,x)^q] = \eps \mathsf{d}(x,y)^q$, and $\Wp(\mu,\nu) = \eps^{1/p} \mathsf{d}(x,y)$. By selecting $d(x,y) = \sigma \eps^{-1/q}$, the maximal distance for which $\mu,\nu \in \cG_q(\sigma)$, we obtain the tight lower bound of $R_{p,\infty}(\cG_q(\sigma),\eps) \gtrsim \Wp(\mu,\nu) = \sigma \eps^{1/p - 1/q}$. Lower bounds for $\cG_\mathrm{subG}(\sigma)$ and $\cG_\mathrm{cov}(\sigma)$ follow by considering appropriate Gaussian mixtures.

\vspace{-1mm}
\begin{remark}[Comparison to mean resilience]
When $\cX = \R^d$, the corresponding rates for mean resilience and population-limit minimax risk for robust mean estimation under the sub-Gaussian and bounded covariance classes are  $O\big(\eps\sqrt{\log(1/\eps)}\big)$ and $O(\sqrt{\eps})$, respectively (see, e.g., \citealp{steinhardt2018resilience}). These match our bounds for $\Wone$ up to $\sqrt{d}$ dependence, which we interpret as reflecting the high-dimensional structure of the Wasserstein metric. In follow-up work, \citet{nietert2022sliced} showed that mean resilience coincides with resilience under max-sliced $\mathsf{W}_1$, which measures $\Wone$ between 1-dimensional projections of the data.
\end{remark}

\begin{remark}[Unknown contamination fraction]
If $\eps$ is unknown, $\mathsf{T}_{[\cG,\|\cdot\|_\tv]}$ can still be applied. Alternatively, one may take $\hat{\eps} = \min \{ \tau \in [0,1) : \exists \nu \in \cG \text{ s.t. } \|\nu - \tilde{\mu}\|_\tv \leq \tau\}$ (ensuring $\hat{\eps} \leq \eps$) and employ $\mathsf{T} = \mathsf{T}_{[\cG,\Wp^{\hat{\eps}}]}$. By construction, the returned $\nu = \mathsf{T}(\tilde{\mu})$ satisfies $\|\nu - \mu\|_\tv \leq \|\nu - \tilde{\mu}\|_\tv + \|\tilde{\mu} - \mu\|_\tv \leq 2\eps$ and thus achieves the risk bounds of Theorem~\ref{thm:concrete-minimax-risk-bounds}.
\end{remark}

\subsection{Finite-Sample Guarantees}
\label{ssec:finite-sample}
We now return to the finite-sample setting. The following result provides a general lower bound and a resilience-based upper bound for minimax risk via MDE. Recall that $\mathsf{T}_{[\cG,\RWp]}^\delta$ returns any distribution which solves the $\RWp$ MDE problem over $\cG$ up to additive error $\delta$.

\vspace{-1mm}

\begin{theorem}[Finite-sample minimax risk]
\label{thm:finite-sample-minimax-risk}
For $0 \leq \eps \leq 0.49$ and $\cG \subseteq \cP(\cX)$, we have
\vspace{-2mm}
\begin{equation*}
    R_{p,n}(\cG,\eps) \gtrsim R_{p,\infty}(\cG,\eps/4) + R_{p,n}(\cG,0).
\end{equation*}
For $\mu \in \cG \subseteq \cW_p(\rho,2\eps)$, the risk of the minimum distance estimator $\mathsf{T}_{[\cG,\RWp]}^\delta$ is bounded as
\vspace{-1mm}
\begin{equation*}
    R_{p,n}(\mathsf{T}_{[\cG,\RWp]}^\delta,\mu,\eps) \lesssim  \rho + \delta + \E[\Wp(\hat{\mu}_n,\mu)].
\vspace{-1mm}
\end{equation*}
Consequently, taking $\delta = 0$, we have $R_{p,n}(\cG,\eps) \lesssim \rho + \sup_{\mu \in \cG}\E[\Wp(\hat{\mu}_n,\mu)]$. %
\end{theorem}

The lower bound formalizes the intuitive notion that finite-sample robust estimation is harder than both population-limit robust estimation and finite-sample standard estimation. The upper bound shows that the risk bound of Proposition~\ref{prop:minimax-risk} translates to the finite-sample setting up to empirical approximation error.
For the distribution families considered in \cref{thm:concrete-minimax-risk-bounds}, these risk bounds match up to the gap between $R_{p,n}(\cG,0)$ and $\sup_{\mu \in \cG}\E[\Wp(\hat{\mu}_n,\mu)]$, i.e., up to the sub-optimality of plug-in estimation under $\Wp$. Under mild conditions given below, both of these quantities scale at rate $n^{-1/d}$.\footnote{Under the further assumption of smooth densities, the plug-in estimator is known to give a suboptimal rate \citep{nilesweed22minimax}; optimization for this regime is beyond the scope of this work.}

\begin{corollary}[Concrete finite-sample risk bounds]
\label{cor:concrete-finite-sample-minimax-risk-bounds}
Let $0 \leq \eps \leq 0.49$, $q > p$, $\cX = \R^d$ for $d > 2p$, and $\cG \in \{\cG_q(\sigma),\cG_\mathrm{subG}(\sigma),\cG_\mathrm{cov}(\sigma)\}$. If $\cG = \cG_q(\sigma)$, further suppose that $d > (1/p - 1/q)^{-1}$, while if $\cG = \cG_\mathrm{cov}(\sigma)$, assume that $p < 2$ and $d > (1/p - 1/2)^{-1}$. Then there exist constants $C_1,C_2 > 0$ depending only on $p$, $q$, and $d$ such that 
\begin{equation*}
    R_{p,\infty}(\cG,\eps) + C_1 \sigma n^{-\frac 1d} \lesssim R_{p,n}(\cG,\eps) \lesssim R_{p,\infty}(\cG,\eps) + C_2 \sigma n^{-\frac 1d},
\end{equation*}
and the upper risk bound is achieved by the minimum distance estimator $\mathsf{T}_{[\cG,\RWp]}$.
\end{corollary}

We view the upper bound of \cref{thm:finite-sample-minimax-risk} as a primary contribution of this work. It has two merits which are essential for meaningful practical guarantees%
. First, we only need to perform approximate MDE, as slack of $\delta$ in optimization only worsens our final bound by $O(\delta)$. Second, the empirical approximation term depends only on the clean measure $\mu$ and can be controlled even under heavy-tailed contamination. If $\mu$ %
has upper $p$-Wasserstein dimension $d_\star$ in the sense of \cite{weed2019} then the $n^{-1/d}$ rate of Corollary~\ref{cor:concrete-finite-sample-minimax-risk-bounds} can be improved to $n^{-1/d_\star}$. We sketch the theorem proof, deferring full details to \cref{prf:finite-sample-minimax-risk}.
\medskip

\begin{proof}[Proof sketch of \cref{thm:finite-sample-minimax-risk}]
Our proof follows the ``weaken the distance'' approach to robust statistics formalized in \cite{zhu2019resilience}. We observe that analysis of MDE under the TV distance, as applied to obtain population-limit upper risk bounds, fails in the finite-sample setting because $\|\hat{\mu}_n - \mu\|_\tv$ may be large.\footnote{E.g., if $\mu$ has a Lebesgue density then $\|\hat{\mu}_n - \mu\|_\tv=1$, for all $n$.} In particular, $\tilde{\mu}_n$ cannot be viewed as a small TV perturbation of $\mu$, even though $\|\tilde{\mu}_n - \hat{\mu}_n\|_\tv \leq \eps$. On the other hand, by Proposition~\ref{prop:RWp-structural-properties}, we can bound $\RWp(\tilde{\mu}_n,\mu) \leq  \RWp(\tilde{\mu}_n,\hat{\mu}_n) + \Wp(\hat{\mu}_n,\mu) = \Wp(\hat{\mu}_n,\mu)$, which is small with high probability; this motivates our choice of approximate MDE under $\RWp$. In order to quantify its performance, we bound the relevant modulus of continuity.

\begin{lemma}[$\bm{\RWp}$ modulus of continuity]
\label{lem:RWp-modulus}
For $0 \leq \eps \leq 0.99$ and $\lambda \geq 0$, we have
\begin{align*}
    \sup_{\substack{\mu,\nu \in \cW_p(\rho,\eps) \\ \RWp(\mu,\nu) \leq \lambda}} \Wp(\mu,\nu) \lesssim \lambda + \rho.
\end{align*}
\end{lemma}

Now, to prove the upper bound, fix $\mu \in \cG$ with empirical measure $\hat{\mu}_n$, and consider any distribution $\tilde{\mu}_n$ with $\|\tilde{\mu}_n - \hat{\mu}_n\|_\tv \leq \eps$. Then, for $\mathsf{T} = \mathsf{T}_{[\cG,\RWp]}^\delta$, we have
\begin{align*} %
    \Wp^{2\eps}\bigl(\mu,\mathsf{T}(\tilde{\mu}_n)\bigr) &\leq \RWp(\mu,\tilde{\mu}_n) + \RWp\bigl(\tilde{\mu}_n,\mathsf{T}(\tilde{\mu}_n)\bigr) \tag*{(Proposition~\ref{prop:RWp-structural-properties})}\\
    &\leq 2\RWp(\mu,\tilde{\mu}_n) + \delta \tag*{(MDE guarantee and $\mu \in \cG$)}\\
    &\leq 2\RWp(\tilde{\mu}_n,\hat{\mu}_n) + 2\Wp(\hat{\mu}_n,\mu) + \delta. \tag*{(Proposition~\ref{prop:RWp-structural-properties})}\\
    &\leq 2\Wp(\hat{\mu}_n,\mu) + \delta. \tag*{($\|\tilde{\mu}_n - \hat{\mu}_n\|_\tv \leq \eps$)}
\end{align*}
Writing $\lambda_n = 2\Wp(\hat{\mu}_n,\mu) + \delta$, Lemma~\ref{lem:RWp-modulus} gives that
\vspace{-1mm}
\begin{equation*}
    \Wp\bigl(\mu,\mathsf{T}(\tilde{\mu}_n)\bigr) %
    \leq  \sup_{\substack{\alpha,\beta \in \cG \\ \Wp^{2\eps}(\alpha,\beta) \leq \lambda_n}} \Wp(\alpha,\beta) \lesssim \lambda_n+\rho,
\vspace{-1mm}
\end{equation*}
and so $R_{p,n}(\mathsf{T},\mu,\eps) \lesssim \rho + \delta + \E[\Wp(\hat{\mu}_n,\mu)]$, as desired. 
\smallskip

For the lower bound, we trivially have that $R_{p,n}(\cG,0) \leq R_{p,n}(\cG,\eps)$ since $\cM_n^\mathrm{adv}(\mu,0) \subseteq \cM_n^\mathrm{adv}(\mu,\eps)$ for all $\mu$, i.e., the adversary may opt to corrupt no samples. Moreover, in \cref{prf:finite-sample-minimax-risk}, we prove that $R_{p,\infty}(\cG,\eps/4) \leq 8R_{p,n}(\cG,\eps)$ by transforming any estimator for the $n$-sample problem into an estimator for the population-limit problem with similar guarantees under slightly less contamination. To do so, we leverage the fact that a population-limit estimator has access to unlimited i.i.d.\ samples from the contaminated distribution and can use these to simulate many copies of an $n$-sample estimator.
\end{proof}

\vspace{-5mm}

\begin{remark}[Exploiting low-dimensional structure]
If $\mu$ is supported on a (possibly unknown) $k$-dimensional subspace of $\R^d$, one can substitute MDE over a class $\cG$ with MDE over the projected family $\cG_k = \{ U_\sharp \nu : U \in \R^{k \times d}, UU^\top = I_k, \nu \in \cG \}$. For the families in \cref{thm:concrete-minimax-risk-bounds}, each occurrence of $d$ in the population-limit risk bounds can then be improved to $k$. Similarly, the $n^{-1/d}$ rate in the finite-sample risk bounds can be improved to $n^{-1/k}$. If $\mu$ admits no such structure, one may still enjoy such improved bounds if willing to relax the $\Wp$ estimation guarantee to a weaker guarantee under the $k$-dimensional max-sliced Wasserstein distance $\MWpk(\mu,\nu) \defeq \sup_{U \in \R^{k \times d}: U U^\top = I_k} \Wp(U_\sharp \mu,U_\sharp \nu)$. See Appendix~\ref{app:sliced} for further details.
\end{remark}

\begin{remark}[Comparison to MDE under TV]
Although $\|\mu - \hat{\mu}_n\|_\tv$ may be large for all $n$, MDE under the TV norm still provides strong robust estimation guarantees in the finite-sample regime via a slight modification to the analysis. For example, if $\mu \in \cG_q(\sigma)$ for $q > p$, then the empirical distribution $\hat{\mu}_n$ belongs to $\cG_q(20\sigma)$ with probability at least $0.95$ by Markov's inequality. Hence, our population-limit risk bound for the class $\cG_q(20\sigma)$ implies that the minimum distance estimator $\mathsf{T} = \mathsf{T}_{[\cG_q(20\sigma),\|\cdot\|_\tv]}$ satisfies the near-optimal guarantee
\begin{align*}
    \Wp\big(\mathsf{T}(\tilde{\mu}_n),\mu\big) &\leq \Wp\big(\mathsf{T}(\tilde{\mu}_n),\hat{\mu}_n\big) + \Wp(\hat{\mu}_n,\mu) \lesssim \sigma \eps^{\frac{1}{p} - \frac{1}{q}} + \E[\Wp(\hat{\mu}_n,\mu)]
\end{align*}
with probability at least $0.9$. Our population limit risks for $\cG_\mathrm{cov}(\sigma)$ and $\cG_\mathrm{subG}(\sigma)$ followed from the inclusions $\cG_\mathrm{cov}(\sigma) \subseteq \cG_2(\sigma \sqrt{d})$ and $\cG_\mathrm{subG}(\sigma) \subseteq \cG_{p \lor \log(1/\eps)}\bigl(O(\sqrt{d + p + \log(1/\eps)} \sigma)\bigr)$, so the approach above translates to these settings as well, matching the guarantees of Corollary~\ref{cor:concrete-finite-sample-minimax-risk-bounds}. However, MDE under TV lacks the connections of $\RWp$ to the theory (POT) and practice (WGAN) of OT in modern machine learning, which is a key motivation for our approach.
\end{remark}

\subsection{\texorpdfstring{$\bm{\RWp}$}{Robust Wp} as a Robust Estimator for \texorpdfstring{$\bm{\Wp}$}{Wp}}
\label{ssec:distance-estimation}

Finally, we take a detour from MDE and examine robust estimation of the distance $\Wp(\mu,\nu)$ itself given $\eps$-corrupted samples from $\mu$ and $\nu$. Formally, this is no harder than the previously considered task of robust distribution estimation under $\Wp$, since the triangle inequality gives $\left|\Wp\bigl(\mathsf{T}(\tilde{\mu}_n),\mathsf{T}(\tilde{\nu}_n)\bigr) - \Wp(\mu,\nu)\right| \leq \Wp\bigl(\mathsf{T}(\tilde{\mu}_n),\mu\bigr) + \Wp\bigl(\mathsf{T}(\tilde{\nu}_n),\nu\bigr)$.
Moreover, a simple modulus of continuity argument reveals that the minimax population-limit risk for this task (restricting $\mu$ and $\nu$ to some clean family $\cG$) coincides with that of the initial problem. While MDE thus solves the problem of \emph{robust distance estimation}, we next explore a simple alternative approach: compute $\RWp(\tilde{\mu}_n,\tilde{\nu}_n)$ and output it as the estimate of $\Wp(\mu,\nu)$. We now provide guarantees for this natural procedure.

\begin{theorem}[Finite-sample robust estimation of $\bm{\Wp}$]
\label{thm:robust-distance-estimation}
Let $0 \leq \eps < 1/3$ and write $\tau = 1 - (1-3\eps)^{1/p} \in [3\eps/p,3\eps]$. Let $\mu,\nu \in \cW_p(\rho,3\eps)$ have empirical measures $\hat{\mu}_n$ and $\hat{\nu}_n$, respectively. For any $\tilde{\mu}_n,\tilde{\nu}_n \in \cP(\cX)$ such that $\|\tilde{\mu}_n - \hat{\mu}_n\|_\tv, \|\tilde{\nu}_n - \hat{\nu}_n\|_\tv \leq \eps$, we have
\vspace{-1mm}
\begin{equation*}
\left|\RWp(\tilde{\mu}_n,\tilde{\nu}_n) - \Wp(\mu,\nu)\right| \leq 2\rho + \tau \Wp(\mu,\nu) + \Wp(\mu,\hat{\mu}_n) + \Wp(\nu,\hat{\nu}_n).
\vspace{-1mm}
\end{equation*}
In particular, if $\eps \leq 0.33$ and $\mu,\nu \in \cG$ for $\cG \in \{\cG_q(\sigma),\cG_\mathrm{subG}(\sigma),\cG_\mathrm{cov}(\sigma)\}$, $q \geq p$, then $\rho$ can be replaced with $R_{p,\infty}(\cG,\eps)$ in the bound above.
\end{theorem}
\vspace{-1mm}

Thus, for large sample sizes, the $\RWp$ estimate suffers from additive estimation error $O(\rho)$ and multiplicative estimation error $O(\tau)$, matching the resilience-based guarantees of MDE when $\Wp(\mu,\nu)$ is sufficiently small. The proof is a straightforward application of techniques from previous sections. We note that the upper bound of $1/3$ on $\eps$, i.e., the breakdown point of the $\RWp$ estimator, cannot be improved. Indeed, for any $\mu,\nu \in \cP(\cX)$, the proof includes a construction of $\tilde{\mu},\tilde{\nu}$ such that $\|\mu - \tilde{\mu}\|_\tv,\|\nu - \tilde{\nu}\|_\tv \leq 1/3$ but $\Wp^{\eps}(\tilde{\mu},\tilde{\nu}) = 0$.\smallskip

If $\eps = o(1)$ as $n \to \infty$, we can recover $\Wp(\mu,\nu)$ exactly in the limit.

\begin{corollary}[Asymptotic consistency]
\label{cor:asymptotic-consistency}
Fix $\mu,\nu \in \cP_q(\cX)$ for $q > p$, and consider any contaminated versions $\{\tilde{\mu}_n\}_{n=1}^\infty$, $\{\tilde{\nu}_n\}_{n=1}^\infty$ such that $\|\tilde{\mu}_n - \hat{\mu}_n\|_\tv \lor \|\tilde{\nu}_n - \hat{\nu}_n\|_\tv \leq \eps_n$, for all $n\in\NN$, with $\eps_n \to 0$. Then, for any sequence $\{\tau_n\}_{n=1}^\infty \subseteq [0,1]$ such that $\eps_n = o(\tau_n)$ and $\tau_n \to 0$, we have $\Wp^{\tau_n}(\tilde{\mu}_n,\tilde{\nu}_n) \to \Wp(\mu,\nu)$.
\end{corollary}

\begin{remark}[Comparison to median-of-means estimator]
This consistency result is reminiscent of that presented by \cite{staerman21} for a median-of-means (MoM) estimator. They produce a robust estimate for $\Wone$ by partitioning the contaminated samples into blocks and replacing each mean appearing in the dual form of $\Wone$ with a median of block means, where the number of blocks depends on the contamination fraction.
Our estimator improves upon the MoM approach in two important respects. First, our guarantees hold even when the support of $\mu$ and $\nu$ is unbounded. Second, their MoM guarantees do not account for the setting of \cref{thm:robust-distance-estimation} when $\eps_n = \eps$ is a fixed constant.
\end{remark}

\cref{thm:robust-distance-estimation} reveals that $\RWp(\tilde{\mu}_n,\tilde{\nu}_n)$ is a good estimate when the true distance $\Wp(\mu,\nu)$ is small. In particular, for $\eps$ bounded away from $1/3$ and $n$ sufficiently large, $\RWp(\tilde{\mu}_n,\tilde{\nu}_n) \asymp \Wp(\mu,\nu) \pm \rho$. This guarantee lends itself well to the tasks of robust two-sample testing and independence testing under $\Wp$. In the first case, one receives $\eps$-corrupted samples $\{\tilde{X_i}\}_{i=1}^n$ and $\{\tilde{Y_i}\}_{i=1}^n$ from $\mu$ and $\nu$, respectively, and seeks to distinguish between the null and alternative hypotheses:
\[H_0: \mu = \nu\quad \mbox{versus}\quad H_1: \Wp(\mu,\nu) > \lambda.\] In the second, one receives $\eps$-corrupted samples $\{(\tilde{X_i},\tilde{Y_i})\}_{i=1}^n$ and seeks to determine whether the joint distribution $\kappa \in \cP(\cX^2)$ of clean samples with marginals $\kappa_1,\kappa_2$ satisfies $\kappa = \kappa_1 \otimes \kappa_2$ or $\Wp(\kappa, \kappa_1 \otimes \kappa_2) > \lambda$, for the metric $\bar{\mathsf{d}}\bigl((x,y),(x',y')\bigr) \coloneqq \mathsf{d}(x,x') \lor \mathsf{d}(y,y')$ on the product space. Below, empirical measures of all sample sizes are defined via a single infinite sequence of samples, i.e., $X_1, X_2, \dots \sim \mu$ are i.i.d.\ and $\hat{\mu}_n \coloneqq \frac{1}{n} \sum_{i=1}^n \delta_{X_i}$ is the empirical measure of the first $n$ samples, for all $n \in \N$.
\vspace{-1mm}

\begin{proposition}[Applications to robust two-sample and independence testing]
\label{prop:two-sample-and-independence-testing}
Let $\eps \in [0,1/4]$ and fix $\mu,\nu \in \cW_p(\rho,3\eps)$. Then, for $\lambda = 45\rho$, we have 
\vspace{-1mm}
\begin{equation*}
    \lim_{n \to \infty} \mathds{1}\left\{\RWp(\tilde{\mu}_n,\tilde{\nu}_n) > 3\rho \right\} = \mathds{1}\left\{\Wp(\mu,\nu) > \lambda \right\} \quad \text{almost surely (a.s.)}
\vspace{-1mm}
\end{equation*}
under both the null and alternative hypotheses for robust two-sample testing. Similarly, for independence testing, if $\kappa_1,\kappa_2 \in \cW_p(\rho,3\eps)$, we have 
\vspace{-1mm}
\begin{equation*}
    \lim_{n \to \infty} \mathds{1}\left\{\RWp(\tilde{\kappa}_n,\tilde{\kappa}_{1,n} \otimes \tilde{\kappa}_{2,n}) > 3\rho \right\} = \mathds{1}\left\{\Wp(\kappa,\kappa_1 \otimes \kappa_2) > \lambda \right\} \quad \text{a.s.}
\vspace{-1mm}
\end{equation*}
under the null and alternative, where $\tilde{\kappa}_n \coloneqq \frac{1}{n} \sum_{i=1} \delta_{(\tilde{X}_i,\tilde{Y}_i)}$ and $\tilde{\kappa}_{1,n}, \tilde{\kappa}_{2,n}$ are its marginals.
\end{proposition}

\section{Experiments}
\label{sec:experiments}

To demonstrate the practical utility of our robust OT framework, we implement a neural network (NN) based approximation for the $\RWone$ minimum distance estimator by adapting the WGAN. Applications to generative modeling with contaminated data are then presented.

\vspace{-2mm}

\subsection{The Wasserstein GAN}
\label{subsec:WGAN}

Recall that WGAN seeks to approximate a data distribution $\mu \in \cP(\R^d)$ by the pushforward of a simple distribution $\nu_0$ through a parameterized map $g_\theta : \R^d \to \R^d$, $\theta \in \Theta$. Employing the $\Wone$ objective, \citet{arjovsky_wgan_2017} coined the following MDE problem as WGAN:
\begin{equation*}
    \inf_{\theta \in \Theta} \Wone\bigl(\mu, (g_\theta)_\sharp \nu_0\bigr) = \inf_{\theta \in \Theta} \sup_{\varphi \in \Lip_1(\R^d)} \int_{\R^d} \varphi \dd \mu - \int_{\R^d} \varphi \circ g_\theta \dd \nu_0 .
\end{equation*}
In practice, the function class $\Lip_1(\R^d)$ is approximated by a parameterized NN family $\cF = \{f_\phi\}_{\phi \in \Phi}$ and the min-max problem is tackled via alternating stochastic gradient descent-ascent. Namely, given a mini-batch of i.i.d.\ samples $X_1, \dots, X_n \sim \mu$ and $Y_1, \dots, Y_n \sim \nu_0$, we update iterates $\theta_t$ and $\phi_t$ according to the respective gradients of the objective
\begin{equation*}
h(\theta,\phi) \defeq \frac{1}{n} \sum_{i=1}^n f_\phi(X_i) - \frac{1}{n} \sum_{i=1}^n f_\phi(g_\theta(Y_i)).
\end{equation*}
There are many challenges to effectively implementing and analyzing this approach. For example, enforcing the Lipschitz constraint is non-trivial and is often performed softly by adding a certain gradient penalty to $h$ \citep{gulrajani2017improved}. Moreover, if the NNs involved are poly-sized, computational hardness results suggest that we must have $\Wone\bigl(\mu,(g_{\theta_\star})_\sharp \nu_0\bigr) \gg 1$ in the worst case, even when $\theta_\star = \argmin_{\theta \in \Theta} \max_{\phi \in \Phi} h(\theta,\phi)$ is an exact min-max solution \citep{chen2022minimax}. Finally, in the absence of significant structural assumptions on $\mu$, any estimate $\hat{\nu}$ for $\mu$ based on $n$ samples must incur worst-case error $\Wone(\mu,\hat{\nu}) \gtrsim n^{-1/d}$ \citep{nilesweed22minimax}. This rate is vacuous for $d \gtrsim \log(n)$ and appears in the standard analysis of Wasserstein MDE, even when the distribution class $\cG \defeq \{(g_\theta)_\sharp \nu_0\}_{\theta \in \Theta}$ is well-behaved and contains $\mu$.
Nevertheless, WGAN has enjoyed enormous practical success, and we therefore seek to translate its performance to our setting with corrupted data.\smallskip

\subsection{Robustifying WGAN}
\label{subsec:comp}

Consider the Huber contamination model, where the learner observes contaminated samples from a mixture $\tilde{\mu} = (1-\eps)\mu + \eps \alpha$, for some undesirable data source $\alpha$. To address this, we replace $\Wone$ above with its robust proxy $\RWone$. As discussed in Remark~\ref{rem:one-sided-distance}, it is slightly more natural to employ the one-sided variant $\RWone(\alpha\|\beta) \defeq \inf_{\alpha' \leq \frac{1}{1-\eps}\alpha}\Wone(\alpha',\beta)$ under Huber contamination. Opting for this approach, an adaptation of Theorem~\ref{thm:RWp-dual} to the one-sided $\RWone$ case (see Appendix~\ref{app:asymmetric-results}) yields the robust MDE problem:
\begin{equation}
\label{eq:RWone-MDE}
    \inf_{\theta \in \Theta} \RWone\bigl(\tilde{\mu}\| (g_\theta)_\sharp \nu_0\bigr) = \inf_{\theta \in \Theta} \sup_{\varphi \in \Lip_1(\R^d)} \frac{1}{1-\eps}\int_{\R^d} \varphi \dd \tilde{\mu} -\int_{\R^d} \varphi \circ g_\theta \dd \nu_0 + \frac{\eps}{1-\eps} \sup_{x \in \cX} \varphi(x).
\end{equation}
For the corresponding NN approximation, we update iterates $\theta_t$ and $\lambda_t$ according to the respective (sub)gradients of the modified objective
\begin{equation*}
h^\eps(\theta,\lambda) \defeq \frac{1}{(1-\eps)n} \sum_{i=1}^n f_\lambda(X_i) - \frac{1}{n} \sum_{i=1}^n f_\lambda(g_\theta(Y_i)) + \frac{\eps}{1-\eps} \max_{i =1,\dots,n} f_\lambda(X_i).
\end{equation*}
Any implementation of WGAN can easily be adapted to support this robustified objective; for example, our \texttt{PyTorch} adjustment takes the following one-line form:
\newcommand\Small{\fontsize{13}{9.2}\selectfont}
\newcommand*\LSTfont{\Small\ttfamily\SetTracking{encoding=*}{-60}\lsstyle}
\begin{lstlisting}[language=Python,basicstyle=\LSTfont,deletendkeywords={max},columns=fullflexible]
score = f_data.mean() - f_gen.mean() # old
score = f_data.mean()/(1-eps) - f_gen.mean()  - eps*f_data.max()/(1-eps) # new
\end{lstlisting}
\vspace{-1mm}
Due to the non-convex and non-smooth nature of the $h^\eps$, formal optimization guarantees seem challenging to obtain. Nevertheless, the next subsection demonstrates strong empirical performance of the resulting robust WGAN.

\vspace{-1mm}

\subsection{Generative Modeling with Contaminated Datasets}
\label{subsec:gen_mod}

We examine outlier-robust generative modeling using the modification suggested above. The experiments were performed on a cluster machine equipped with a NVIDIA Tesla V100. Code for these results is available at \url{https://github.com/sbnietert/robust-OT}, and full experimental details are provided in Appendix~\ref{app:experiment-details}.

First, we trained two WGAN with gradient penalty (WGAN-GP) models \citep{gulrajani2017improved} on a contaminated dataset with 80\% MNIST data and 20\% uniform random noise. Both models used a standard selection of hyperparameters but one of them was adjusted to compute gradient updates according to the robust objective with $\eps = 0.25$. In \cref{fig:generated-samples} (top), we display generated samples produced by both networks after processing 125k contaminated batches of 64 images. The effect of outliers is clearly mitigated by training with the robustified objective. 

We also applied our robustification technique to more sophisticated generative models which incorporate additional regularization on top of WGAN. In particular, we trained two off-the-shelf StyleGAN~2 models \citep{karras2020} using contaminated data---this time, 80\% CelebA-HQ face photos and 20\% MNIST data---with one tweaked to perform gradient updates according to the robust objective with $\eps = 0.25$. We present generated samples in \cref{fig:generated-samples} (bottom). Again, the modified objective enables learning a model that is largely outlier-free despite being trained on corrupted data.

\begin{wrapfigure}{R}{0.5\textwidth}
\centering
\vspace{0mm}
\hspace{1mm} Robust WGANs \hspace{5mm} Standard WGANs\\
\vspace{2mm}
\includegraphics[width=.48\linewidth]{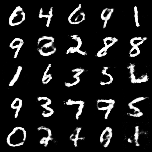}
\includegraphics[width=.48\linewidth]{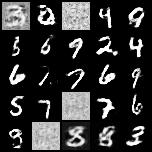}\\
\vspace{0.75mm}
\includegraphics[width=.48\linewidth]{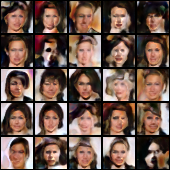}
\includegraphics[width=.48\linewidth]{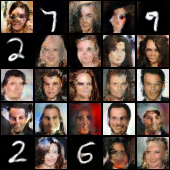}
\caption{(top): samples generated by robustified (left) and standard (right) WGAN-GP after training on corrupted MNIST dataset. (bottom): samples generated by robustified (left) and standard (right) StyleGAN 2 after training on corrupted CelebA-HQ dataset (left).}\label{fig:generated-samples}
\vspace{-5mm}
\end{wrapfigure}

Finally, we compared our robust WGAN approach with existing techniques of \citet{balaji2020} and \citet{staerman21}.\footnote{Unfortunately, the WGAN described in \cite{mukherjee2021} does not appear to scale to high-dimensional image data, and hence no comparison is presented.} Using the publicly available code for these papers, we trained robust WGANs with default options for 100k batches of 128 images, using 95\% 

\noindent CIFAR-10 training images with 5\% uniform random noise. In \cref{fig:generated-samples-2} and \cref{table:fids}, the control WGAN is that of \citet{balaji2020} without its robustification weights enabled. Our robust WGAN is implemented by adding the objective modification with $\eps=0.07$ to this control. \cref{fig:generated-samples-2} presents random samples of generated images from the four WGANs, while \cref{table:fids} displays Frechet inception distance (FID) scores \citep{heusel2017} over the course of training (lower scores correlate with high quality images). Our method performs favorably without tuning, but more detailed empirical study is needed to separate the impact of various hyperparameters. In particular, the poor relative performance of \cite{staerman21} is likely due in part to its distinct architecture. We note that despite using the code provided by \cite{balaji2020} with the hyperparameters specified in their appendix%
, the FID scores we observed were lower than those presented in their paper.

\begin{remark}[Gaps between theory and practice]
Suppose optimistically that our robust WGAN implementation could find an exact minimizer for \eqref{eq:RWone-MDE}. Even then, the theory from Section~\ref{sec:robustness} is insufficient to fully justify strong performance for distribution estimation if the family of generated distributions, $\cG = \{ (g_\theta)_\sharp \nu_0 \}_{\theta \in \Theta}$, is sufficiently expressive. Indeed, if $\cG$ contains both the clean and corrupted population measures, i.e., $\tilde{\mu} = (g_{\tilde{\theta}})_\sharp \nu_0$ and $\mu = (g_{\theta_\star})_\sharp \nu_0$ for some $\tilde{\theta},\theta_\star \in \Theta$, then both $\tilde{\theta}$ and $\theta_\star$ minimize \eqref{eq:RWone-MDE}. Our formal risk bounds in Section~\ref{sec:robustness} sidestep this issue by assuming that the distributions in $\cG$ satisfy Wasserstein resilience bounds. However, we make no effort to explicitly control the complexity of $\cG$ in our experiments. Thus, a satisfactory explanation for our strong empirical performance seems to require a white-box examination of the optimization procedure and its implicit biases. We defer a formal explanation of this phenomenon
for future work.
\end{remark}

{
\begin{figure}[H]
{
\small
\hspace{16.5mm} Ours \hspace{16mm} \cite{balaji2020} \hspace{2.5mm} \cite{staerman21} \hspace{10mm} Control
}
\begin{center}
\includegraphics[width=.23\linewidth]{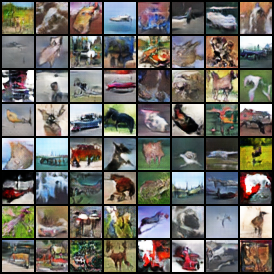}
\includegraphics[width=.23\linewidth]{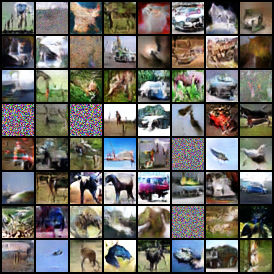}
\includegraphics[width=.23\linewidth]{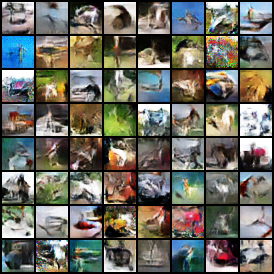}
\includegraphics[width=.23\linewidth]{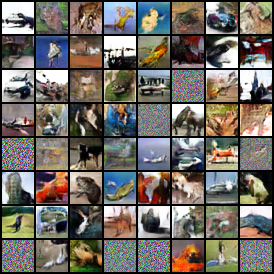}
\end{center}
\caption{Samples generated by various robust WGANs after 100k batches of training.}\label{fig:generated-samples-2}
\end{figure}

\begin{table}[H]
\centering
\begin{tabular}{lcccc}
\toprule
& \multicolumn{4}{c}{Robust WGAN Method}
\\\cmidrule(lr){2-5}
 Batch       & Ours  & \cite{balaji2020} & \cite{staerman21} & Control \citep{balaji2020} \\\midrule
20k    & 38.61 & 60.92 & 58.19 & 59.19 \\
60k & 22.08 & 35.49 & 49.64 & 39.08\\
100k & 18.87 & 30.44 & 40.81 & 31.40\\\bottomrule
\end{tabular}
\caption{FID scores between uncontaminated CIFAR-10 image dataset and datasets generated by various robust WGANs during training (lower scores are better).}
\label{table:fids}
\end{table}
}

\section{Proofs}
\label{sec:proofs}
We now provide proofs for the main results of this work.

\subsection{Proofs for \cref{sec:RWp}}
\label{prf:RWp}

We start by reviewing some relevant facts regarding OT, couplings, and $c$-transforms. Recall that the cost $c$ is of the form $c(x,y) = h(d(x,y))$, where $h:\R_{\geq 0} \to \R_{\geq 0}$ is continuous and strictly increasing with $h(0) = 0$. Further, we write $\lambda \cP(\cX) \coloneqq \{ \mu \in \cM_\plus(\cX): \mu(\cX) = \lambda \}$ for $\lambda \geq 0$, and, for $\varphi:\cX \to \R$, write $\varphi_{\min} \defeq \inf_{x \in \cX}\varphi(x)$ and $\varphi_{\max} \defeq \sup_{x \in \cX}\varphi(x)$.

\begin{fact}
\label{fact:ignore-shared-mass}
For any $\mu,\nu \in \cP(\cX)$, we have $\OT(\mu,\nu) \leq \OT(\mu - \mu \land \nu, \nu - \mu \land \nu).$
\end{fact}
\begin{proof}
For any $\pi \in \Pi(\mu - \mu \land \nu, \nu - \mu \land \nu)$, the joint distribution $\pi' = (\Id,\Id)_\#(\mu \land \nu) + \pi$ is a coupling of $\mu$ and $\nu$ with $\int_\cX c \dd \pi' = \int_\cX c \dd \pi$. Hence, $\OT(\mu,\nu) \leq \inf_{\pi \in \Pi(\mu - \mu \land \nu, \nu - \mu \land \nu)} \int c \dd \pi = \OT(\mu - \mu \land \nu, \nu - \mu \land \nu)$.
\end{proof}
\vspace{-5mm}

\begin{fact}
\label{fact:coupling-decomposition}
For $\mu_1,\mu_2,\nu \in \cM_\plus(\cX)$ such that $(\mu_1 + \mu_2)(\cX) = \nu(\cX)$, any joint measure $\pi \in \Pi(\mu_1 + \mu_2,\nu)$ can be decomposed as $\pi = \pi_{\mu_1,\nu} + \pi_{\mu_2,\nu}$ for $\pi_{\mu_1,\nu},\pi_{\mu_2,\nu} \in \cM_\plus(\cX^2)$ such that $\mu_1(S) = \pi_{\mu_1,\nu}(S \times \cX)$ and $\mu_2(A) = \pi_{\mu_2,\nu}(S \times \cX)$ for all measurable $S \subseteq \cX$.
\end{fact}
\begin{proof}
Write $\pi(y|x)$ for any regular conditional probability measure such that $\pi(S \times T) = \int_A \pi(T|x) \dd (\mu_1 + \mu_2)(x)$ for measurable $S,T \subseteq \cX$. Then $\pi_{\mu_1,\nu},\pi_{\mu_2,\nu} \in \cM_\plus(\cX^2)$ defined by
\begin{equation*}
    \pi_{\mu_1,\nu}(S \times T) \coloneqq \int_S \pi(T|x) \dd \mu_1(x) \quad \text{and} \quad \pi_{\mu_2,\nu}(S \times T) \coloneqq \int_S \pi(T|x) \dd \mu_2(x)
\end{equation*}
sum to $\pi$, with $\pi_{\mu_1,\nu}(S \times \cX) = \int_S \pi(\cX|x) \dd \mu_1(x) = \mu_1(S)$ and $\pi_{\mu_2,\nu}(S \times \cX) = \mu_2(S)$.
\end{proof}

\begin{fact}
\label{fact:TV-alternate-versions}
If $\mu,\nu \in \cP(\cX)$, then $\|\mu - \nu\|_\tv = 1 - \max_{\substack{0 \leq\alpha \leq \mu \land \nu}} \alpha(\cX) = \min_{\beta \geq \mu \lor \nu} \beta(\cX) - 1$.
\end{fact}
\begin{proof}
If $\|\mu - \nu\|_\tv = \eps$, then $(\mu \land \nu)(\cX) = 1-\eps$. Thus, one can take $\alpha = \mu \land \nu$ to bound the second expression by $\eps$. If the second expression is at most $\eps$, i.e., there is $\alpha \in (1-\eps)\cP(\cX)$ such that $\alpha \leq \mu \land \nu$, then one can take $\beta = \mu + \nu - \alpha$ to show the third expression is at most $\eps$. Finally, if the third expression is at most $\eps$, i.e., there exists $\beta \in (1+\eps)\cP(\cX)$ such that $\mu \leq \beta$ and $\nu \leq \beta$, then $\|\mu - \nu\|_\tv \leq \|\mu - \beta\|_\tv + \|\beta - \nu\|_\tv \leq \eps/2 + \eps/2 = \eps$.
\end{proof}

\begin{lemma}
\label{lem:OT-TV-commutativity}
For $\alpha,\beta_\plus \in \cM_\plus(\cX)$ such that $\alpha(\cX) \leq \beta_\plus(\cX)$, we have
\begin{equation*}
    \min_{\substack{\beta \in \cM_\plusb(\cX)\\\beta \leq \beta_\plusb\\\beta(\cX) = \alpha(\cX)}} \OT(\alpha,\beta) = \min_{\substack{\alpha_\plusb \in \cM_\plusb(\cX)\\\alpha \leq \alpha_\plusb\\\alpha_\plusb(\cX) = \beta_\plusb(\cX)}} \OT(\alpha_\plusb,\beta_\plusb).
\end{equation*}
\end{lemma}
\begin{proof}
For the ``$\leq$'' direction, fix $\alpha_\plus$ feasible for the right minimum, and let $\pi_\plus \in \Pi(\alpha_\plus,\beta_\plus)$ be an optimal coupling for the $\OT(\alpha_\plus,\beta_\plus)$ problem. By Fact~\ref{fact:coupling-decomposition}, we can decompose $\pi_\plus = \pi + \kappa$ for $\pi,\kappa \in \cM_\plus(\cX^2)$ such that $\alpha$ is the left marginal of $\pi$. Then, taking $\beta \in \cM_\plus(\cX)$ to be the right marginal of $\pi$, $\beta$ is clearly feasible for  the left infimum. Moreover, we bound
\begin{equation*}
    \OT(\alpha,\beta) \leq \int_\cX c \dd \pi \leq \int_\cX c \dd \pi_\plus = \OT(\alpha_\plus,\beta_\plus). 
\end{equation*}
Thus, taking an infimum over $\alpha_\plus$ gives the ``$\leq$'' inequality.

For the ``$\geq$'' direction, let $\beta$ be feasible for the left infimum, and take $\pi \in \Pi(\alpha,\beta)$ to be an optimal coupling for the $\OT(\alpha,\beta)$ problem. Set $\alpha_\plus = \alpha + (\beta_\plus - \beta)$ and let $\pi_\plus = \pi + (\Id,\Id)_\sharp (\beta_\plus - \beta)$. By design, $\pi_\plus \in \Pi(\alpha_\plus,\beta_\plus)$, and so
\begin{equation*}
    \OT(\alpha_\plus, \beta_\plus) \leq \int_\cX c \dd \pi_\plus = \int_\cX c \dd \pi = \OT(\alpha,\beta).
\end{equation*}
Taking an infimum over $\beta$ gives the ``$\geq$'' inequality. Existence of a minimizer for the left minimum follows by a standard compactness argument, and our explicit construction above carries this to the right minimum as well.
\end{proof}

\begin{lemma}
\label{lem:c-transform-structure}
For all bounded $\varphi:\cX \to \R$, we have $\varphi^c \in C_b(\cX)$ with $\varphi^c_{\min} = - \varphi_{\max}$, $\varphi^c \leq -\varphi$, and $\varphi^{cc} \geq \varphi$. In particular, $\|\varphi^c\|_\infty \leq \|\varphi\|_\infty$.
\end{lemma}

\begin{proof}
To start, we compute $\varphi^c_{\min} = \inf_{x,y \in \cX} c(x,y) - \varphi(x) = \inf_{x \in \cX} -\varphi(x) = -\varphi_{\max}$, using that $c \geq 0$ with $c(x,x) = 0$. Moreover, for each $y \in \cX$, we have
\begin{equation*}
    \varphi^c(y) = \inf_{x \in \cX} c(x,y) - \varphi(x) \leq c(y,y) - \varphi(y) = -\varphi(y).
\end{equation*}
In particular, we have $\|\varphi^c\|_\infty \leq \|\varphi\|_\infty$. Further, for each $x \in \cX$, we bound
\begin{equation*}
    \varphi^{cc}(x) = \inf_{y \in \cX} c(x,y) + \sup_{x' \in \cX} \varphi(x') - c(x',y) \geq \inf_{y \in \cX} c(x,y) + \varphi(x) - c(x,y) = \varphi(x),
\end{equation*}
as desired. It remains to show continuity of $\varphi^c$. Fixing $y \in \cX$, we may restrict the infimum defining $\varphi^c(y)$ to those $x \in \cX$ satisfying $c(x,y) - \varphi(x) \leq \|\varphi\|_\infty$ and thus $\mathsf{d}(x,y) \leq h^{-1}(2\|\varphi\|_\infty) \eqqcolon R$. Now take any sequence $\{y\}_{n \in \N} \subseteq \cX$ such that $y_n \to y$. Since $\mathsf{d}(y,y_n) \leq 1$ for sufficiently large $n$, we have
\begin{align*}
    \lim_{n\to\infty}|\varphi^c(y) - \varphi^c(y_n)| &= \lim_{n \to \infty}\left|\inf_{\substack{x \in \cX\\ \mathsf{d}(x,y) \leq R}} [\mathsf{d}(x,y)^p - \varphi(x)] - \inf_{\substack{x \in \cX\\ \mathsf{d}(x,y_n) \leq R}}[\mathsf{d}(x,y_n)^p - \varphi(x)]\right|\\
    &= \lim_{n \to \infty}\left|\inf_{\substack{x \in \cX\\ \mathsf{d}(x,y) \leq R + 1}} [\mathsf{d}(x,y)^p - \varphi(x)] - \inf_{\substack{x \in \cX\\ \mathsf{d}(x,y) \leq R + 1}}[\mathsf{d}(x,y_n)^p - \varphi(x)]\right|\\
    &\leq \lim_{n \to \infty} \sup_{\substack{x \in \cX\\ \mathsf{d}(x,y) \leq R + 1}} \left|\mathsf{d}(x,y)^p - \mathsf{d}(x,y_n)^p\right|\\
    &\leq \lim_{n \to \infty} \sup_{\substack{x \in \cX\\ \mathsf{d}(x,y) \leq R + 1}} p (R + 2)^{p-1} |\mathsf{d}(x,y) - \mathsf{d}(x,y_n)|\\
    &\leq p (R + 2)^{p-1} \lim_{n \to \infty} \mathsf{d}(y,y_n) = 0.
\end{align*}
Thus, $\varphi^c$ is also continuous.
\end{proof}

\subsubsection{Proof of Proposition~\ref{prop:alternative-primal-probs}}

We prove a slightly strengthened result.

\begin{lemma}
\label{lem:alternative-primal-probs-extended}
For $\eps \in (0,1]$ and $\mu,\nu \in \cP(\cX)$, we have $\OT^\eps(\mu,\nu) = \min_{(\mu',\nu') \in A}\OT(\mu',\nu')$, where $A$ is any of the following families:
\begin{equation*}
\begin{matrix}
    A_\minus^1 \defeq & A_\minus^2 \defeq  & A_\minus^3 \defeq\\
    \left\{\begin{matrix}(\mu_\minus,\nu_\minus) \in \cM_\plus(\cX)^2:\\ \mu_\minus \leq \mu,\, \nu_\minus \leq \nu\\ \mu_\minus(\cX) = \nu_\minus(\cX) = 1-\eps\end{matrix}\right\}, &\left\{\begin{matrix}(\mu_\minus,\nu_\minus) \in \cM_\plus(\cX)^2:\\ \mu_\minus \leq \mu,\, \nu_\minus \leq \nu\\ \mu_\minus(\cX) = \nu_\minus(\cX) \geq 1-\eps\end{matrix}\right\}, &\left\{\begin{matrix}(\mu_\minus,\nu_\minus) \in \cM_\plus(\cX)^2:\\ \mu_\minus \leq \mu,\, \nu_\minus \leq \nu\\ \mu_\minus, \nu_\minus \geq \mu \land \nu \end{matrix}\right\},\\[7mm]
    A_\plus^1 \defeq & A_\plus^2 \defeq  & A_\plus^3 \defeq\\
    \left\{\begin{matrix}(\mu_\plus,\nu_\plus) \in \cM_\plus(\cX)^2:\\ \mu_\plus \geq \mu,\, \nu_\plus \geq \nu\\ \mu_\plus(\cX) = \nu_\plus(\cX) = 1+\eps\end{matrix}\right\}, & \left\{\begin{matrix}(\mu_\plus,\nu_\plus) \in \cM_\plus(\cX)^2:\\ \mu_\plus \geq \mu,\, \nu_\plus \geq \nu\\ \mu_\plus(\cX) = \nu_\plus(\cX) \leq 1+\eps\end{matrix}\right\}, &\left\{\begin{matrix}(\mu_\plus,\nu_\plus) \in \cM_\plus(\cX)^2:\\ \mu_\plus \leq \mu,\, \nu_\plus \leq \nu\\ \mu_\plus, \nu_\plus \leq \mu \lor \nu \end{matrix}\right\},\\[8mm]
    \multicolumn{3}{c}{A_\tv^1 \defeq \left\{\begin{matrix}(\mu',\nu) \in \cP(\cX) \times \{\nu\} \\ \|\mu' - \mu\|_\tv \leq \eps\end{matrix}\right\}, \quad  A_\tv^2 \defeq \left\{\begin{matrix}(\mu,\nu') \in \{\mu\} \times \cP(\cX): \\ \|\nu' - \nu\|_\tv \leq \eps \end{matrix}\right\}.}
\end{matrix}
\end{equation*}
\end{lemma}

\begin{proof}
Throughout, we abbreviate $W = \OT^\eps(\mu,\nu)$, $W_\plusminus^i = \inf_{(\mu',\nu') \in A_{\plusminus}^i} \OT(\mu',\nu')$ for type $i \in \{1,2,3\}$, and $W_\tv^{i} = \inf_{(\mu',\nu') \in A_{\tv}^i} \OT(\mu',\nu')$ for type $i \in \{1,2\}$.  Clearly, $W = W_\minus^{1}$, since the marginals of $\pi \in \Pi_\eps(\mu,\nu)$ are in bijection with $(\mu_\minus,\nu_\minus) \in A_\minus^1$. Moreover, both infima are achieved by a standard argument. Indeed, $\Pi_\eps(\mu,\nu)$ is compact under the topology of weak convergence of measures, and $\pi \mapsto \int c \dd \pi$ is continuous w.r.t.\ this topology.

\paragraph{Inequality/equality constraints:} To see that $W_\minus^{1} = W_\minus^{2}$, note that any $(\mu_\minus,\nu_\minus) \in A_\minus^2$ with mass $\mu_\minus(\cX) = \nu_\minus(\cX) = m > 1-\eps$ can be scaled down by a factor of $(1-\eps)/m$ to have mass exactly $1-\eps$, retaining feasibility while only decreasing the objective $\OT(\mu_\minus,\nu_\minus)$. Further, any minimizer for $W_\minus^1$ is still a feasible minimizer for $W_\minus^2$.

Similarly, if $(\mu_\plus,\nu_\plus) \in A_\plus^2$ with $\mu_\plus(\cX) = \nu_\plus(\cX) = m < 1+\eps$, then, for any $x_0 \in \cX$, $\mu_\plus + (1+\eps -m)\delta_{x_0}$ and $\nu_\plus + (1+\eps - m)\delta_{x_0}$ have mass exactly $1+\eps$, are still feasible, and can only have a smaller objective (see Fact~\ref{fact:ignore-shared-mass}). Thus, $W_\plus^{1} = W_\plus^{2}$. If $W_\plus^1$ admits a minimizer (to be shown), it is still a feasible minimizer for $W_\plus^2$.

\paragraph{Mass removal/addition equivalence:}  We next show that $W_\minus^1 = W_\tv^1$. Indeed, we have
\begin{align*}
    \min_{\substack{\mu_\minusb,\nu_\minusb \in (1-\eps)\cP(\cX)\\\mu_\minusb \leq \mu,\,\nu_\minusb \leq \nu}} \OT(\mu_\minus,\nu_\minus) &= \min_{\substack{\mu_\minusb\in (1-\eps)\cP(\cX),\,\mu' \in \cP(\cX)\\\mu_\minusb \leq \mu,\,\mu_\minusb \leq \mu'}} \OT(\mu',\nu) \tag{Lemma~\ref{lem:OT-TV-commutativity}}\\
    &= \min_{\substack{\mu' \in \cP(\cX)\\\|\mu' - \mu\|_\tv \leq \eps}} \OT(\mu',\nu). \tag{Fact~\ref{fact:TV-alternate-versions}}
\end{align*}
In particular, any minimizer for $W_\minus^1$ corresponds to a feasible minimizer for $W_\tv^1$. We similarly prove that $W_\plus^1 = W_\tv^1$, equating
\begin{align*}
    \min_{\substack{\mu' \in \cP(\cX)\\\|\mu' - \mu\|_\tv \leq \eps}} \OT(\mu',\nu)     &= \min_{\substack{\mu_\plusb\in (1+\eps)\cP(\cX),\, \mu' \in \cP(\cX)\\\mu \leq \mu_\plusb,\, \mu' \leq \mu_\plusb}} \OT(\mu',\nu) \tag{Fact~\ref{fact:TV-alternate-versions}}\\
    &= \min_{\substack{\mu_\plusb,\nu_\plusb\in (1+\eps)\cP(\cX)\\\mu \leq \mu_\plusb,\, \nu \leq \nu_\plusb}} \OT(\mu_\plus,\nu_\plus). \tag{Lemma~\ref{lem:OT-TV-commutativity}}
\end{align*}
In particular, any minimizer for $W_\tv^1$ corresponds to a minimizer for $W_\plus^1$. Finally, we clearly have $W_\tv^1 = W_\tv^2$ (and existence of minimizers for the latter) by symmetry.\smallskip

\paragraph{Lower envelope:}
We turn to the structure of minimizers, focusing on mass removal. Fixing $(\mu_\minus,\nu_\minus) \in A_\minus^2$ optimal for $W_\minus^2$, we seek to prove that one can find a modified pair $(\mu_\minus',\nu_\minus')$ further satisfying $\mu_\minus',\nu_\minus' \geq \mu \land \nu$ and $\OT(\mu_\minus',\nu_\minus') \leq \OT(\mu_\minus,\nu_\minus)$, implying that $W_\minus^2 = W_\minus^3$.

Fixing base measure $\lambda \defeq \mu + \nu$, write $f_\kappa = \dd \kappa/\dd \lambda$ for the Radon–Nikodym derivative of $\kappa \in \cM_\plus(\cX)$ w.r.t.\ $\lambda$. Write $\alpha \defeq (\mu \land \nu - \mu_\minus)_\plus$ for the piece of mass which violates $\mu_\minus \geq \mu \land \nu$. Next, by adding $(f_\mu - f_{\mu_\minus}) \land (f_{\nu} - f_{\nu_\minus})$ to both $f_{\mu_\minus}$ and $f_{\nu_\minus}$, we can assume without loss of generality that $\alpha = (\mu \land \nu_{\minus} - \mu_{\minus})_\plus$. Indeed, this transformation maintains feasibility and can only decrease the transport cost (see Fact~\ref{fact:ignore-shared-mass}). Moreover, it ensures $f_\nu(x) = f_{\nu_\minus}(x)$ whenever $f_{\mu_\minus}(x) < f_\mu(x)$, and so $f_\alpha = (f_\mu \land f_\nu - f_{\mu_\minus})_\plus = (f_\mu \land f_{\nu_\minus} - f_{\mu_\minus})_\plus$.

Next, take $\pi \in \Pi(\mu_\minus,\nu_\minus)$ to be an optimal coupling for $\OT(\mu_\minus,\nu_\minus)$, and write $\pi(x|y)$ for a regular conditional probability measure such that $\pi(S \times T) = \int_T \pi(S|y) \dd \nu_\minus(y)$ for all measurable $S,T \subseteq \cX$. Define $\beta \in \cM_\plus(\cX)$ by $\beta(\cdot) \defeq \int_\cX \pi(\cdot|y) \dd \alpha(y)$. Since $\alpha \leq \nu_\minus$, we have $\beta \leq \mu_\minus$. Now, take $\mu'_\minus \defeq \mu_\minus - \beta + \alpha \leq \mu$ and define $\pi' \in \Pi(\mu'_\minus,\nu_\minus)$ by
\begin{equation*}
    \pi'(S \times T) \defeq \int_T \pi(S|y) \dd (\nu_\minus - \alpha)(y) + \alpha(S \cap T)
\end{equation*}
for measurable $S,T \subseteq \cX$. Then, $(\mu'_\minus,\nu_\minus) \in A_\minus^2$ and 
\begin{align*}
    \OT(\mu_\minus',\nu_\minus) &\leq \int_{\cX^2} c \dd \pi'\\
    &= \int_{\cX^2} c(x,y) \dd \pi(x|y) \dd (\nu_\minus - \alpha)(y)\\
    &= \OT(\mu_\minus,\nu_\minus) - \int_{\cX^2} c(x,y) \dd \pi(x|y) \dd \alpha (y).
\end{align*}
Since $c(x,y)$ only vanishes if $x=y$, this contradicts optimality of $\mu_\minus$ unless $\pi(\cdot|y) = \delta_y$ for $\alpha$-almost all $y \in \cX$. In this case, we have
\begin{align*}
    \mu_\minus(\supp(\alpha)) = \int_{\cX} \pi(\supp(\alpha)|y) \dd\nu_\minus(y) \geq \int_{\supp(\alpha)} \pi(\supp(\alpha)|y) \dd\nu_\minus(y) = \nu_\minus(\supp(\alpha)).
\end{align*}
By our initial assumption, however, $f_{\nu}(x) = f_{\nu_\minus}(x)$ for $\lambda$-almost all $x \in \supp(\alpha)$. Thus,
\begin{align*}
    \nu_\minus(\supp(\alpha)) &= \nu(\supp(\alpha))\\
    &\geq  (\mu \land \nu)(\supp(\alpha))\\
    &= (\mu \land \nu - \mu_-)(\supp(\alpha)) + \mu_-(\supp(\alpha))\\
    &= \alpha(\cX) + \mu_-(\supp(\alpha)).
\end{align*}
This contradicts optimality unless $\alpha(\cX) = 0$, in which case $\mu_\minus \geq \mu \land \nu$. A symmetric argument allows us to assume $\nu_\minus \geq \mu \land \nu$.

\paragraph{Upper envelope:} Take $(\mu_\minus,\nu_\minus) \in A_\minus^3$ optimal for $W_\minus^3$, and define $\mu_\plus \defeq \mu + \nu - \nu_\minus$ and $\nu_\plus \defeq \nu + \mu - \mu_\minus$. By construction, $\mu_\plus = \mu \lor \nu + \mu \land \nu - \nu_\minus \leq \mu \lor \nu$, and, similarly $\nu_\plus \leq \mu \lor \nu$. Thus, $(\mu_\plus,\nu_\plus) \in A_\plus^3$. Moreover, these are obtained by adding shared mass $(\mu - \mu_\minus) + (\nu - \nu_\minus)$ to the original measures, so their transport cost cannot increase (see Fact~\ref{fact:ignore-shared-mass}). Thus, $\OT(\mu_\plus,\nu_\plus) \leq \OT(\mu_\minus,\mu_\minus) = W_\minus^3 = W_\plus^2 \leq W_\plus^3$. Consequently, $\mu_\plus$ and $\nu_\plus$ are optimal for the $W_\plus^3$ problem, and $W_\plus^3 = W$.
\end{proof}

\subsubsection{Proof of Proposition~\ref{prop:RWp-structural-properties}}
We start with Claim 1. Let $\mu,\nu \in \cP(\cX)$ and $0 < \eps_1 \leq \eps_2 < \|\mu - \nu\|_\tv$ as in the proposition statement. Since $\cX$ is a Polish space, $\mu$ and $\nu$ are tight; in particular, there is a compact set $K \subseteq \cX$ such that $\mu(K),\nu(K) \geq 1-\eps_1$. Letting $\mu' = \frac{1-\eps_1}{\mu(K)}\mu|_K$ and $\nu' = \frac{1 - \eps_1}{\nu(K)}\nu|_K$, we have
\begin{equation*}
    \OT^{\eps_1}(\mu,\nu) \leq \OT(\mu',\nu') \leq h(\diam(K)) < \infty.
\end{equation*}
Now, fix any minimizers $\mu_\minus,\nu_\minus \in (1-\eps_1)\cP(\cX)$ for the mass removal formulation of $\OT^{\eps_1}(\mu,\nu)$. Letting $\mu_\minus' = \frac{1-\eps_2}{1-\eps_1} \mu_\minus$ and $\nu_\minus' = \frac{1-\eps_2}{1-\eps_1}\nu_\minus$, we have
\vspace{-2mm}
\begin{equation*}
    \OT^{\eps_2}(\mu,\nu) \leq \OT(\mu_\minus',\nu_\minus') = \tfrac{1-\eps_2}{1-\eps_1}\OT(\mu_\minus,\nu_\minus) = \tfrac{1-\eps_2}{1-\eps_1}\ROT(\mu,\nu).
\vspace{-1mm}
\end{equation*}
Next, we note that ${\|\mu - \nu\|_\tv} \leq \eps_2$ if and only if there exists $\kappa \in (1-\eps_2)\cP(\cX)$ with $\kappa \leq \mu$ and $\kappa \leq \nu$, which occurs if and only if $\OT^{\eps_2}(\mu,\nu) = 0$. Finally, $\OT^0(\mu,\nu)$ trivially coincides with $\OT(\mu,\nu)$ since $\Pi_0(\mu,\nu) = \Pi(\mu,\nu)$.\smallskip

For Claim 2, we first note that our condition on $h$ implies a triangle inequality for OT.

\begin{lemma}[OT triangle inequality]
\label{lem:OT-triangle-inequality}
If $h$ and $\log \circ \,h \circ \exp$ are convex, then 
\begin{equation*}
    h^{-1}\bigl(\OT(\mu,\nu)\bigr) \leq h^{-1}\bigl(\OT(\mu,\kappa)\bigr) + h^{-1}\bigl(\OT(\kappa,\nu)\bigr).
\end{equation*}
\end{lemma}
\begin{proof}
We employ the standard proof for $\Wp$ (see, e.g., Theorem 7.3 of \citealp{villani2003}) but employ a generalized version of Minkowski's inequality. Taking $\pi_{12} \in \Pi(\mu,\kappa)$ optimal for $\OT(\mu,\kappa)$ and $\pi_{23} \in \Pi(\kappa,\nu)$ optimal for $\OT(\kappa,\nu)$, the gluing lemma (see, e.g., Lemma 7.6 of \citealp{villani2003}), guarantees that there exists $\pi_{123} \in \cP(\cX^3)$ whose marginal w.r.t.\ its first two coordinates is $\pi_{12}$ and w.r.t.\ its last two coordinates is $\pi_{23}$. We then bound
\begin{align*}
    h^{-1}\bigl(\OT(\mu,\nu)\bigr) &\leq h^{-1}\left(\int_\cX h(d(x,z)) \dd \pi_{123}(x,y,z)\right)\\
    &\leq h^{-1}\left(\int_\cX h(d(x,y) + d(y,z)) \dd \pi_{123}(x,y,z)\right)\\
    &\leq h^{-1}\left(\int_\cX h(d(x,y)) \dd \pi_{123}(x,y,z)\right) +  h^{-1}\left(\int_\cX h(d(y,z)) \dd \pi_{123}(x,y,z)\right)\\
    &= h^{-1}\bigl(\OT(\mu,\kappa)\bigr) + h^{-1}\bigl(\OT(\kappa,\nu)\bigr),
\end{align*}
where the third inequality follows by the generalized Minkowski's inequality due to Mulholland (see Theorem 1 of  \citealp{mulholland1949minkowski}).
\end{proof}

We can now prove the claim. Define $W_1 \coloneqq \OT^{\eps_1}(\mu,\kappa)$ and $W_2 \coloneqq \OT^{\eps_2}(\kappa,\nu)$. By Proposition~\ref{prop:alternative-primal-probs}, there exist $\kappa',\kappa'' \in \cP(\cX)$ such that $\OT(\mu,\kappa') \leq W_1$, $\|\kappa-\kappa'\|_\tv \leq \eps_1$, $\OT(\nu,\kappa'') \leq W_2$, and $\|\kappa - \kappa''\|_\tv \leq \eps_2$. By the TV triangle inequality, we further have $\|\kappa' - \kappa''\|_\tv \leq \eps_1 + \eps_2$. Thus,
$\OT^{\eps_1 + \eps_2}(\mu,\kappa'') \leq W_1$, and so, by Lemma~\ref{lem:alternative-primal-probs-extended}, there exists $\gamma \in \cP(\cX)$ such that $\|\mu - \gamma\|_\tv \leq \eps_1 + \eps_2$ and $\OT(\gamma,\kappa'') \leq W_1$. By Lemma~\ref{lem:OT-triangle-inequality}, we have $h^{-1}\bigl(\OT(\gamma,\nu)\bigr) \leq h^{-1}(W_1) + h^{-1}(W_2)$. Applying Lemma~\ref{lem:alternative-primal-probs-extended} once more, we obtain 
\begin{equation*}
    h^{-1}\bigl(\OT^{\eps_1 + \eps_2}(\mu,\nu)\bigr) \leq h^{-1}(W_1) + h^{-1}(W_2) = h^{-1}\bigl(\OT^{\eps_1}(\mu,\kappa)\bigr) + h^{-1}\bigl(\OT^{\eps_2}(\kappa,\nu)\bigr),
\end{equation*} as desired.\qed

\subsubsection{Proof of \cref{thm:RWp-dual}}
\label{prf:RWp-dual}

\paragraph{Strong duality:} We first prove a strong duality result, a bit more general than that in the theorem statement. For distributions $\mu,\nu \in \cP(\cX)$, a measure on the product space $\pi \in \cM_\plus(\cX^2)$, $\varphi,\psi \in C_b(\cX)$, and $\eps > 0$, define
\begin{align*}
    \mathsf{I}_c(\pi) &\coloneqq \int_{\cX^2} c \dd \pi,\\
    \Pi_\plus^\eps &\coloneqq \big\{ \pi \in (1+\eps)\cP(\cX^2) : \pi_1 \geq \mu, \pi_2 \geq \nu \big\},\\
    \mathsf{J}^\eps(\varphi,\psi) &\coloneqq \int_{\cX} \varphi \dd \mu + \int_{\cX} \psi \dd \nu + \eps \left(\varphi_{\min} + \psi_{\min}\right),\\
    \mathsf{J}^\eps_c(\varphi) &\coloneqq \int_{\cX} \varphi \dd \mu + \int_{\cX} \varphi^c \dd \nu - 2 \eps \|\varphi\|_\infty,\\
    \cF_c &\coloneqq \left\{ (\varphi,\psi) \in C_b(\cX)^2 : \varphi(x) + \psi(y) \leq c(x,y),  \forall x,y \in \cX
    \right\},
\end{align*}
and write $\varphi^{cc} \coloneqq (\varphi^c)^c$. The result below recovers strong duality for Theorem~\ref{thm:RWp-dual} when $c = \mathsf{d}^p$.

\begin{lemma}
\label{lem:duality-general}
For $\eps \in (0,1]$ and $\mu,\nu \in \cP(\cX)$, we have
\begin{equation*}
    \OT^\eps(\mu,\nu) = \inf_{\pi \in \Pi_\plus^\eps} \mathsf{I}_c(\pi) = \sup_{(\varphi,\psi) \in \cF_c} \mathsf{J}^\eps(\varphi,\psi) = \sup_{\varphi \in C_b(\cX)} \mathsf{J}^{\eps}_c(\varphi) = \sup_{\varphi \in C_b(\cX):\, \varphi = \varphi^{cc}} \mathsf{J}_c^\eps(\varphi).
\end{equation*}
\end{lemma}

\begin{proof}
The first equality follows from Lemma~\ref{lem:alternative-primal-probs-extended}.
It is straightforward to prove the third and fourth equalities and the ``$\geq$'' direction of the second equality. Indeed, we bound
\begin{align*}
    \inf_{\pi \in \Pi_\plus^\eps} \mathsf{I}_c(\pi) &= \inf_{\substack{\mu',\nu' \in (1+\eps)\cP(\cX)\\\mu' \geq \mu,\, \nu' \geq \nu}} \OT(\mu,\nu)\\
    &= \inf_{\substack{\mu',\nu' \in (1+\eps)\cP(\cX)\\\mu' \geq \mu,\, \nu' \geq \nu}} \sup_{(\varphi,\psi) \in \cF_c} \int_\cX \varphi \dd \mu' + \int_\cX \psi \dd \nu'\\
    &\geq \sup_{(\varphi,\psi) \in \cF_c} \inf_{\substack{\mu',\nu' \in (1+\eps)\cP(\cX)\\\mu' \geq \mu,\, \nu' \geq \nu}}  \int_\cX \varphi \dd \mu' + \int_\cX \psi \dd \nu'\\
    &= \sup_{(\varphi,\psi) \in \cF_c} \mathsf{J}^\eps(\varphi,\psi),
\end{align*}
where the second equality follows by standard Kantorovich duality. Substituting $\phi \gets \varphi^c$ can only increase the objective while preserving feasibility, so we have
\begin{align*}
    \sup_{(\varphi,\psi) \in \cF_c} \mathsf{J}^\eps(\varphi,\psi) &= \sup_{(\varphi,\psi) \in \cF_c} \int_\cX \varphi \dd \mu + \int_\cX \phi \dd \nu + \eps \varphi_{\min} + \eps \phi_{\min}\\
    &= \sup_{\varphi \in C_b(\cX)} \int_\cX \varphi \dd \mu + \int_\cX \varphi^c \dd \nu + \eps \varphi_{\min} + \eps {\varphi^c}_{\min}\\
    &= \sup_{\varphi \in C_b(\cX)} \int_\cX \varphi \dd \mu + \int_\cX \varphi^c \dd \nu - \eps (\varphi_{\max} -  {\varphi}_{\min})\\
    &= \sup_{\varphi \in C_b(\cX)} \int_\cX \varphi \dd \mu + \int_\cX \varphi^c \dd \nu - \eps 2\|\varphi\|_\infty\\
    &= \sup_{\varphi \in C_b(\cX)}\mathsf{J}_c^\eps(\varphi),
\end{align*}
where the third equality uses that $\varphi^c_{\min} = -\varphi_{\max}$ (see Lemma~\ref{lem:c-transform-structure}). The fourth step uses the substitution $\varphi \gets \varphi - \varphi_{\min} - (\varphi_{\max} - \varphi_{\min})/2$, which doesn't change the objective, preserves feasibility, and ensures that $\varphi_{\max} - \varphi_{\min} = 2\|\varphi\|_\infty$. Finally, we may restrict to $\varphi = \varphi^{cc}$ since this preserves feasibility and can only increase the objective (see Lemma~\ref{lem:c-transform-structure}).
\smallskip

It suffices now to show that $\OT^{\eps}(\mu,\nu)  \leq \sup_{\varphi \in C_b(\cX)} \mathsf{J}^{\eps}(\varphi)$. We start from the augmented primal form for $\OT^{\eps}$ given in Proposition~\ref{prop:RWp-augmented-Wp}, defined over the augmented space $\bar{\cX} = \cX \sqcup \{ \bar{x}\}$ with augmented cost $\bar{c}$. In particular, we have
\begin{align*}
    \OT^{\eps}(\mu,\nu) &= \mathrm{OT}_{\bar{c}}(\mu + \eps \delta_{\bar{x}}, \nu + \eps \delta_{\bar{x}}) = \sup_{\bar{\varphi} \in C_b(\bar{\cX})} \int_{\bar{\cX}} \bar{\varphi} \dd(\mu + \eps \delta_{\bar{x}}) + \int_{\bar{\cX}} \bar{\varphi}^{\,\bar{c}} \dd(\nu + \eps \delta_{\bar{x}}),
\end{align*}
where the last equality follows by standard Kantorovich duality for the augmented problem\footnote{The augmented space $\bar{\cX}$ is a Polish space over which $\bar{c}$ is lower-semicontinuous, so strong duality holds.}. Note that each $\bar{\varphi} \in C_b(\bar{X})$ can be decomposed into a pair $(\varphi,a) \in C_b(\cX) \times \R$, where $a = \bar{\varphi}(\bar{x})$. Using this notation, we compute
\begin{equation*}
    \bar{\varphi}^{\,\bar{c}}(y) = \inf_{x \in \bar{\cX}} \bar{c}(x,y) - \bar{\varphi}(x) =
    \begin{cases}
       \varphi^c(y) \land (-a), &y \in \cX\\
       - \varphi_{\max}, &y = \bar{x}
    \end{cases}.
\end{equation*}
Consequently, we have
\begin{equation}
\label{eq:augmented-c-transform-dual}
    \OT^{\eps}(\mu,\nu) =\sup_{\substack{\varphi \in C_b(\cX)\\a \in \R}} \int_\cX \varphi \dd \mu + \eps a + \int_\cX (\varphi^c \land (-a)) \dd \nu - \eps \varphi_{\max}.
\end{equation}
We now consider substituting $\varphi$ with $\varphi' = \varphi \lor (a \land \varphi_{\max})$ within the objective. This can only increase the first integral and does not effect the second term. Moreover, computing
\begin{align*}
    (\varphi')^c(y) = \inf_{x \in \cX} c(x,y) - (\varphi(x) \lor (a \land \varphi_{\max})) = \varphi^c(y) \land (-a \lor -\varphi_{\max}),
\end{align*}
we see that $(\varphi')^c \land (-a) = \varphi^c \land (-a)$, so the third term in \eqref{eq:augmented-c-transform-dual} is unchanged. Finally, we compute $\varphi'_{\max} = \varphi_{\max} \lor (a \land \varphi_{\max}) = \varphi_{\max}$, so the last term of \eqref{eq:augmented-c-transform-dual} is also unchanged. In the case that $\varphi_{\max} \leq a$, we have that $\varphi' = \varphi \lor \varphi_{\max} = \varphi_{\max}$ is constant. Otherwise, we have $\varphi' = \varphi \lor a \geq a$. Restricting to the constant case, we bound
\begin{equation*}
    \sup_{\varphi,a \in \R} \varphi + \eps a + (-\varphi) \land (-a) - \eps \varphi =  0.
\end{equation*}
For the other case, we bound
\begin{align*}
    &\sup_{\substack{\varphi \in C_b(\cX)\\a \in \R: \varphi \geq a}} \int_\cX \varphi \dd \mu + \eps a + \int_\cX (\varphi^c \land (-a)) \dd \nu - \eps \varphi_{\max}\\
    \leq \,&\sup_{\substack{\varphi \in C_b(\cX)\\a \in \R: \varphi \geq a}} \int_\cX \varphi \dd \mu + \eps a + \int_\cX \varphi^c \dd \nu - \eps \varphi_{\max}\\
    = \, &\sup_{\substack{\varphi \in C_b(\cX)}} \int_\cX \varphi \dd \mu + \int_\cX \varphi^c \dd \nu - \eps (\varphi_{\max} - \varphi_{\min}).
\end{align*}
By plugging in $\varphi \equiv 0$, we see that this quantity is always non-negative, and so the previous bound does not come into play. All together, we have
\begin{align*}
    \OT^{\eps}(\mu,\nu) \leq \sup_{\substack{\varphi \in C_b(\cX)}} \int_\cX \varphi \dd \mu + \int_\cX \varphi^c \dd \nu - \eps (\varphi_{\max} - \varphi_{\min}) = \sup_{\varphi \in C_b(\cX)} \mathsf{J}^{\eps}(\varphi),
\end{align*}
as desired.
\end{proof}

\paragraph{Existence of maximizers:} Let $\lambda \in \cP(\cX)$ denote any probability distribution with $\supp(\lambda) = \cX$ (such $\lambda$ always exists because $\cX$ is separable).
To start, we prove that strong duality still holds if the infima defining $\mathsf{J}^\eps$ are relaxed to essential infima w.r.t.\ $\kappa \coloneqq \mu + \nu + \lambda$ and if the dual potentials are uniformly bounded by $R \defeq \OT^{\eps/2}(\mu,\nu)/\eps$. The selection of $\kappa \in \cM_\plus(\cX)$ is somewhat arbitrary; we only use that $\mu$ and $\nu$ are absolutely continuous w.r.t.\ $\kappa$ and that $\supp(\kappa) = \cX$.
Formally, for $\varphi,\psi \in L^1(\kappa)$, we define
\begin{align*}
    \bar{\mathsf{J}}^\eps(\varphi,\psi) &\coloneqq \int_\cX \varphi \dd \mu + \int_\cX \dd \nu + \eps \left(\essinf_\kappa \varphi + \essinf_\kappa \psi\right),\\
    \bar{\cF}_c &\coloneqq \left\{ (\varphi,\psi)\in L^\infty(\kappa) \times L^\infty(\kappa) : 
    \begin{array}{c}
        \varphi(x) + \psi(y) \leq c(x,y),\   \kappa \otimes \kappa \text{ a.e.}\\
        \|\varphi\|_{L^\infty(\kappa)},\|g\|_{L^\infty(\kappa)} \leq R
    \end{array}
    \right\}.
\end{align*}
Note that $R < \infty$ by Proposition~\ref{prop:RWp-structural-properties} since $\eps > 0$. We then have the following.

\begin{lemma}
For $\eps \in (0,1]$ and $\mu,\nu \in \cP(\cX)$, we have $\OT^\eps(\mu,\nu) = \sup_{(\varphi,\psi) \in \bar{\cF}_c} \bar{\mathsf{J}}_\eps(\varphi,\psi)$.
\end{lemma}
\begin{proof}
By Lemma~\ref{lem:duality-general}, we have
\begin{equation*}
    \OT^\eps(\mu,\nu) = \sup_{\varphi \in C_b(\cX)} \mathsf{J}_c^\eps(\varphi) = \sup_{\varphi \in C_b(\cX)} \int_\cX \varphi \dd \mu + \int_\cX \varphi^c \dd \nu - 2\eps \|\varphi\|_\infty.
\end{equation*}
Since this quantity is non-negative, we can restrict to potentials $\varphi \in C_b(\cX)$ for which the objective is non-negative. For such $\varphi$, we have
\begin{equation*}
    \eps \|\varphi\|_\infty \leq  \int_\cX \varphi \dd \mu + \int_\cX \varphi^c \dd \nu - \eps \|\varphi\|_\infty \leq \OT^{\eps/2}(\mu,\nu),
\end{equation*}
and so $\|\varphi\|_\infty \leq R$. Since $\|\varphi^c\|_\infty \leq \|\varphi^c\|_\infty$ (see Lemma~\ref{lem:c-transform-structure}), we have $(\varphi,\varphi^c) \in \bar{\cF}_c$, and so $\OT^{\eps}(\mu,\nu) \leq \sup_{(\varphi,\psi) \in \bar{\cF}_c} \bar{\mathsf{J}}_\eps(\varphi,\psi)$.
\smallskip

For the opposite inequality, fix $(\varphi,\psi) \in \bar{\cF}_c$ and $\pi \in \Pi_\plus^\eps(\mu,\nu)$ with marginals $\mu' = \pi(\cdot \times \cX)$ and $\nu' = \pi(\cX \times \cdot)$ satisfying $\mu',\nu' \leq \mu + \nu$. This ensures that $\pi$ is absolutely continuous w.r.t.\ $\kappa$, and so $\varphi(x) + \phi(y) \leq c(x,y)$ $\pi$-almost everywhere. We thus compute
\begin{align*}
    \bar{\mathsf{J}}^\eps(\varphi,\psi) &= \int_\cX \varphi \dd \mu + \int_\cX \psi \dd \nu + \eps \left(\essinf_\kappa \varphi + \essinf_\kappa \psi \right)\\
    &\leq \int_\cX \varphi \dd \mu' + \int_\cX \psi \dd \nu'\\
    &= \int_{\cX^2} \left[ \varphi(x) + \psi(y) \right] \dd \pi(x,y)\\
    &\leq \int_{\cX^2} c \dd \pi = \mathsf{I}_c(\pi),
\end{align*}
where the first inequality uses that $\mu + \nu$ is absolutely continuous w.r.t.\ $\kappa$ and the second uses that $\varphi(x) + \phi(y) \leq c(x,y)$ $\pi$-almost everywhere.
Combining this with Lemma~\ref{lem:alternative-primal-probs-extended}, we conclude by bounding
\begin{equation*}
\ROT(\mu,\nu) = \inf_{\substack{\pi \in \Pi_\plus^\eps(\mu,\nu) \\ \pi(\cdot \times \cX), \pi(\cX \times \cdot) \leq \mu + \nu}} \mathsf{I}_{c}(\pi) \geq \sup_{(\varphi,\psi) \in \bar{\cF}_c} \bar{\mathsf{J}}^\eps(\varphi,\psi).\qedhere
\end{equation*}
\end{proof}

To show that the supremum is achieved, we will prove that $\bar{\cF}_c$ is compact and that $\bar{\mathsf{J}}_\eps$ is upper-semicontinuous with respect to an appropriate topology. In what follows, we consider $L^2(\kappa) \times L^2(\kappa)$ as a Banach space under the norm $\|(\varphi,\psi)\| \coloneqq \|\varphi\|_{L^2(\kappa)} \lor \|\psi\|_{L^2(\kappa)}$ and write $C_0(\cX) \subseteq C_b(\cX)$ for the space of continuous functions with compact support.\vspace{-1mm}

\begin{lemma}
\label{lem:dual-feasible-set-compactness}
Viewed as a subset of $L^2(\kappa) \times L^2(\kappa)$, $\bar{\cF}_c$ is weakly compact.
\end{lemma}\vspace{-1mm}
\begin{proof}
Since $L^2(\kappa)$ is reflexive, so is the space $L^2(\kappa) \times L^2(\kappa)$. Thus, $\bar{\cF}_c$ is weakly compact if and only if it is closed and bounded under this norm. Boundedness is trivial, since $\|(\varphi,\psi)\| \leq R < \infty$ for all $(\varphi,\psi) \in \bar{\cF}_c$. Moreover, by the density of $C_0(\cX)$ in $L^1(\kappa)$, we have that $\|\varphi\|_{L^\infty(\kappa)} \leq R$ if and only if $\int \varphi f \dd \kappa \leq R$ for all $f \in C_0(\cX)$ such that $\|f\|_{L^1(\kappa)} \leq 1$. Since the map $\varphi \mapsto \int \varphi f \dd \kappa$ is weakly continuous for all such $f$, the intersection of these constraints defines a weakly closed set. To see that the final constraint preserves closedness, take feasible $\{(\varphi_n,\psi_n)\}_{n \in \N}$ weakly converging to some $(\varphi,\psi) \in L^2(\kappa) \times L^2(\kappa)$. Since the countable union of null sets is a null set, there exists $N \subseteq \cX$ such that $\varphi_n(x) + \psi_n(y) \leq c(x,y)$ for all $x,y \in \cX \setminus N$ and $n \in \N$. Thus, fixing $y_0 \in \cX \setminus N$ and $f \in C_0(\cX)$ with $f \geq 0$, we have
\begin{align*}
    0 \leq \int_\cX \big(c(x,y_0) &- \varphi_n(x) - \psi_n(y_0)\big) f(x) \dd \kappa(x) \to \int_\cX \big(c(x,y) - \varphi(x) - \psi(y_0)\big) f(x) \dd \kappa(x)
\end{align*}
as $n \to \infty$. Taking an infimum over $f$ gives that $\varphi(x) + \psi(y_0) \leq c(x,y)$ for $\kappa$-almost all $x$. A symmetric argument shows that for fixed $x_0 \in \cX \setminus N$, the inequality holds for $\kappa$-almost all $y$. Combining, we have the inequality for $\kappa$-almost all $x$ and $y$, as desired.
\end{proof}

\begin{lemma}
\label{lem:dual-objective-semicontinuity}
The objective $\bar{\mathsf{J}}^\eps$ is weakly upper-semicontinuous over $\bar{\cF}_c$.
\end{lemma}\vspace{-1mm}
\begin{proof}
Let $\{(\varphi_n,\psi_n)\}_{n \in \N} \subseteq \bar{\cF}_c$ converge weakly to some $(\varphi,\psi) \in L^2(\kappa) \times L^2(\kappa)$. Thus, $\int_\cX \varphi_n \dd \mu = \int_\cX \varphi_n \frac{\dd \mu}{\dd \kappa} \dd \kappa \to \int_\cX \varphi \dd \mu$, since $\frac{\dd \mu}{\dd \kappa}$ is bounded and measurable, belonging to $L^2(\kappa)$. The same argument gives $\int_\cX \psi_n \dd \mu \to \int_\cX \psi \dd \mu$. Mirroring the proof of Lemma~\ref{lem:dual-feasible-set-compactness}, we use the density of $C_0(\cX)$ in $L^1(\kappa)$ to bound
\begin{align*}
\essinf_\kappa(\varphi) = \mspace{-6mu}\inf_{\substack{f \in C_0(\cX)\\ f \geq 0,\, \|f\|_{L^1(\kappa)} \leq 1}} \int_\cX \varphi f \dd \kappa= \mspace{-6mu}\inf_{\substack{f \in C_0(\cX)\\ f \geq 0,\, \|f\|_{L^1(\kappa)} \leq 1}} \lim_{n \to \infty} \int_\cX \varphi_n f \dd \kappa
\geq \limsup_{n \to \infty} \: \essinf_\kappa(\varphi_n).\mspace{-4mu}
\end{align*}
The same argument gives $\essinf_\kappa(\psi) \geq \limsup_{n \to \infty} \essinf_{\kappa}(\psi_n)$. Combining the above we have $\bar{\mathsf{J}}^\eps(\varphi,\psi) \geq \limsup_{n \to \infty} \bar{\mathsf{J}}^\eps(\varphi_n,\psi_n)$, as desired.
\end{proof}

Combining Lemmas~\ref{lem:dual-feasible-set-compactness} and \ref{lem:dual-objective-semicontinuity}, we find there exists $(\varphi_0,\psi_0) \in \bar{\cF}_c$ such that $\ROT(\mu,\nu) = \bar{\mathsf{J}}^\eps(\varphi_0,\psi_0)$. By modifying $\varphi_0$ on a set of $\kappa$-measure 0 (leaving the dual objective unchanged), we may assume that $\|\varphi_0\|_\infty = \|\varphi_0\|_{L^\infty(\kappa)}$, and thus Lemma~\ref{lem:c-transform-structure} implies that $\varphi_0^c$ and $\varphi_0^{cc}$ are continuous and bounded.
As before, taking $c$-transforms to obtain $(\varphi_0^{cc},\varphi_0^c)$ maintains dual feasibility and the objective value, so this pair is still maximizing. Finally, since these two potentials are continuous and $\supp(\kappa) = \cX$, we can substitute $\bar{\mathsf{J}}^\eps$ with $\mathsf{J}^\eps$. That is, we have $
    \ROT(\mu,\nu) = \bar{\mathsf{J}}^\eps(\varphi_0^{cc},\varphi_0^c) = \mathsf{J}^\eps(\varphi_0^{cc},\varphi_0^c) = \sup_{(\varphi,\psi) \in \cF_c} \mathsf{J}^\eps(\varphi,\psi)$,
as desired.
\vspace{-1mm}

\paragraph{Structure of maximizers:}
Take $\varphi \in C_b(\cX)$ maximizing \eqref{eq:RWp-dual}.
Fix any $\mu_\minus = \mu - \alpha$ and $\nu_\minus = \nu - \beta$ optimal for the mass removal formulation of $\ROT(\mu,\nu)$, where $\alpha,\beta \in \eps\cP(\cX)$ satisfy $\alpha \leq \mu$ and $\beta \leq \nu$. Of course, $\mu + \beta$ and $\nu + \alpha$ are then optimal for the mass addition formulation of $\ROT(\mu,\nu)$. By strong duality, $\mu + \beta,\nu + \alpha$ and $\varphi,\varphi^c$ must be a minimax equilibrium, and so
\begin{align}
    \ROT(\mu,\nu) &= \int_\cX \varphi \dd \mu + \int_\cX \varphi^c \dd \nu + \int_\cX \varphi \dd \beta + \int_\cX \varphi^c \dd \alpha \label{eq:dual-maximizers-1}\\
    &=\int_\cX \varphi \dd \mu + \int_\cX \varphi^c \dd \nu + \eps \varphi_{\min} - \eps \varphi_{\max} \label{eq:dual-maximizers-2}.
\end{align}
By Lemma~\ref{lem:c-transform-structure}, $\varphi^c_{\min} = - \varphi_{\max}$ (and minimizers of $\varphi^c$ correspond to maximizers of $\varphi$), and so \eqref{eq:dual-maximizers-1} is strictly less than \eqref{eq:dual-maximizers-2} unless $\supp(\beta) \subseteq \argmin(\varphi)$ and $\supp(\alpha) \subseteq \argmax(\varphi)$.

\begin{remark}[Optimal perturbations]
The above suggests taking the perturbations as $\alpha = \mu|_{\argmax(\varphi)}$ and $\beta = \nu|_{\argmin(\varphi)}$, but we cannot do so in general. Indeed, consider the case where $\mu$ and $\nu$ are uniform discrete measures on $n$ points and $\eps$ is not a multiple of $1/n$. Issues with that approach also arise when $\mu$ and $\nu$ are both supported on $\argmax(\varphi)$ and the optimal $\mu_\minus$ satisfies $\mu_\minus \geq \mu \land \nu$.  
\end{remark}

\subsubsection{Proof of Proposition~\ref{prop:loss-trimming}}
Let $\mu,\nu \in \cP(\cX)$. For this result, we will apply Sion's minimax theorem to the mass-removal formulation of $\ROT$. Mirroring the proof of \cref{thm:RWp-dual}, we compute
\begin{align*}
    \ROT(\mu,\nu) &= \inf_{\substack{\mu_\minusb,\nu_\minusb \in (1-\eps)\cP(\cX)\\ \mu_\minusb \leq \mu,\, \nu_\minusb \leq \nu}} \sup_{\substack{\varphi,\psi \in C_b(\cX)\\ \varphi(x) + \psi(y) \leq c(x,y) }}\int_\cX \varphi \dd \mu_\minus +  \int_\cX \psi \dd \nu_\minus\\
    &=  \sup_{\substack{\varphi,\psi \in C_b(\cX)\\ \varphi(x) + \psi(y) \leq c(x,y)}} \left( \inf_{\substack{\mu_{\minusb} \in (1-\eps)\cP(\cX) \\ \mu_{\minusb} \leq \mu}} \int_\cX \varphi \dd \mu_{\minus} + \inf_{\substack{\nu_{\minusb} \in (1-\eps)\cP(\cX) \\ \nu_\minusb \leq \nu}} \int_\cX \psi \dd \nu_\minus \right)\\
    &=  \sup_{\varphi \in C_b(\cX)} \left(\inf_{\substack{\mu_\minusb \in (1-\eps)\cP(\cX) \\ \mu_\minusb \leq \mu}} \int_\cX \varphi \dd \mu_\minus + \inf_{\substack{\nu_\minusb \in (1-\eps)\cP(\cX) \\ \nu_\minusb \leq \nu}} \int_\cX \varphi^c \dd \nu_\minus\right).
\end{align*}
Here we have used that the infimum constraint set is convex and compact w.r.t.\ the topology of weak convergence, that the supremum constraint set is convex, and that the objective is bilinear and continuous in both arguments (with the function space equipped with the sup-norm).
When $\mu$ and $\nu$ are uniform distributions over $n$ points and $\eps$ is a multiple of $1/n$, this further simplifies to
\begin{align*}
    \ROT(\mu,\nu) &=  \sup_{\varphi \in C_b(\cX)} \left(\min_{\substack{S \subseteq \supp(\mu)\\ |S| = (1-\eps)n}} \frac{1}{n}\sum_{x \in S} \varphi(x) + \min_{\substack{T \subseteq \supp(\nu)\\ |T| = (1-\eps)n}} \frac{1}{n}\sum_{y \in T} \varphi^c(y) \right).\qed
\end{align*}
\vspace{-1mm}

\subsection{Proofs for \cref{sec:robustness}}
\label{subsec:prfs-robustness}

In what follows, resilience of a random variable refers to resilience of its probability law. Mean resilience refers to taking $\mathsf{D} = \mathsf{D}_\mathrm{mean}$.

\subsubsection{Proof of \cref{thm:concrete-minimax-risk-bounds}}
\label{prf:concrete-minimax-risk-bounds}
We begin by showing that resilience w.r.t.\ $\Wp$ is implied by standard mean resilience of the $p$th power of the metric $\mathsf{d}$.

\begin{lemma}
\label{lem:Wp-resilience-from-mean-resilience}
Fix $X \sim \mu \in \cP(\cX)$ such that $\E[\mathsf{d}(X,x_0)^p] \leq \sigma^p$ for some $\sigma \geq 0$ and $x_0 \in \cX$. Suppose further that $\mathsf{d}(X,x_0)^p$ is $(\rho,\eps)$-resilient in mean for some $\rho \geq 0$ and $0 \leq \eps < 1$. Then $\mu \in \cW_p(2\rho^{1/p} + 2\eps^{1/p}\sigma,\eps)$.
\end{lemma}
\begin{proof}
Let $\eps > 0$ (noting that the result is trivial when $\eps = 0$), and fix any $\nu \leq \frac{1}{1-\eps}\mu$. Writing $\mu = (1-\eps)\nu + \eps \alpha$ for the appropriate $\alpha \in \cP(\cX)$ and taking $\tau \coloneqq \eps \lor (1-\eps)$, we have
\begin{align*}
    \Wp(\mu,\nu) &= \Wp\bigl((1-\eps)\nu + \eps \alpha,\nu\bigr)\\
    &\leq \Wp(\eps\alpha, \eps\nu) \tag{\Cref{fact:ignore-shared-mass}}\\
    &= \eps^{\frac{1}{p}} \Wp(\alpha,\nu) \tag{homogeneity of $\Wp^p$}\\
    &\leq \eps^{\frac{1}{p}} \bigl(\Wp(\alpha,\delta_{x_0}) + \Wp(\delta_{x_0},\nu)\bigr) \tag{triangle inequality for $\Wp$}\\
    &= \eps^{\frac{1}{p}} \left(\E_\alpha[\mathsf{d}(Y,x_0)^p]^{\frac{1}{p}} + \E_\nu[\mathsf{d}(Y,x_0)^p]^{\frac{1}{p}}\right) \tag{definition of $\Wp$}\\
    &\leq 2\eps^{\frac{1}{p}} \sup_{\substack{\beta \in \cP(\cX), \beta \leq \frac{1}{1-\tau} \mu}} \E_\beta[\mathsf{d}(Y,x_0)^p]^{\frac{1}{p}}.
\end{align*}
Now, fix any $\beta \leq \frac{1}{1-\tau}\mu$ and write $\bar{\rho} \coloneqq (1 \lor \frac{1-\eps}{\eps})\rho$. We have
\begin{align*}
    \E_\beta[d(Y,x_0)^p] &\leq \bigl|\E_\beta[d(Y,x_0)^p] - \E_\mu[d(X,x_0)^p]\bigr| + \E_\mu[d(X,x_0)^p] \tag{triangle inequality}\\
    &\leq \bar{\rho} + \E_\mu[d(X,x_0)^p] \tag{Lemma 10 of \citealp{steinhardt2018resilience}}\\
    &\leq \bar{\rho} + \sigma^p. \tag{moment bound for $\mu$}
\end{align*}
Combining this with the previous bound, we obtain
\begin{align*}
    \Wp(\mu,\nu) &\leq 2\eps^{\frac{1}{p}} \left(\bar{\rho} + \sigma^p\right)^{\frac{1}{p}}
    \leq 2\eps^{\frac{1}{p}}(\bar{\rho}^{\frac{1}{p}} + \sigma) %
    \leq 2 \rho^{\frac{1}{p}} + 2 \eps^{\frac{1}{p}}\sigma, %
\end{align*}
where the last step is by definition of $\bar{\rho}$. Thus, $\mu$ is $(2\rho^{1/p} + 2\eps^{1/p}\sigma,\eps)$-resilient w.r.t.\ $\Wp$.
\end{proof}
\vspace{-3mm}

We now prove risk bounds for \cref{thm:concrete-minimax-risk-bounds}, beginning with the class $\cG_q(\sigma)$ for $q \geq p$. If $X \sim \mu \in \cG_q(\sigma)$, then there exists $x_0 \in \cX$ such that $\E\left[\big(\mathsf{d}(x_0,X)^p\big)^{q/p}\right] = \E\big[\mathsf{d}(x_0,X)^q\big] \leq \sigma^q = (\sigma^p)^{q/p}$. By Lemmas C.3 and E.2 of \cite{zhu2019resilience}, $\mathsf{d}(x_0,X)^p$ is thus $\bigl(O(\sigma^p \eps^{1-p/q}),\eps\bigr)$-resilient in mean for all $0 \leq \eps \leq 0.99$. Noting that $\cG_q(\sigma) \subseteq \cG_p(\sigma)$, Lemma~\ref{lem:Wp-resilience-from-mean-resilience} gives $\mu \in \cW_p\bigl(O(\sigma \eps^{1/p-1/q}),\eps\bigr)$ for all $0 \leq \eps \leq 0.99$. Lemma~\ref{prop:minimax-risk} then implies the upper risk bound, while the lower bound was provided in the discussion following the theorem.

Next, we fix $\cX = \R^d$ and consider $X \sim \mu \in \cP(\R^d)$ such $\E[|\langle \theta, X - \E[X] \rangle|^q] \leq \sigma^q$ for some $q > p$ (capturing $\cG_\mathrm{cov}(\sigma)$ as a special case when $q = 2$). By Lemmas 5 of \cite{nietert2022sliced}, we then have $\mu \in \cG_q\bigl(O(\sqrt{1 + d/q} \, \sigma)\bigr)$, implying the desired risk bound. For the lower bound, consider the pair of distributions $\mu = \delta_0$ and $\nu = (1-\eps)\delta_0 + \eps \cN(0,\sigma^2/\eps)$. By design, for $p < 2$, we have $\mu,\nu \in \cG_\mathrm{cov}(\sigma)$ and $\|\mu - \nu\|_\tv \leq \eps$, so $\Wp(\mu,\nu) = \sigma \eps^{1/p - 1/2} \E_{Z \sim \cN(0,I_d)}[\|Z\|^p]^{1/p} = \Omega (\sigma \sqrt{d}\eps^{1/p - 1/2})$, which yields a matching lower risk bound for $\cG_\mathrm{cov}(\sigma)$.

Finally, let $X \sim \mu \in \cG_\mathrm{subG}(\sigma)$. By Proposition 2.5.2 of \cite{vershynin2018}, we have that $\E[|\langle \theta, X - \E[X] \rangle|^q]^{1/q} \lesssim \sqrt{q} \sigma$ for all $q \geq 1$. Taking $q = p \lor \log(1/\eps)$, we obtain $\mu \in \cG_{p \lor \log(1/\eps)}\bigl(\sqrt{d + p \lor \log(1/\eps)}\bigr)$, implying that $R_{p,\infty}(\cG_\mathrm{subG}(\sigma),\eps)$ is bounded by\vspace{-1mm}
\begin{align*}
    \sigma \sqrt{d + p \lor \log\big(\tfrac{1}{\eps})} \, \eps^{\frac{1}{p} - \frac{1}{p \lor \log(1/\eps)}} = \sigma \sqrt{d + p \lor \log\big(\tfrac{1}{\eps}\big)} \, \left(1 \land e\eps^{\frac{1}{p}}\right) \lesssim \sigma \sqrt{d + p + \log\big(\tfrac{1}{\eps}\big)} \, \eps^{\frac{1}{p}},\vspace{-1mm}
\end{align*}
for all $0 \leq \eps \leq 0.49$. For the lower bound, first consider the pair of distributions $\mu = \delta_0$ and $\nu = (1-\eps)\delta_0 + \eps \delta_{x}$ for some $x \in \R^d$ with $\|x\| = \sigma \sqrt{\log(1/\eps)}$. By design, we have $\mu,\nu  \in \cG_\mathrm{subG}(\sigma)$ and $\|\mu - \nu\|_\tv \leq \eps$, so $\Wp(\mu,\nu) = \sigma \sqrt{\log(1/\eps)} \, \eps^{1/p}$ serves as a lower bound. Similarly, the pair $\mu = \delta_0$ and $\nu = (1-\eps)\delta_0 + \eps \cN(0,\sigma)$ gives a lower bound of $\Wp(\mu,\nu) = \sigma \eps^{1/p} \E_{Z \sim \cN(0,I_d)}[\|Z\|^p]^{1/p} = \Omega(\sigma \sqrt{d + p} \, \eps^{1/p})$. Combining these two bounds gives the desired lower risk bound of $\Omega\big(\sigma \sqrt{d + p + \log(1/\eps)} \, \eps^{1/p}\big)$.\qed

\subsubsection{Proof of \cref{thm:finite-sample-minimax-risk}}
\label{prf:finite-sample-minimax-risk}
For the upper bound, it remains to prove Lemma~\ref{lem:RWp-modulus}.
\medskip

\begin{proof}[Proof of Lemma~\ref{lem:RWp-modulus}]
Let $0 \leq \eps \leq 0.99$, fix $\mu,\nu \in \cW_p(\rho,\eps)$, and take $\mu_\minus,\nu_\minus \in (1-\eps)\cP(\cX)$ with $\mu_\minus \leq \mu$ and $\nu_\minus \leq \nu$ optimal for the mass-removal formulation of $\RWp(\mu,\nu)$. Then, \vspace{-2mm}
\begin{align*}
    \Wp(\mu,\nu) &\leq \Wp\bigl(\mu,\tfrac{1}{1-\eps}\mu_\minus\bigr) + \Wp\bigl(\tfrac{1}{1-\eps}\mu_\minus,\tfrac{1}{1-\eps}\nu_\minus\bigr) + \Wp\bigl(\tfrac{1}{1-\eps}\nu_\minus\bigr) \leq \left(\frac{1}{1-\eps}\right)^{\!\frac{1}{p}} \RWp(\mu,\nu) + 2\rho,\vspace{-3mm}%
\end{align*}
implying the lemma.
\end{proof}

For the lower bound, it remains to prove that $R_{p,\infty}(\cG,\eps/4) \leq 8 R_{p,n}(\cG,\eps)$. To see this, take $\mu \in \cG$, fix any distribution $\tilde{\mu}$ with $\|\mu - \tilde{\mu}\|_\tv \leq \eps/4$, and let $\mathsf{T}_n$ be a minimax optimal $n$-sample estimator achieving $R_{p,n}(\cG,\eps)$. By the coupling formulation of the TV distance, there exists $\pi \in \Pi(\mu,\tilde{\mu})$ such that $(X,\tilde X)\sim \pi$ satisfy $\PP(X \neq \tilde{X}) \leq \eps/4$. Consequently, the product distribution $\pi^{\otimes n}$ is a coupling of $\PP_n \defeq \mu^{\otimes n}$ and $\tilde{\PP}_n \defeq \tilde{\mu}^{\otimes n}$ such that the contamination fraction $\frac{1}{n}\sum_{i=1}^n \mathds{1} {\bigl\{X_i \neq \tilde{X}_i\bigr\}}$ has expected value at most $\eps/4$. By Markov's inequality, there exists an event $E$ with $\pi^{\otimes n}(E) \geq 3/4$ such that $\frac{1}{n}\sum_{i=1}^n \mathds{1} {\bigl\{X_i \neq \tilde{X}_i\bigr\}} \leq \eps$ under $\pi$ conditioned on $E$. Taking $\tilde{\PP}'_n$ to be the right marginal of $\pi^{\otimes n}$ conditioned on $E$, we have $\tilde{\PP}_n' \in \cM^{\mathrm{adv}}(\mu,\eps)$ by construction, and so $\E_{\tilde{\PP}_n'}\bigl[\Wp\bigl(\mathsf{T}_n(\tilde{\mu}_n),\mu)\bigr)\bigr] \leq R_{n,p}(\cG,\eps)$. Thus, by Markov's, $\Wp\bigl(\mathsf{T}_n(\tilde{\mu}_n),\mu)\bigr) \leq 4R_{n,p}(\cG,\eps)$ with probability at least $3/4$ under $\tilde{\PP}_n'$. Combining with the previous bound we have that $\Wp\bigl(\mathsf{T}_n(\tilde{\mu}_n),\mu)\bigr) \leq 4R_{n,p}(\cG,\eps)$ with probability at least $3/4 \cdot 3/4 > 1/2$ under $\tilde{\PP}_n$.

Now, set $\mathsf{T}(\tilde{\mu})$ to be any $\nu \in \cG$ such that $\Wp\bigl(\mathsf{T}_n(\tilde{\mu}_n),\nu)\bigr) \leq 4R_{n,p}(\cG,\eps)$ with probability greater than 1/2 under $\tilde{\PP}_n$ (note that such $\nu$ can be chosen as a function only of $\tilde{\mu}$). By a union bound and the triangle inequality, $\Wp\bigl(\mu,\mathsf{T}(\tilde{\mu})\bigr) \leq 8R_{n,p}(\cG,\eps)$ with positive probability under $\tilde{\PP}_n$. Since this inequality involves no random variables, it must hold unconditionally, thus bounding $R_{p,\infty}(\cG,\eps/4) \leq 8R_{n,p}(\cG,\eps)$.\qed

\subsubsection{Proof of Corollary~\ref{cor:concrete-finite-sample-minimax-risk-bounds}}
\label{prf:concrete-finite-sample-minimax-risk-bounds}
Given \cref{thm:finite-sample-minimax-risk} and \cref{thm:concrete-minimax-risk-bounds}, it remains to show that, for $\cG$ as in the corollary statement and $\mu \in \cG$, we have\vspace{-1mm}
\begin{equation*}
    \E\left[\Wp(\hat{\mu}_n,\mu)\right] \leq C_{p,q,d}\, n^{-1/d} \quad \text{ and } \quad R_{p,n}(\cG,0) \geq c_{p,q,d}\, n^{-1/d},\vspace{-1mm}
\end{equation*}
where dependence on $q$ only appears for $\cG = \cG_q(\sigma)$. The upper bound follows by Theorem 3.1 of \cite{lei2020convergence} with the substitution of $q$ with $2p$ for $\cG = \cG_\mathrm{subG}(\sigma)$ and $2$ for $\cG = \cG_\mathrm{cov}(\sigma)$. The lower bound follows by Theorem 1 of \cite{nilesweed22minimax} with smoothness parameter $s$ taken as 0 (noting that their asymptotic notation hides dependence on the constants we specify here).\qed

\subsubsection{Proof of \cref{thm:robust-distance-estimation}}
\label{prf:robust-distance-estimation}

Fix $0 \leq \eps < 1/3$ and $\tau = 1 - (1-3\eps)^{1/p} \in [3\eps/p,3\eps]$ as in the theorem statement. Take $\mu,\nu \in \cW_p(\rho,3\eps)$ with empirical measures $\hat{\mu}_n$ and $\hat{\nu}_n$, respectively, and let $\tilde{\mu}_n,\tilde{\nu}_n \in \cP(\cX)$ be such that $\|\tilde{\mu}_n - \hat{\mu}_n\|_\tv, \|\tilde{\nu}_n - \hat{\nu}_n\|_\tv \leq \eps$. To begin, we bound $\RWp(\tilde{\mu}_n,\tilde{\nu}_n)$ from below. By Proposition~\ref{prop:RWp-structural-properties}, we have $\RWp(\tilde{\mu}_n,\tilde{\nu}_n) \geq \Wp^{3\eps}(\mu,\nu) - \Wp(\mu,\hat{\mu}_n) - \Wp(\nu,\hat{\nu}_n)$. Then, by resilience of $\mu$ and $\nu$ and Proposition~\ref{prop:alternative-primal-probs}, we bound
\begin{align*}
    \Wp^{3\eps}(\mu,\nu) &= (1-\tau)\inf_{\substack{\mu',\nu' \in \cP(\cX)\\ \mu' \leq \frac{1}{1-3\eps}\mu,\, \nu' \leq \frac{1}{1-3\eps}\nu}} \Wp(\mu',\nu')\\
    &\geq (1-\tau) \bigl(\Wp(\mu,\nu) - 2\rho \bigr)\\
    &\geq (1-\tau) \Wp(\mu,\nu) - 2\rho.
\vspace{-1mm}
\end{align*}
Combining, we obtain $\RWp(\tilde{\mu}_n,\tilde{\nu}_n) \geq (1-\tau)\Wp(\mu,\nu) - 2\rho - \Wp(\mu,\hat{\mu}_n) - \Wp(\nu,\hat{\nu}_n)$.

For the upper bound, we note that by Lemma~\ref{lem:alternative-primal-probs-extended} there exist $\tilde{\mu},\tilde{\nu} \in \cP(\cX)$ such that $\|\tilde{\mu} - \mu\|_\tv \leq \eps$, $\|\tilde{\nu} - \nu\|_\tv \leq \eps$, $\Wp(\tilde{\mu},\tilde{\mu}_n) \leq \Wp(\mu,\hat{\mu}_n)$, and $\Wp(\tilde{\nu},\tilde{\nu}_n) \leq \Wp(\nu,\hat{\nu}_n)$. Hence, Proposition~\ref{prop:RWp-structural-properties} gives that
\begin{equation*}
    \RWp(\tilde{\mu}_n,\tilde{\nu}_n) \leq \RWp(\tilde{\mu},\tilde{\nu}) + \Wp(\tilde{\mu},\tilde{\mu}_n) + \Wp(\tilde{\nu},\tilde{\nu}_n) \leq \RWp(\tilde{\mu},\tilde{\nu}) + \Wp(\mu,\hat{\mu}_n) + \Wp(\nu,\hat{\nu}_n).
\end{equation*}
Defining the midpoint distributions $\bar{\mu} = \frac{1}{1-\|\tilde{\mu} - \mu\|_\tv} \tilde{\mu} \land \mu$ and $\bar{\nu} = \frac{1}{1-\|\tilde{\nu} - \nu\|_\tv} \tilde{\nu} \land \nu$, we have
\begin{align*}
    \RWp(\tilde{\mu},\tilde{\nu}) &= (1-\eps)^{\frac{1}{p}} \inf_{\substack{\mu',\nu' \in \cP(\cX) \\ \mu' \leq \frac{1}{1-\eps}\tilde{\mu},\, \nu' \leq \frac{1}{1-\eps}\tilde{\nu}}} \Wp(\mu',\nu')\\
    &< \Wp(\bar{\mu},\bar{\nu})\\
    &\leq \Wp(\mu,\nu) + \Wp(\bar{\mu},\mu) + \Wp(\bar{\nu},\nu)\\
    &\leq \Wp(\mu,\nu) + 2\rho.
\end{align*}
All together, we obtain
\begin{equation*}
    \RWp(\tilde{\mu}_n,\tilde{\nu}_n) \leq \Wp(\mu,\nu) + 2\rho + \Wp(\mu,\hat{\mu}_n) + \Wp(\nu,\hat{\nu}_n).
\end{equation*}
Combining the upper and lower bounds gives the theorem.
\qed

\begin{remark}[Breakdown point]
For $\mu,\nu \in \cP(\cX)$, take minimizers $\mu_\minus = \mu - \alpha$ and $\nu_\minus = \nu - \beta$ for the $\Wp^{3\eps}(\mu,\nu)$ mass removal problem, where $\alpha,\beta \in 3\eps \cP(\cX)$ satisfy $\alpha \leq \mu$ and $\beta \leq \nu$. Defining $\tilde{\mu} \coloneqq \mu - \alpha/3 + \beta/3$ and $\tilde{\nu} \coloneqq \nu - \beta/3 + \alpha/3$, we find that
\begin{equation*}
\RWp(\tilde{\mu},\tilde{\nu}) \leq \Wp(\mu - 2\alpha/3 + \beta/3, \nu - 2\beta/3 + \alpha/3) \leq \Wp(\mu - \alpha, \nu - \beta) = \Wp^{3\eps}(\mu,\nu).
\end{equation*}
Since $\|\tilde{\mu} - \mu\|_\tv, \|\tilde{\nu} - \nu\|_\tv \leq \eps$, we cannot obtain meaningful robust estimation guarantees when $\eps \geq 1/3$ and $\RWp(\tilde{\mu},\tilde{\nu}) = 0$.
\end{remark}

\subsubsection{Proof of Corollary~\ref{cor:asymptotic-consistency}}
Since $\eps_n = o(\tau_n)$ as $n \to \infty$, there exists some $n_0 \in \N$ such that $\tau_n \geq \eps_n$ for all $n \geq n_0$. Since $\mu,\nu \in \cP_q(\cX)$ for $q > p$, there exists $\sigma < \infty$ such that $\mu,\nu \in \cG_q(\sigma)$, and $\Wp(\mu,\nu) < \infty$. By the proof of \cref{thm:concrete-minimax-risk-bounds}, we then have $\mu \in \cW_p\bigl(O(\sigma \eps^{1/p-1/q}),\eps\bigr)$ for $0 \leq \eps \leq 0.99$.
Thus, \cref{thm:robust-distance-estimation} gives that\vspace{-1mm}
\begin{equation*}
    |\Wp^{\tau_n}(\tilde{\mu}_n,\tilde{\nu}_n) - \Wp(\mu,\nu)| \leq O(\sigma \tau_n^{1/p-1/q}) + 3\tau_n\Wp(\mu,\nu) + \Wp(\hat{\mu}_n,\mu) + \Wp(\hat{\nu}_n,\nu) \to 0\vspace{-1mm}
\end{equation*}
almost surely as $n \to \infty$.\qed

\subsubsection{Proof of Proposition~\ref{prop:two-sample-and-independence-testing}}
By the proof of \cref{thm:robust-distance-estimation}, we have, for $\mu,\nu \in \cW_p(\rho,3\eps)$, that
\begin{align*}
    \Wp(\mu,\nu) &\geq \RWp(\tilde{\mu}_n,\tilde{\nu}_n) - 2\rho - \Wp(\mu,\hat{\mu}_n) - \Wp(\nu,\hat{\nu}_n)\\
    \Wp(\mu,\nu) &\leq (1-3\eps)^{-\frac{1}{p}} \left( \RWp(\tilde{\mu}_n,\tilde{\nu}_n) + 2\rho + \Wp(\mu,\hat{\mu}_n) + \Wp(\nu,\hat{\nu}_n) \right)
\end{align*}
Since $\mu,\nu \in \cP_p(\cX)$ and $\eps \leq 1/4$, Theorem 7.12 of \citet{villani2003} gives
\begin{align*}
    \limsup_{n \to \infty} \RWp(\tilde{\mu}_n,\tilde{\nu}_n) - 2\rho \leq \Wp(\mu,\nu) \leq \liminf_{n \to \infty} 4\RWp(\tilde{\mu}_n,\tilde{\nu}_n) + 8\rho \quad \text{a.s.}
\end{align*}
Thus, for the two-sample testing application, under the null hypothesis $H_0:\mu = \nu$, we have $\lim_{n \to \infty} \mathds{1} \{ \RWp(\tilde{\mu}_n,\tilde{\nu}_n) > 3\rho\} = 0$ a.s., and, under the alternative $H_1: \Wp(\mu,\nu) > 45 \rho$, we have $\lim_{n \to \infty} \mathds{1} \{ \RWp(\tilde{\mu}_n,\tilde{\nu}_n) > 3\rho\} = 1$ a.s. 

\medskip
For independence testing, we note that $\kappa \in \cP(\cX^2)$ is $(\rho,3\eps)$-resilient w.r.t.\ $\Wp$ on the product space since both of its marginals lie in $\cW_p(\rho,3\eps)$, by our choice of $\bar{\mathsf{d}}$. Thus, we have $\lim_{n \to \infty} \mathds{1} \{ \RWp(\tilde{\kappa}_n,\tilde{\kappa}_{1,n} \otimes \tilde{\kappa}_{2,n}) > 3\rho\} = 0$ a.s., under the null, and $\lim_{n \to \infty} \mathds{1} \{ \RWp(\tilde{\kappa}_n,\tilde{\kappa}_{1,n} \otimes \tilde{\kappa}_{2,n}) > 3\rho\} = 1$ a.s., under the alternative, as desired.\qed

\section{Concluding Remarks}\label{SEC:summary}

To perform robust distribution estimation under $\Wp$, this paper introduced the outlier-robust Wasserstein distance $\RWp$, which measures proximity between probability distributions using POT. By applying MDE under this robust distance, we achieved minimax-optimal population-limit and near-optimal finite-sample risk guarantees. Our analysis relied on a new approximate triangle inequality for POT and %
an equivalence result between mass addition and mass removal formulations of $\RWp$. %
These robust estimation guarantees were complemented by a comprehensive duality theory, mirroring the classical Kantorovich dual up to a regularization term scaling with the sup-norm of the dual potential. This gave rise to an elementary robustification technique for WGAN which enabled generative modeling experiments with contaminated image data.
We also addressed the problem of estimating the Wasserstein distance itself based on contaminated data (as opposed to estimating a distribution under the distance), and showed that $\RWp$ serves as a near-optimal and efficiently computable estimate. 
Moving forward, we hope that our framework of MDE under POT can find broader applications to both the theory and practice of robust statistics.

\subsection*{Acknowledgements}

The authors would like to thank Benjamin Grimmer for helpful conversations surrounding minimax optimization, Jacob Steinhardt and Adam Sealfon for useful discussions on robust statistics, and Jason Gaitonde for advice on high-dimensional probability. S. Nietert was supported by the National Science Foundation (NSF) Graduate Research Fellowship under Grant DGE-1650441. R. Cummings was supported in part by the NSF CAREER award under Grant CNS-1942772, a Mozilla Research Grant, and a JPMorgan Chase Faculty Research Award. Z. Goldfeld was supported in part by the NSF CAREER award under Grant CCF-2046018, an NSF Grant DMS-2210368, and the
IBM Academic Award.

\bibliographystyle{abbrvnat}
\bibliography{references}

\begin{thebibliography}{64}
\providecommand{\natexlab}[1]{#1}
\providecommand{\url}[1]{\texttt{#1}}
\expandafter\ifx\csname urlstyle\endcsname\relax
  \providecommand{\doi}[1]{doi: #1}\else
  \providecommand{\doi}{doi: \begingroup \urlstyle{rm}\Url}\fi

\bibitem[{\'A}lvarez-Esteban et~al.(2011){\'A}lvarez-Esteban, del Barrio, Cuesta-Albertos, and Matr{\'a}n]{alvarez2011uniqueness}
P.~C. {\'A}lvarez-Esteban, E.~del Barrio, J.~A. Cuesta-Albertos, and C.~Matr{\'a}n.
\newblock {Uniqueness and approximate computation of optimal incomplete transportation plans}.
\newblock \emph{Annales de l'Institut Henri Poincaré, Probabilités et Statistiques}, 47\penalty0 (2):\penalty0 358 -- 375, 2011.

\bibitem[Arjovsky et~al.(2017)Arjovsky, Chintala, and Bottou]{arjovsky_wgan_2017}
M.~Arjovsky, S.~Chintala, and L.~Bottou.
\newblock {W}asserstein generative adversarial networks.
\newblock In \emph{International Conference on Machine Learning (ICML)}, 2017.

\bibitem[Balaji et~al.(2020)Balaji, Chellappa, and Feizi]{balaji2020}
Y.~Balaji, R.~Chellappa, and S.~Feizi.
\newblock Robust optimal transport with applications in generative modeling and domain adaptation.
\newblock In \emph{Advances in Neural Information Processing Systems (NeurIPS)}, 2020.

\bibitem[Bassetti et~al.(2006)Bassetti, Bodini, and Regazzini]{bassetti2006kantorovich}
F.~Bassetti, A.~Bodini, and E.~Regazzini.
\newblock On minimum {K}antorovich distance estimators.
\newblock \emph{Statistics \& Probability Letters}, 76\penalty0 (12):\penalty0 1298--1302, 2006.

\bibitem[Bayraktar and Guo(2021)]{bayraktar2021}
E.~Bayraktar and G.~Guo.
\newblock {Strong equivalence between metrics of {W}asserstein type}.
\newblock \emph{Electronic Communications in Probability}, 26:\penalty0 1--13, 2021.

\bibitem[Bernton et~al.(2017)Bernton, Jacob, Gerber, and Robert]{bernton2017inference}
E.~Bernton, P.~E. Jacob, M.~Gerber, and C.~P. Robert.
\newblock Inference in generative models using the {W}asserstein distance.
\newblock \emph{arXiv preprint arXiv:1701.05146}, 1\penalty0 (8):\penalty0 9, 2017.

\bibitem[Blondel et~al.(2018)Blondel, Seguy, and Rolet]{blondel2018smooth}
M.~Blondel, V.~Seguy, and A.~Rolet.
\newblock Smooth and sparse optimal transport.
\newblock In \emph{International conference on artificial intelligence and statistics}, pages 880--889. PMLR, 2018.

\bibitem[Boissard et~al.(2015)Boissard, Gouic, and Loubes]{boissard2015wasserstein}
E.~Boissard, T.~L. Gouic, and J.-M. Loubes.
\newblock {Distribution’s template estimate with Wasserstein metrics}.
\newblock \emph{Bernoulli}, 21\penalty0 (2):\penalty0 740 -- 759, 2015.

\bibitem[Bonnotte(2013)]{bonnotte2013unidimensional}
N.~Bonnotte.
\newblock \emph{Unidimensional and evolution methods for optimal transportation}.
\newblock PhD thesis, Paris 11, 2013.

\bibitem[Caffarelli and McCann(2010)]{caffarelli2010}
L.~A. Caffarelli and R.~J. McCann.
\newblock Free boundaries in optimal transport and {M}onge-{A}mp\`ere obstacle problems.
\newblock \emph{Annals of Mathematics. Second Series}, 171\penalty0 (2):\penalty0 673--730, 2010.

\bibitem[Cao(2017)]{cao2017}
M.~Cao.
\newblock {WGAN-GP}.
\newblock \url{https://github.com/caogang/wgan-gp}, 2017.

\bibitem[Chapel et~al.(2020)Chapel, Alaya, and Gasso]{chapel2020}
L.~Chapel, M.~Z. Alaya, and G.~Gasso.
\newblock Partial optimal tranport with applications on positive-unlabeled learning.
\newblock In \emph{Advances in Neural Information Processing Systems (NeurIPS)}, 2020.

\bibitem[Chen et~al.(2022)Chen, Li, Li, and Meka]{chen2022minimax}
S.~Chen, J.~Li, Y.~Li, and R.~Meka.
\newblock Minimax optimality (probably) doesn't imply distribution learning for {GAN}s.
\newblock In \emph{International Conference on Learning Representations (ICLR)}, 2022.

\bibitem[Cheng et~al.(2019)Cheng, Diakonikolas, and Ge]{cheng2019high}
Y.~Cheng, I.~Diakonikolas, and R.~Ge.
\newblock High-dimensional robust mean estimation in nearly-linear time.
\newblock In \emph{ACM-SIAM Symposium on Discrete Algorithms (SODA)}, 2019.

\bibitem[Chizat et~al.(2018{\natexlab{a}})Chizat, Peyr\'{e}, Schmitzer, and Vialard]{chizat2018}
L.~Chizat, G.~Peyr\'{e}, B.~Schmitzer, and F.-X. Vialard.
\newblock Unbalanced optimal transport: dynamic and {K}antorovich formulations.
\newblock \emph{Journal of Functional Analysis}, 274\penalty0 (11):\penalty0 3090--3123, 2018{\natexlab{a}}.

\bibitem[Chizat et~al.(2018{\natexlab{b}})Chizat, Peyr\'{e}, Schmitzer, and Vialard]{chizat2018scaling}
L.~Chizat, G.~Peyr\'{e}, B.~Schmitzer, and F.-X. Vialard.
\newblock Scaling algorithms for unbalanced optimal transport problems.
\newblock \emph{Mathematics of Computation}, 87\penalty0 (314):\penalty0 2563--2609, 2018{\natexlab{b}}.

\bibitem[Del~Barrio and Matr{\'a}n(2013)]{del2013rates}
E.~Del~Barrio and C.~Matr{\'a}n.
\newblock Rates of convergence for partial mass problems.
\newblock \emph{Probability Theory and related fields}, 155\penalty0 (3):\penalty0 521--542, 2013.

\bibitem[Diakonikolas and Kane(2019)]{diakonikolas2019recent}
I.~Diakonikolas and D.~M. Kane.
\newblock Recent advances in algorithmic high-dimensional robust statistics.
\newblock \emph{arXiv preprint arXiv:1911.05911}, 2019.

\bibitem[Diakonikolas et~al.(2016)Diakonikolas, Kamath, Kane, Li, Moitra, and Stewart]{diakonikolas2016}
I.~Diakonikolas, G.~Kamath, D.~M. Kane, J.~Li, A.~Moitra, and A.~Stewart.
\newblock Robust estimators in high dimensions without the computational intractability.
\newblock In \emph{IEEE Symposium on Foundations of Computer Science (FOCS)}, 2016.

\bibitem[Donoho and Liu(1988)]{donoho88}
D.~L. Donoho and R.~C. Liu.
\newblock {The "Automatic" Robustness of Minimum Distance Functionals}.
\newblock \emph{The Annals of Statistics}, 16\penalty0 (2):\penalty0 552 -- 586, 1988.

\bibitem[Fatras et~al.(2021)Fatras, Sejourne, Flamary, and Courty]{fatras21a}
K.~Fatras, T.~Sejourne, R.~Flamary, and N.~Courty.
\newblock Unbalanced minibatch optimal transport; applications to domain adaptation.
\newblock In \emph{International Conference on Machine Learning (ICML)}, 2021.

\bibitem[Figalli(2010)]{figalli2010}
A.~Figalli.
\newblock The optimal partial transport problem.
\newblock \emph{Archive for Rational Mechanics and Analysis}, 195:\penalty0 533--560, 2010.

\bibitem[Fournier and Guillin(2015)]{fournier2015rate}
N.~Fournier and A.~Guillin.
\newblock On the rate of convergence in wasserstein distance of the empirical measure.
\newblock \emph{Probability theory and related fields}, 162\penalty0 (3):\penalty0 707--738, 2015.

\bibitem[Fukunaga and Kasai(2022)]{fukunaga2021}
T.~Fukunaga and H.~Kasai.
\newblock Block-coordinate {F}rank-{W}olfe algorithm and convergence analysis for semi-relaxed optimal transport problem.
\newblock In \emph{IEEE International Conference on Acoustics, Speech and Signal Processing (ICASSP)}, 2022.

\bibitem[Gulrajani et~al.(2017)Gulrajani, Ahmed, Arjovsky, Dumoulin, and Courville]{gulrajani2017improved}
I.~Gulrajani, F.~Ahmed, M.~Arjovsky, V.~Dumoulin, and A.~C. Courville.
\newblock Improved training of {W}asserstein {GAN}s.
\newblock In \emph{Advances in Neural Information Processing Systems (NeurIPS)}, 2017.

\bibitem[Hanin(1992)]{hanin1992}
L.~G. Hanin.
\newblock Kantorovich-{R}ubinstein norm and its application in the theory of {L}ipschitz spaces.
\newblock \emph{Proceedings of the American Mathematical Society}, 115\penalty0 (2):\penalty0 345--352, 1992.

\bibitem[Heusel et~al.(2017)Heusel, Ramsauer, Unterthiner, Nessler, and Hochreiter]{heusel2017}
M.~Heusel, H.~Ramsauer, T.~Unterthiner, B.~Nessler, and S.~Hochreiter.
\newblock {GANs} trained by a two time-scale update rule converge to a local nash equilibrium.
\newblock In \emph{Advances in Neural Information Processing Systems (NeurIPS)}, 2017.

\bibitem[Hogg et~al.(2005)Hogg, McKean, and Craig]{hogg2005introduction}
R.~Hogg, J.~McKean, and A.~Craig.
\newblock \emph{Introduction to Mathematical Statistics}.
\newblock Pearson Education, 2005.

\bibitem[Hopkins et~al.(2020)Hopkins, Li, and Zhang]{hopkins2020robust}
S.~B. Hopkins, J.~Li, and F.~Zhang.
\newblock Robust and heavy-tailed mean estimation made simple, via regret minimization.
\newblock In \emph{Proceedings of the 34th International Conference on Neural Information Processing Systems}, 2020.

\bibitem[Huber(1964)]{huber64}
P.~J. Huber.
\newblock {Robust Estimation of a Location Parameter}.
\newblock \emph{The Annals of Mathematical Statistics}, 35\penalty0 (1):\penalty0 73--101, 1964.

\bibitem[Karras et~al.(2018)Karras, Aila, Laine, and Lehtinen]{karras2018}
T.~Karras, T.~Aila, S.~Laine, and J.~Lehtinen.
\newblock Progressive growing of {GANs} for improved quality, stability, and variation.
\newblock In \emph{International Conference on Learning Representations (ICLR)}, 2018.

\bibitem[Karras et~al.(2020)Karras, Laine, Aittala, Hellsten, Lehtinen, and Aila]{karras2020}
T.~Karras, S.~Laine, M.~Aittala, J.~Hellsten, J.~Lehtinen, and T.~Aila.
\newblock Analyzing and improving the image quality of {StyleGAN}.
\newblock In \emph{IEEE/CVF Conference on Computer Vision and Pattern Recognition (CVPR)}, 2020.

\bibitem[Krizhevsky et~al.(2009)Krizhevsky, Hinton, et~al.]{krizhevsky2009learning}
A.~Krizhevsky, G.~Hinton, et~al.
\newblock Learning multiple layers of features from tiny images.
\newblock Technical report, University of Toronto, 2009.

\bibitem[Le et~al.(2021)Le, Nguyen, Nguyen, Pham, Bui, and Ho]{le2021}
K.~Le, H.~Nguyen, Q.~M. Nguyen, T.~Pham, H.~Bui, and N.~Ho.
\newblock On robust optimal transport: Computational complexity and barycenter computation.
\newblock In \emph{Advances in Neural Information Processing Systems (NeurIPS)}, 2021.

\bibitem[Lecun et~al.(1998)Lecun, Bottou, Bengio, and Haffner]{lecun1998recognition}
Y.~Lecun, L.~Bottou, Y.~Bengio, and P.~Haffner.
\newblock Gradient-based learning applied to document recognition.
\newblock \emph{Proceedings of the IEEE}, 86\penalty0 (11):\penalty0 2278--2324, 1998.

\bibitem[Lei(2020)]{lei2020convergence}
J.~Lei.
\newblock {Convergence and concentration of empirical measures under Wasserstein distance in unbounded functional spaces}.
\newblock \emph{Bernoulli}, 26\penalty0 (1):\penalty0 767 -- 798, 2020.

\bibitem[Liero et~al.(2018)Liero, Mielke, and Savar\'{e}]{liero2018}
M.~Liero, A.~Mielke, and G.~Savar\'{e}.
\newblock Optimal entropy-transport problems and a new {H}ellinger-{K}antorovich distance between positive measures.
\newblock \emph{Inventiones Mathematicae}, 211\penalty0 (3):\penalty0 969--1117, 2018.

\bibitem[Lin et~al.(2021)Lin, Zheng, Chen, Cuturi, and Jordan]{lin2021projection}
T.~Lin, Z.~Zheng, E.~Chen, M.~Cuturi, and M.~I. Jordan.
\newblock On projection robust optimal transport: Sample complexity and model misspecification.
\newblock In \emph{International Conference on Artificial Intelligence and Statistics (AISTATS)}, 2021.

\bibitem[Ma et~al.(2023)Ma, Liu, La~Vecchia, and Lerasle]{ma2023inference}
Y.~Ma, H.~Liu, D.~La~Vecchia, and M.~Lerasle.
\newblock Inference via robust optimal transportation: theory and methods.
\newblock \emph{arXiv preprint arXiv:2301.06297}, 2023.

\bibitem[Mukherjee et~al.(2021)Mukherjee, Guha, Solomon, Sun, and Yurochkin]{mukherjee2021}
D.~Mukherjee, A.~Guha, J.~Solomon, Y.~Sun, and M.~Yurochkin.
\newblock Outlier-robust optimal transport.
\newblock In \emph{International Conference on Machine Learning (ICML)}, 2021.

\bibitem[Mulholland(1949)]{mulholland1949minkowski}
H.~P. Mulholland.
\newblock On generalizations of minkowski's inequality in the form of a triangle inequality.
\newblock \emph{Proceedings of the London Mathematical Society}, s2-51\penalty0 (1):\penalty0 294--307, 1949.

\bibitem[Nadjahi et~al.(2019)Nadjahi, Durmus, Simsekli, and Badeau]{nadjahi2019asymptotic}
K.~Nadjahi, A.~Durmus, U.~Simsekli, and R.~Badeau.
\newblock Asymptotic guarantees for learning generative models with the sliced-{W}asserstein distance.
\newblock In \emph{Advances in Neural Information Processing Systems (NeurIPS)}, 2019.

\bibitem[Nadjahi et~al.(2020)Nadjahi, Durmus, Chizat, Kolouri, Shahrampour, and Simsekli]{nadjahi2020statistical}
K.~Nadjahi, A.~Durmus, L.~Chizat, S.~Kolouri, S.~Shahrampour, and U.~Simsekli.
\newblock Statistical and topological properties of sliced probability divergences.
\newblock \emph{Advances in Neural Information Processing Systems (NeurIPS)}, 2020.

\bibitem[Nath(2020)]{nath2020}
J.~S. Nath.
\newblock Unbalanced optimal transport using integral probability metric regularization.
\newblock \emph{arXiv preprint arXiv:2011.05001}, 2020.

\bibitem[Nietert et~al.(2022)Nietert, Sadhu, Goldfeld, and Kato]{nietert2022sliced}
S.~Nietert, R.~Sadhu, Z.~Goldfeld, and K.~Kato.
\newblock Statistical, robustness, and computational guarantees for sliced {W}asserstein distances.
\newblock In \emph{Advances in Neural Information Processing Systems (NeurIPS)}, 2022.

\bibitem[Nietert et~al.(2024)Nietert, Goldfeld, and Shafiee]{nietert2024distribution}
S.~Nietert, Z.~Goldfeld, and S.~Shafiee.
\newblock Robust distribution learning with local and global adversarial corruptions (extended abstract).
\newblock In S.~Agrawal and A.~Roth, editors, \emph{Proceedings of Thirty Seventh Conference on Learning Theory (COLT)}, 2024.

\bibitem[Niles-Weed and Berthet(2022)]{nilesweed22minimax}
J.~Niles-Weed and Q.~Berthet.
\newblock Minimax estimation of smooth densities in {W}asserstein distance.
\newblock \emph{The Annals of Statistics}, 50\penalty0 (3):\penalty0 1519--1540, 2022.

\bibitem[Niles-Weed and Rigollet(2022)]{niles2022estimation}
J.~Niles-Weed and P.~Rigollet.
\newblock Estimation of {w}asserstein distances in the spiked transport model.
\newblock \emph{Bernoulli}, 28\penalty0 (4):\penalty0 2663--2688, 2022.

\bibitem[Peyr{\'{e}} and Cuturi(2019)]{peyre2019}
G.~Peyr{\'{e}} and M.~Cuturi.
\newblock Computational optimal transport.
\newblock \emph{Foundations and Trends in Machine Learning}, 11\penalty0 (5-6):\penalty0 355--607, 2019.

\bibitem[Piccoli and Rossi(2014)]{piccoli2014}
B.~Piccoli and F.~Rossi.
\newblock Generalized {W}asserstein distance and its application to transport equations with source.
\newblock \emph{Archive for Rational Mechanics and Analysis}, 211\penalty0 (1):\penalty0 335--358, 2014.

\bibitem[Ronchetti and Huber(2009)]{ronchetti2009robust}
E.~M. Ronchetti and P.~J. Huber.
\newblock \emph{Robust Statistics}.
\newblock John Wiley \& Sons Hoboken, NJ, USA, 2009.

\bibitem[Santambrogio(2015)]{santambrogio2015}
F.~Santambrogio.
\newblock \emph{Optimal Transport for Applied Mathematicians}, volume~87 of \emph{Progress in Nonlinear Differential Equations and their Applications}.
\newblock Birkh\"{a}user/Springer, Cham, 2015.
\newblock Calculus of variations, PDEs, and modeling.

\bibitem[Schmitzer and Wirth(2019)]{schmitzer2019}
B.~Schmitzer and B.~Wirth.
\newblock A framework for {W}asserstein-1-type metrics.
\newblock \emph{Journal of Convex Analysis}, 26\penalty0 (2):\penalty0 353--396, 2019.

\bibitem[Shen and Sanghavi(2019)]{shen19}
Y.~Shen and S.~Sanghavi.
\newblock Learning with bad training data via iterative trimmed loss minimization.
\newblock In \emph{International Conference on Machine Learning (ICML)}, 2019.

\bibitem[Singh and P{\'o}czos(2018)]{singh2018minimax}
S.~Singh and B.~P{\'o}czos.
\newblock Minimax distribution estimation in wasserstein distance.
\newblock \emph{arXiv preprint arXiv:1802.08855}, 2018.

\bibitem[Sonthalia and Gilbert(2020)]{sonthalia2020dual}
R.~Sonthalia and A.~C. Gilbert.
\newblock Dual regularized optimal transport.
\newblock \emph{arXiv preprint arXiv:2012.03126}, 2020.

\bibitem[Staerman et~al.(2021)Staerman, Laforgue, Mozharovskyi, and d'Alch{\'{e}}{-}Buc]{staerman21}
G.~Staerman, P.~Laforgue, P.~Mozharovskyi, and F.~d'Alch{\'{e}}{-}Buc.
\newblock When {OT} meets {MoM:} robust estimation of {W}asserstein distance.
\newblock In \emph{International Conference on Artificial Intelligence and Statistics (AISTATS)}, 2021.

\bibitem[Steinhardt et~al.(2018)Steinhardt, Charikar, and Valiant]{steinhardt2018resilience}
J.~Steinhardt, M.~Charikar, and G.~Valiant.
\newblock Resilience: {A} criterion for learning in the presence of arbitrary outliers.
\newblock In \emph{Innovations in Theoretical Computer Science Conference (ITCS)}, volume~94, 2018.

\bibitem[Tolstikhin et~al.(2018)Tolstikhin, Bousquet, Gelly, and Sch{\"o}lkopf]{tolstikhin2018wasserstein}
I.~Tolstikhin, O.~Bousquet, S.~Gelly, and B.~Sch{\"o}lkopf.
\newblock {W}asserstein auto-encoders.
\newblock In \emph{International Conference on Learning Representations (ICLR)}, 2018.

\bibitem[Varuna~Jayasiri(2020)]{labml}
N.~W. Varuna~Jayasiri.
\newblock labml.ai annotated paper implementations, 2020.
\newblock URL \url{https://nn.labml.ai/}.

\bibitem[Vershynin(2018)]{vershynin2018}
R.~Vershynin.
\newblock \emph{High-Dimensional Probability: An Introduction with Applications in Data Science}.
\newblock Cambridge Series in Statistical and Probabilistic Mathematics. Cambridge University Press, 2018.

\bibitem[Villani(2003)]{villani2003}
C.~Villani.
\newblock \emph{Topics in Optimal Transportation}.
\newblock Graduate Studies in Mathematics. American Mathematical Society, 2003.

\bibitem[Weed and Bach(2019)]{weed2019}
J.~Weed and F.~Bach.
\newblock Sharp asymptotic and finite-sample rates of convergence of empirical measures in {W}asserstein distance.
\newblock \emph{Bernoulli}, 25\penalty0 (4A):\penalty0 2620--2648, 2019.

\bibitem[Zhu et~al.(2022)Zhu, Jiao, and Steinhardt]{zhu2019resilience}
B.~Zhu, J.~Jiao, and J.~Steinhardt.
\newblock {Generalized resilience and robust statistics}.
\newblock \emph{The Annals of Statistics}, 50\penalty0 (4):\penalty0 2256 -- 2283, 2022.

\end{thebibliography}

\begin{appendices}
\crefalias{section}{appendix}
\crefalias{subsection}{appendix}
\section{Asymmetric Robust Distance}
\label{app:asymmetric-results}

Define the asymmetric robust $p$-Wasserstein distance with robustness radii $\eps_1,\eps_2 \in [0,1]$ by
\begin{equation*}
    \Wp^{\eps_1,\eps_2}(\mu,\nu) \defeq \ \inf_{\substack{\mu' \in \cP(\cX),\, \mu' \leq \frac{1}{1-\eps_1}\mu\\
    \nu' \in \cP(\cX),\, \nu' \leq \frac{1}{1-\eps_2}\nu}} \Wp\bigl(\mu',\nu'\bigr),
\end{equation*}
so that $\RWp = (1-\eps)^{1/p}\,\Wp^{\eps,\eps}$ and $\Wp(\mu \| \nu) = \Wp^{\eps,0}(\mu,\nu)$.
Our results for robust distribution estimation still hold for MDE under $\RWp(\cdot\|\cdot)$ if we take care to pass the contaminated empirical measure as the left argument. This extension relies on the following lemma.

\begin{lemma}
\label{lem:one-sided-RWp-bound}
Fix $\mu \in \cW_p(\rho,\eps)$ and $\nu \in \cP(\cX)$. Then, for any $\tilde{\nu} \in \cP(\cX)$ such that $\|\tilde{\nu} - \nu\|_\tv \leq \eps$, we have $\RWp(\tilde{\nu} \| \mu) \leq (1-\eps)^{-1/p}\bigl(\Wp(\nu,\mu) + \rho\bigr)$.
\end{lemma}
\begin{proof}
Write $\nu' \coloneqq \frac{1-\eps}{\nu \land \tilde{\nu}(\cX)}\nu \land \tilde{\nu}$. By design, we have $\nu' \leq \nu$, $\nu' \leq \tilde{\nu}$, and $\nu'(\cX) = 1-\eps$. By Lemma~\ref{lem:OT-TV-commutativity}, there exists $\mu' \in (1-\eps)\cP(\cX)$ with $\mu' \leq \mu$ such that $\Wp(\mu',\nu') \leq \Wp(\mu,\nu)$. We then compute
\begin{align*}
    (1-\eps)^{1/p}\RWp(\tilde{\nu} \| \mu) &\leq \Wp(\nu',(1-\eps)\mu)\\
    &\leq \Wp(\nu',\mu') + \Wp(\mu',(1-\eps)\mu)\\
    &\leq \Wp(\nu,\mu) + \Wp(\mu',(1-\eps)\mu)\\
    &\leq \Wp(\nu,\mu) + (1-\eps)^{1/p}\rho\\
    &\leq \Wp(\nu,\mu) + \rho.\qedhere
\end{align*}
\end{proof}

\noindent Next, given any $\tilde{\mu},\nu \in \cP(\cX)$, it follows from definitions that $\RWp(\tilde{\mu},\nu) \leq \RWp(\tilde{\mu}\|\nu)$, since the feasible set of input pairs for the two-sided problem is larger. Recall now the setup for the proof of \cref{thm:finite-sample-minimax-risk}. We have clean measure $\mu \in \cG \subseteq \cW_p(\rho,2\eps)$, empirical measure $\hat{\mu}_n$, and $\eps$-corrupted measure $\tilde{\mu}_n$. Taking $\mathsf{T}$ as the $\delta$-approximate minimum distance estimator over $\cG$ under $\RWp(\cdot\|\cdot)$, minding the order of arguments, we have
\begin{align*}
    \Wp^{2\eps}\bigl(\mu,\mathsf{T}(\tilde{\mu}_n)\bigr) &\leq \RWp(\mu,\tilde{\mu}_n) + \RWp\bigl(\tilde{\mu}_n,\mathsf{T}(\tilde{\mu}_n)\bigr) \tag{Proposition~\ref{prop:RWp-structural-properties}}\\
    &\leq \RWp(\mu \| \tilde{\mu}_n) + \RWp\bigl(\tilde{\mu}_n,\mathsf{T}(\tilde{\mu}_n)\bigr)\\
    &\leq 2\RWp(\mu\|\tilde{\mu}_n) + \delta \tag{MDE guarantee and $\mu \in \cG$}\\
    &\lesssim \Wp(\mu,\hat{\mu}_n) + \rho + \delta, \tag{Lemma~\ref{lem:one-sided-RWp-bound}}\vspace{-1mm}
\end{align*}
At this point, the remainder of the proof goes through, and final bound still holds up to constant factors. Next, although the proof of Proposition~\ref{prop:alternative-primal-probs} was given for the symmetric distance, an identical argument reveals that the following variant holds in general:\vspace{-1mm}
\begin{equation*}
    \Wp^{\eps_1,\eps_2}(\mu,\nu) = \frac{1}{[(1-\eps_1)(1-\eps_2)]^{1/p}}\inf_{\substack{\mu' \geq (1-\eps_2)\mu, \, \nu' \geq (1-\eps_1)\nu\\ \mu'(\cX) = \nu'(\cX) = 1 - \eps_1\eps_2}} \Wp(\mu',\nu').\vspace{-1mm}
\end{equation*}
We now translate \cref{thm:RWp-dual} to the asymmetric case, following the same argument of the original proof to write $\Wp^{\eps_1,\eps_2}(\mu,\nu)^p$ as\vspace{-1mm}
\begin{align*}
    \sup_{\varphi \in C_b(\cX)} \frac{1}{1-\eps_1} \int_\cX \varphi \dd \mu + \frac{1}{1-\eps_2} \int_\cX \varphi^c \dd \nu + \frac{\eps_2}{1-\eps_2} \inf_x \varphi(x) + \frac{\eps_1}{1-\eps_1} \sup_x \varphi(x).\vspace{-1mm}
\end{align*}
Note that this matches the symmetric case when $\eps_1 = \eps_2$ (up to a multiplicative factor of $(1-\eps)^{-1}$) and matches the one-sided dual used in Section~\ref{sec:experiments} when $\eps_2 = 0$.

\section{Application to Sliced OT}
\label{app:sliced}

Here, we extend the theory developed in the main paper to projection-robust and sliced OT. These variants of the classic Wasserstein distance are obtained by taking either an average or maximum of $\Wp$ between lower-dimensional projections. Throughout, we fix $\cX \subseteq \R^d$ and $k \in \{1,\dots,d\}$. 
Write $\stfl \coloneqq \{ U \in \R^{d \times k} : U^\top U = \I_k \}$ for the Stiefel manifold of orthonormal $k$-frames in $\R^d$ and $\sigma_k \in \cP(\stfl)$ for its Haar measure.
For $U \in \stfl$, let $\mathfrak{p}^U : \R^d \to \R^k$ denote the corresponding projection mapping $\mathfrak{p}^U(x) \coloneqq U^\top x$.  Also let $U_k \coloneqq (\I_k \:\: 0_{d-k,d-k})^\top \in \stfl$ be the
standard projection matrix onto $\R^k$.
We then define the \emph{$k$-dimensional average-sliced Wasserstein distance} between $\mu,\nu \in \cP(\R^d)$ by
\vspace{-1mm}
\begin{align*}
    \SWpk(\mu,\nu) \coloneqq \lft(\int_{\stfl} \Wp\lft(\projU \mu, \projU \nu\rght)^p \dd \sigma_k(U)\rght)^\frac{1}{p},
\vspace{-1mm}
\end{align*}
and the \emph{$k$-dimensional max-sliced Wasserstein distance} between them by
\begin{align*}
    \MWpk(\mu,\nu) \coloneqq \sup_{U \in \stfl} \Wp\lft(\projU \mu, \projU \nu\rght).
\vspace{-1mm}
\end{align*}
In the literature, ``sliced OT'' most typically refers to the setting where $k=1$, while ``projection-robust OT'' refers to our max-sliced definition. Both variants are metrics over $\cP_p(\R^d)$ which generate the same topology as classic $\Wp$ \citep{bonnotte2013unidimensional,nadjahi2019asymptotic,bayraktar2021,nadjahi2020statistical}. Despite the name, formal outlier-robustness guarantees for either sliced distance were not presented until our preliminary work \citep{nietert2022sliced}. We now expand on this work by extending $\RWp$ to the sliced setting, starting with a unified framework for robust distribution estimation under sliced Wasserstein distances. %
Proofs are deferred to Section~\ref{prf:sliced}.
\vspace{-2mm}

\paragraph{Error and risk.}
Fix a statistical distance $\mathsf{D} \in \{\SWpk,\MWpk\}$ and family of clean distributions $\cG \subseteq \cP(\cX)$. Employing the strong $\eps$-corruption model of \cref{sec:robustness}, we seek an estimator $\mathsf{T}$ operating on contaminated samples which minimizes the worst-case $n$-sample risk
\vspace{-1mm}
\begin{equation*}
    R_{n}(\mathsf{T},\mathsf{D},\cG,\eps) \coloneqq \sup_{\mu \in \cG} \sup_{\tilde{\PP}_n \in \cM_n^\mathrm{adv}(\mu,\eps)} \E_{\PP_n}\bigl[\mathsf{D}\bigl(\mathsf{T}(\tilde{\mu}_n),\mu\bigr)\bigr].
\end{equation*}
As before, we write $R_{n}(\mathsf{D},\cG,\eps) \coloneqq \inf_\mathsf{T} R_{n}(\mathsf{T},\mathsf{D},\cG,\eps)$ for the $n$-sample minimax risk.
\vspace{-2mm}

\paragraph{The estimators.}
We next define robust proxies of the average- and max-sliced distances, respectively, as:
\vspace{-2mm}
\begin{align*}
    \RSWpk(\mu,\nu) &\coloneqq \lft(\int_{\stfl} \RWp\lft(\projU \mu, \projU \nu\rght)^p \dd \sigma_k(U)\rght)^\frac{1}{p}\\    
    \RMWpk(\mu,\nu) &\coloneqq \sup_{U \in \stfl} \RWp\lft(\projU \mu, \projU \nu\rght),
\vspace{-1mm}
\end{align*}
substituting $\RWp$ for $\Wp$ in their definitions.
To solve the estimation tasks, we again perform MDE, considering the minimum distance estimates
\begin{align*}
    \mathsf{T}_{[\cG,\RSWpk]}(\tilde{\mu}_n) &\in \argmin_{\nu \in \cG} \RSWpk(\tilde{\mu}_n,\nu)\\
    \mathsf{T}_{[\cG,\RMWpk]}(\tilde{\mu}_n) &\in \argmin_{\nu \in \cG} \RMWpk(\tilde{\mu}_n,\nu).
\end{align*}

\subsection{Robust Estimation Guarantees for Sliced \texorpdfstring{$\bm{\Wp}$}{Wp}}
\label{ssec:robust-sliced-distribution-est}

We start by characterizing population-limit risk. For a distance $\mathsf{D} \in \{\SWpk, \MWpk\}$, a clean family $\cG \subseteq \cP(\cX)$ and a map $\mathsf{T}:\cP(\cX) \to \cP(\cX)$, we again define
\begin{equation*}
    R_{\infty}(\mathsf{T},\mathsf{D}, \cG,\eps) \coloneqq \sup_{\substack{\mu \in \cG, \,\tilde{\mu} \in \cP(\cX)\\\|\tilde{\mu} - \mu\|_\tv \leq \eps}} \mathsf{D}\bigl(\mathsf{T}(\tilde{\mu}),\mu\bigr),
\end{equation*}
with corresponding minimax risk $R_{\infty}(\mathsf{D},\cG,\eps) = \inf_{\mathsf{T}} R_{p,\infty}(\mathsf{T},\mathsf{D},\cG,\eps)$. We can obtain tight bounds for this quantity in terms of the previously considered minimax risks.

\begin{theorem}[Population-limit minimax risk for sliced $\bm{\Wp}$]
\label{thm:sliced-concrete-minimax-risk-bounds}
Fix $0 \leq \eps \leq 0.49$, $q > p$, and $\cG \in \{\cG_q(\sigma),\cG_\mathrm{subG}(\sigma),\cG_\mathrm{cov}(\sigma)\}$ (assuming in the last case that $p < 2$). Letting $\cG_k \coloneqq \{ \mathfrak{p}^{U_k}_\#\mu : \mu \in \cG\}$ denote the orthogonal projection of $\cG$ onto $\R^k$, we have
\vspace{-2mm}
\begin{align*}
    R_\infty(\SWpk,\cG,\eps) &\asymp \sqrt{1 \land \frac{k + p}{d}} \, R(\Wp,\cG,\eps),\\    R_\infty(\MWpk,\cG,\eps) &\asymp R(\Wp,\cG_k ,\eps),
\end{align*}
and these risks are achieved by the estimators and $\mathsf{T}_{[\cG,\RSWpk]}$ and $\mathsf{T}_{[\cG,\RMWpk]}$, respectively (as well as by $\mathsf{T}_{[\cG,\|\cdot\|_\tv]}$ in both cases). 
\end{theorem}

Specializing these bounds to each of the mentioned classes, we have the following~corollary. 

\begin{corollary}[Concrete population-limit risk bounds]
For $0 \leq \eps \leq 0.49$, we have
\begin{align*}
    R_p(\MWpk, \cG,\eps) &\asymp \begin{cases}
        \sigma \eps^{\frac{1}{p}-\frac{1}{q}}, & \text{if $\cG = \cG_q(\sigma)$ for $q \geq p$}\\
        \sigma \sqrt{k + p + \log\left(\frac{1}{\eps}\right)}\,\eps^{\frac 1p}, & \text{if $\cG = \cG_\mathrm{subG}(\sigma)$}\\
        \sigma \sqrt{k}\, \eps^{\frac{1}{p}-\frac{1}{2}}, & \text{if $p < 2$, and $\cG = \cG_\mathrm{cov}(\sigma)$}\\
    \end{cases}\\
    \vspace{5mm}
    R_p(\SWpk, \cG,\eps) &\asymp \begin{cases}
        \sigma \sqrt{1 \land \frac{k + p}{d}} \, \eps^{\frac{1}{p}-\frac{1}{q}}, & \text{if $\cG = \cG_q(\sigma)$ for $q \geq p$}\\
        \sigma \sqrt{\left(1 \land \frac{k + p}{d}\right) \left(d + p + \log\left(\frac{1}{\eps}\right)\right)}\,\eps^{\frac 1p}, & \text{if $\cG = \cG_\mathrm{subG}(\sigma)$}\\
        \sigma \sqrt{k}\, \eps^{\frac{1}{p}-\frac{1}{2}}. & \text{if $p < 2$, and $\cG = \cG_\mathrm{cov}(\sigma)$}\\
    \end{cases}
\end{align*}
\end{corollary}

That is, average-sliced risks are proportional to standard risks with a multiplier of $\sqrt{k/d}$ when $d \gg p$, and max-sliced risks are equal to standard risks in $\R^k$. This extends Theorem 2 of \cite{nietert2022sliced} to general $k$ and a wider selection of clean families. Moreover, the proof is markedly simpler, connecting risks directly to those under classic $\Wp$. As before, our argument extends to a class of distributions with bounded Orlicz norm of a certain type, including those with bounded projected $q$th moments for $q > 2$. For the upper bounds, we show that the modulus of continuity characterizing $\MWpk$ risk equals that of the $\Wp$ in $\R^k$, and that resilience w.r.t.\ $\SWpk$ is bounded by the same term appearing in the proof for $\Wp$ up to a prefactor of $O(\sqrt{1\land (k+p)/d})$.
For the lower bounds, we are able to reuse examples constructed for \cref{thm:concrete-minimax-risk-bounds}.

\smallskip 

We next provide near-tight risk bounds for the classes above in the finite-sample regime.

\begin{theorem}[Finite-sample minimax risk for sliced $\bm{\Wp}$]
\label{thm:sliced-finite-sample-minimax-risk}
Fix $0 \leq \eps \leq 0.49$, $q \geq p$, and $\cG \in \{\cG_q(\sigma),\cG_\mathrm{subG}(\sigma),\cG_\mathrm{cov}(\sigma)\}$ (assuming in the last case that $p < 2$). Then, for $(\mathsf{D},\mathsf{D}^\eps) \in \{(\SWpk,\RSWpk),(\MWpk,\RMWpk)\}$, we have
\begin{equation*}
    R_{\infty}(\mathsf{D},\cG,\eps) + R_{n}(\mathsf{D},\cG,0) \lesssim R_{n}(\mathsf{D},\cG,\eps) \lesssim R_{\infty}(\MWpk,\cG,\eps) + \sup_{\mu \in \cG}\E[\mathsf{D}(\hat{\mu}_n,\mu)],
\vspace{-1mm}
\end{equation*}
and the upper bound is achieved by the minimum distance estimator $\mathsf{T}_{[\cG,\mathsf{D}^\eps]}$.%
\end{theorem}
\vspace{-1mm}

\begin{remark}[Population-limit discrepancy for $\bm{\SWpk}$]
Ideally, the population-limit term in the upper bound would be adapted to $\mathsf{D}$ rather than fixed to the larger max-sliced risk. However, the proof in \cref{prf:sliced-finite-sample-minimax-risk} relies on the max-sliced resilience of $\cG$, and it is not clear whether $\mathsf{T}_{[\cG,\RSWpk]}$ can achieve this lower risk in general. For specific families $\cG$, we may have $R_\infty(\MWpk,\cG,\eps) \asymp R_\infty(\SWpk,\cG,\eps)$, in which case this gap is negligible. For example, when $p = 1$ and $\cG = \cG_\mathrm{cov}(\sigma)$, both risks are $\Theta(\sqrt{k\eps})$.
\end{remark}
\vspace{-1mm}

For $\MWpk$, we match the population-limit risk up to empirical approximation error for the standard classes.
In this case, the gap between the upper and lower bounds is simply the sub-optimality of plug-in estimation under $\MWpk$. For small $k$, we note that the empirical approximation error $\E[\mathsf{D}(\hat{\mu}_n,\mu)]$ can be significantly smaller than that under $\Wp$ \citep{lin2021projection,niles2022estimation} (see also \citealp[Section 3]{nietert2022sliced}). Compared to the analogous results for $k=1$ in \cite{nietert2022sliced}, our results hold for all sample sizes $n$, and our proof is cleaner, following the approach of \cref{thm:finite-sample-minimax-risk} via approximate triangle inequalities and appropriate modulus of continuity bounds. As was the case there, our argument also extends to approximate MDE. 
\vspace{-1mm}

\begin{remark}[Average-sliced improvements]
Note that for the average-sliced Wasserstein metric, there is a gap between the upper and lower bounds in \cref{thm:sliced-finite-sample-minimax-risk}. In particular, the upper bound does not coincide with the population risk as $n \to \infty$. It remains unclear whether the finite-sample risk bound for MDE under $\RSWpk$ (or an alternative robust distance) can be improved to $   R_\infty(\SWpk,\cG,\eps)+\sup_{\mu \in \cG}\E[\mathsf{D}(\hat{\mu}_n,\mu)]$, so that the first term matches the population-limit risk from \cref{thm:sliced-concrete-minimax-risk-bounds} and the second is the current empirical convergence term from \cref{thm:sliced-finite-sample-minimax-risk}. We remark that Proposition 2 in \cite{nietert2022sliced} bridges this gap for $k=1$ via MDE under the TV norm, but the resulting bound contains a non-standard truncated empirical convergence term and only holds for sufficiently large sample sizes.
\end{remark}

\vspace{-1mm}
Next, we consider direct estimation of the sliced distances via their robust proxies.

\vspace{-1mm}
\begin{theorem}[Finite-sample robust estimation of sliced $\bm{\Wp}$]
\label{thm:robust-sliced-distance-estimation}
Let $0 \leq \eps < 1/3$, write $\tau = 1 - (1-3\eps)^{1/p} \in [3\eps/p,3\eps]$, and fix $(\mathsf{D},\mathsf{D}^\eps) \in \{(\SWpk,\RSWpk),(\MWpk,\RMWpk)\}$.
Let $\mu,\nu \in \cP(\cX)$ be $(\rho,3\eps)$-resilient w.r.t. $\MWpk$ with empirical measures $\hat{\mu}_n$ and $\hat{\nu}_n$, respectively. Then for any $\tilde{\mu}_n,\tilde{\nu}_n \in \cP(\cX)$ such that $\|\tilde{\mu}_n - \hat{\mu}_n\|_\tv, \|\tilde{\nu}_n - \hat{\nu}_n\|_\tv \leq \eps$, we have
\begin{equation*}
\big|\mathsf{D}^\eps(\tilde{\mu}_n,\tilde{\nu}_n) - \mathsf{D}(\mu,\nu)\big| \lesssim \rho + \tau \mathsf{D}(\mu,\nu) + \mathsf{D}(\mu,\hat{\mu}_n) + \mathsf{D}(\nu,\hat{\nu}_n).
\end{equation*}
In particular, if $\eps \leq 0.33$ and $\mu,\nu \in \cG$ for $\cG \in \{\cG_q(\sigma),\cG_\mathrm{subG}(\sigma),\cG_\mathrm{cov}(\sigma)\}$, $q > p$, then $\rho$ can be replaced with $R_\infty(\mathsf{D},\cG,\eps)$ in the bound above.
\end{theorem}

As with \cref{thm:robust-distance-estimation}, the $\RWp$ estimate suffers from additive estimation error $O(\rho)$ and multiplicative estimation error $O(\tau)$, matching the resilience-based guarantees of MDE when $\Wp(\mu,\nu)$ is sufficiently small and $n$ is sufficiently large.
\vspace{-1mm}

\begin{remark}[Dual forms for robust sliced $\bm{\Wp}$]
For $\eps \in (0,1]$ and $\mu,\nu \in \cP(\cX)$, plugging in the $\RWp$ dual into the definitions of the robust sliced distances gives
\begin{align*}
\vspace{-1mm}
\RSWpk(\mu,\nu)^p &= \int_{\stfl} \lft(
\sup_{\substack{f \in C_b(\R^k)}} \int_{\R^k} f \dd \projU \mu + \int_{\R^k} f^c \dd \projU \nu -  2 \eps \|f\|_\infty \rght) \dd \sigma_k(U)\\    
    \RMWpk(\mu,\nu)^p &=
\sup_{\substack{U \in \stfl \\ f \in C_b(\R^k)}} \int_{\R^k} f \dd \projU \mu + \int_{\R^k} f^c \dd \projU \nu -  2 \eps \|f\|_\infty.
\end{align*}
\vspace{-1mm}
\end{remark}

\section{Proofs for \cref{app:sliced}}
\label{prf:sliced}

\subsection{Proof of \cref{thm:sliced-concrete-minimax-risk-bounds}}

We begin with some preliminary facts. For $x \in \R^d$, let $x_{1:k} \in \R^k$ denote its first $k$ components.

\begin{fact}[Concentration of Gaussian norm]
\label{fact:gaussian-norm-concentration}
For $Z \sim \cN(0,\I_d)$ and $t \geq 0$, we have\vspace{-1mm}
\begin{equation*}
    \PP\left(\left|\|Z\| - \sqrt{d}\right| > t\right) \leq 2e^{-\frac{t^2}{8}}.\vspace{-1mm}
\end{equation*}
\end{fact}

\begin{lemma}[Moments of Gaussian norm]
\label{lem:partial-gaussian-moment}
For $Z \sim \cN(0,\I_d)$ and $q \geq 1$, we have\vspace{-1mm}
\begin{equation*}
    \E\left[\|Z\|^q\right]^\frac{1}{q} = \frac{\sqrt{2}\,\Gamma\left(\frac{d+q}{2}\right)^{\frac{1}{q}}}{\Gamma\left(\frac{d}{2}\right)^{\frac{1}{q}}} \asymp \sqrt{d + q}.\vspace{-1mm}
\end{equation*}
\end{lemma}
\begin{proof}
For the first equality, we use that the expectation is the $(q/2)$-th moment of a $\chi^2$ random variable with $d$ degrees of freedom (see, e.g., Section 3.3.1 of \citealp{hogg2005introduction}). By Stirling's approximation%
, we then have\vspace{-1mm}
\begin{equation*}
    \E\left[\|Z\|^q\right]^\frac{1}{q} \asymp \frac{\left(\frac{d+q}{2}\right)^{\frac{d+q-1}{2q}} e^{-\frac{d+q}{2q}}}{\left(\frac{d}{2}\right)^{\frac{d-1}{2q}} e^{-\frac{d}{2q}}} \asymp \left(\frac{d+q}{d}\right)^{\frac{d-1}{2q}} \sqrt{d+q} \asymp \sqrt{d+q}.\qedhere
\end{equation*}
\end{proof}

\begin{lemma}
\label{lem:partial-sphere-moment}
For $X \sim \Unif(\unitsph)$ and $q \geq 1$, we have $\E\left[\|X_{1:k}\|^q\right]^{1/q} \asymp \sqrt{1 \land \frac{k + q}{d}}$.
\end{lemma}
\begin{proof}
We start with the upper bound. Letting $Z = (Z_1, \dots, Z_d) \sim \cN(0,\I_d)$, we compute
\begin{align*}
    \E\left[\|X_{1:k}\|^q\right] &= \E\left[\frac{\|Z_{1:k}\|^{q}}{\|Z\|^{q}}\,\right]\\
    &= \int_0^1 \PP\left(\frac{\|Z_{1:k}\|^{q}}{\|Z\|^{q}} \geq t\right) \dd t \tag{$\|Z_{1:k}\| \leq 1$}\\
    &\leq \int_0^\infty \PP\left(\|Z_{1:k}\|^{q} \geq t\left(\!\frac{\sqrt{d}}{2}\right)^{q}\,\right) \dd t + \PP\left(\|Z\| \leq \frac{\sqrt{d}}{2}\,\right) \tag{union bound}\\
    &= 2^q d^{-\frac{q}{2}} \E\left[\|Z_{1:k}\|^q\,\right] + \PP\left(\|Z\| \leq \frac{\sqrt{d}}{2}\,\right)\\
    &\leq 2^q d^{-\frac{q}{2}} \E\left[\|Z_{1:k}\|^q\,\right] + 2e^{-\frac{d}{32}}. \tag{\cref{fact:gaussian-norm-concentration}}
\end{align*}
Assume now that $d \geq 100q$ (otherwise, the desired upper bound is trivial). 
Applying Lemma~\ref{lem:partial-gaussian-moment} and noting that $\exp(-u/32) \leq u^{-1/2}$ for $u \geq 100$, we further obtain
\begin{align*}
    \E\left[\|X_{1:k}\|^q\right]^\frac{1}{q} &\lesssim \sqrt{\frac{k+q}{d}} + e^{-\frac{d}{32q}} \lesssim \sqrt{\frac{k+q}{d}} + \sqrt{q/d} \asymp \sqrt{\frac{k+q}{d}},
\end{align*}
matching the stated bound once we cap the quantity at 1. For the lower bound, we compute
\begin{align*}
     \E\left[\|X_{1:k}\|^q\right]^\frac{1}{q} \geq \E\left[X_1^q\right]^\frac{1}{q} \asymp \sqrt{1 \land \frac{q}{d}},
\end{align*}
where the final inequality follows by Lemma 6 of \cite{nietert2022sliced}. This matches the stated bound when $k = O(1)$. For $Z$ as above, \cref{fact:gaussian-norm-concentration}, gives that $\|Z\| > 2\sqrt{d}$ or $\|Z_{1:k}\| < \sqrt{k}/2$ with probability at most $2e^{-d/8} + 2e^{-k/32} \leq 4e^{-k/32}$. Hence, for $k \geq 100$, we have
\begin{align*}
    \E\left[\|X_{1:k}\|^q\right] = \E\left[\frac{\|Z_{1:k}\|^{q}}{\|Z\|^{q}}\,\right] \gtrsim \left(1-4e^{-\frac{k}{32}}\right) \left(\frac{k}{d}\right)^\frac{q}{2} \gtrsim \left(\frac{k}{d}\right)^\frac{q}{2}.
\end{align*}
Combining the previous two inequalities shows that the upper bound is tight.
\end{proof}

We now return to the proof of \cref{thm:sliced-concrete-minimax-risk-bounds}.

\paragraph{Risk bounds for $\MWpk$:} Fix $\eps$, $q$, $\cG$, and $\cG_k$ as in the theorem statement. For any $\mu,\nu \in \cG$ with $\|\mu - \nu\|_\tv \leq \eps$ and any $U \in \stfl$, observe that $\projU \mu, \projU \nu \in \cG_k$ with $\|\projU \mu - \projU \nu\|_\tv \leq \eps$. Moreover, this embedding is surjective, since $\cG_k$ embeds naturally into $\cG$ for all of the considered families. We have
\begin{align*}
    \mathfrak{m}_{\,\MWpk}(\cG,\|\cdot\|_\tv,\eps) %
    = \mspace{-7mu}\sup_{\substack{\mu,\nu \in \cG, \,U \in \stfl \\ \|\mu - \nu\|_\tv \leq \eps}} \mspace{-7mu}\Wp(\projU \mu,\projU \nu)= \mspace{-7mu}\sup_{\substack{\mu,\nu \in \cG_k \\ \|\mu - \nu\|_\tv \leq \eps}} \mspace{-7mu}\Wp(\mu,\nu)= \mathfrak{m}_{\,\Wp}(\cG_k,\|\cdot\|_\tv,\eps).
\end{align*}
Consequently, $\MWpk$ inherits the upper and lower risk bounds from \cref{sec:robustness} for $\Wp$ in $\R^k$.

\paragraph{Risk bounds for $\SWpk$:} For the average-sliced distance, take $X \sim \mu \in \cP(\cX)$ and $U = (U_1,\dots,U_d)^\top \sim \sigma_k$ to be independent, and fix $x_0 \in \cX$. We compute
\begin{align}
    \SWpk(\mu,\delta_{x_0})^p = \E\left[\Wp\lft(\projU \mu, \delta_{x_0}\rght)^p\right]
    = \E\left[\|U^\top (X - x_0)\|^p \right]
    = \E\left[\|U_1\|^p\right] \, \Wp(\mu,\delta_{x_0})^p,\label{eq:SWpk-with-point-mass}
\end{align}
where the third equality follows by rotational symmetry.

Since the upper risk bounds in \cref{sec:robustness} all followed from Lemma~\ref{lem:Wp-resilience-from-mean-resilience}, we now mimic that proof to bound generalized resilience w.r.t.\ $\SWpk$. Let $\eps > 0$ and fix any $\nu \leq \frac{1}{1-\eps}\mu$. Writing $\mu = (1-\eps)\nu + \eps \alpha$ for the appropriate $\alpha \in \cP(\cX)$ and taking $\tau \coloneqq \eps \lor (1-\eps)$, we have
\begin{align*}
    \SWpk(\mu,\nu) %
    \leq \SWpk(\eps\alpha, \eps\nu)
    = \eps^\frac{1}{p} \SWpk(\alpha,\nu)\leq %
    \E\left[\|U_1\|^p\right]^\frac{1}{p} \, \eps^\frac{1}{p} \bigl(\Wp(\alpha,\delta_{x_0}) + \Wp(\delta_{x_0},\nu)\bigr).
\end{align*}
Here, the properties of $\SWpk$ applied for the first four lines are inherited from $\Wp$. The final expression matches that in the proof of Lemma~\ref{lem:Wp-resilience-from-mean-resilience} up to the initial factor. Consequently, all of the risk bounds from \cref{sec:robustness} translate to the average-sliced setting up this factor.

For the lower bound, we observe that all of the lower risk bounds in \cref{sec:robustness} are of the form $\Wp\bigl(\delta_0, (1-\eps)\delta_0 + \eps \kappa\bigr) = \eps^{1/p} \,\Wp(\delta_0, \kappa)$ for some $\kappa \in \cP(\cX)$. Translating these to our setting, we obtain
\begin{equation*}
    \SWpk\bigl(\delta_0, (1-\eps)\delta_0 + \eps \kappa\bigr) = \eps^\frac{1}{p} \SWpk(\delta_0,\kappa) = \E\left[\|U_1\|^p\right]^\frac{1}{p} \eps^\frac{1}{p} \Wp(\delta_0,\kappa),
\end{equation*}
matching the previous bounds exactly up to the same prefactor appearing in the upper bounds. Noting that $U_1$ is equal in distribution to the first $k$ components of a vector sampled uniformly at random from $\unitsph$, an application of Lemma~\ref{lem:partial-sphere-moment} completes the proof.\qed

\subsection{Proof of \cref{thm:sliced-finite-sample-minimax-risk}}
\label{prf:sliced-finite-sample-minimax-risk}

Write $\underline{\cW}_{p,k}(\rho,\eps)$ and $\overline{\cW}_{p,k}(\rho,\eps)$ for the families of distributions which are $(\rho,\eps)$-resilient w.r.t.\ $\SWpk$ and $\MWpk$, respectively. We start with a simple characterization of $\overline{\cW}_{p,k}(\rho,\eps)$.

\begin{lemma}
\label{lem:MWpk-resilience}
We have $\mu \in \overline{\cW}_{p,k}(\rho,\eps)$ if and only if $\projU \mu \in \cW_p(\rho,\eps)$ for all $U \in \stfl$.
\end{lemma}
\begin{proof}
Since suprema over fixed sets commute, we have\vspace{-1mm}
\begin{equation*}
    \sup_{\substack{\nu \in \cP(\cX)\\ \nu \leq \frac{1}{1-\eps}\mu}} \MWpk(\mu,\nu) = \sup_{U \in \stfl} \sup_{\substack{\nu \in \cP(\cX)\\ \nu \leq \frac{1}{1-\eps}\mu}} \Wp(\projU \mu,\projU \nu).\qedhere
\end{equation*}
\end{proof}

Next, we translate the approximate triangle inequality to the sliced regime.

\begin{lemma}[Approximate triangle inequality for sliced $\Wp$]
\label{lem:sliced-Wp-triangle-inequality}
Let $\mu,\nu,\kappa \in \cP(\cX)$ and $0 \leq \eps_1,\eps_2 \leq 1$. Then we have\vspace{-1mm}
\begin{align*}
    \SWpk^{\eps_1 + \eps_2}(\mu,\nu) &\leq \SWpk^{\eps_1}(\mu,\kappa) + \SWpk^{\eps_2}(\kappa,\nu) \text{ and } \: \MWpk^{\eps_1 + \eps_2}(\mu,\nu) \leq \MWpk^{\eps_1}(\mu,\kappa) + \MWpk^{\eps_2}(\kappa,\nu).\vspace{-1mm}
\end{align*}
\end{lemma}
\begin{proof}
For the average-sliced distance, Proposition~\ref{prop:RWp-structural-properties} and the $L^p(\sigma_k)$ triangle inequality give\vspace{-1mm}
\begin{align*}
    \SWpk^{\eps_1 + \eps_2}(\mu,\nu) &= \left(\int_{\stfl} \Wp^{\eps_1 + \eps_2}\left(\projU \mu, \projU \nu\right)^p \dd \sigma_k(U)\right)^\frac{1}{p}\\
    &\leq \left(\int_{\stfl} \left[\Wp^{\eps_1}\left(\projU \mu, \projU \nu\right) + \Wp^{\eps_2}\left(\projU \mu, \projU \nu\right)\right]^p \dd \sigma_k(U)\right)^\frac{1}{p}\\
    &\leq \SWpk^{\eps_1}(\mu,\nu) + \SWpk^{\eps_2}(\mu,\nu).\vspace{-1mm}
\end{align*}
For the max-sliced distance, we again apply Proposition~\ref{prop:RWp-structural-properties} to bound\vspace{-1mm}
\begin{align*}
     \MWpk^{\eps_1 + \eps_2}(\mu,\nu) &= \sup_{U \in \stfl} \Wp^{\eps_1 + \eps_2}\left(\projU \mu, \projU \nu\right)\\
     &\leq \sup_{U \in \stfl} \left[\Wp^{\eps_1}\left(\projU \mu, \projU \nu\right) + \Wp^{\eps_2}\left(\projU \mu, \projU \nu\right)\right]\\
     &\leq \MWpk^{\eps_1}(\mu,\nu) + \MWpk^{\eps_2}(\mu,\nu).\qedhere
\end{align*}
\end{proof}

Next, we bound the appropriate modulus of continuity, mirroring Lemma~\ref{lem:RWp-modulus}. 

\begin{lemma}[Sliced $\Wp$ modulus of continuity]
\label{lem:sliced-Wp-modulus}
For $0 \leq \eps \leq 0.99$ and $\lambda \geq 0$, we have\vspace{-1mm}
\begin{align*}
    \sup_{\substack{\mu,\nu \in \overline{\cW}_{p,k}(\rho,\eps) \\ \RSWpk(\mu,\nu) \leq \lambda}} \SWpk(\mu,\nu) \lesssim \lambda + \rho \qquad \mbox{and}\qquad\sup_{\substack{\mu,\nu \in \overline{\cW}_{p,k}(\rho,\eps) \\ \RMWpk(\mu,\nu) \leq \lambda}} \MWpk(\mu,\nu) \lesssim \lambda + \rho.\vspace{-1mm}
\end{align*}
\end{lemma}

Note that the second suprema has $\mu$ and $\nu$ range over $\overline{\cW}_{p,k}(\rho,\eps)$ rather than  $\underline{\cW}_{p,k}(\rho,\eps)$. This leads to the sub-optimality of the average-sliced finite-sample risk bound we obtained for MDE under $\RSWpk$. It is unavoidable with our current analysis due to our reliance on Lemma~\ref{lem:MWpk-resilience}, which pertains only to max-sliced resilience.
\vspace{1mm}

\begin{proof}
For the average-sliced distance, fix $\mu,\nu \in \overline{\cW}_{p,k}(\rho,\eps)$, and, for any $U \in \stfl$, let $\mu_U,\nu_U \in (1-\eps)\cP(\R^k)$ with $\mu_U \leq \projU \mu$ and $\nu_U \leq \projU \nu$ be minimizers for the $\RWp(\projU \mu,\projU \nu)$ problem. We then bound
\begin{align*}
    &\SWpk(\mu,\nu)^p =\int_{\stfl} \Wp\left(\projU \mu, \projU \nu\right)^p\dd \sigma_k(U)\\
    &\leq \int_{\stfl} \left[\Wp\left(\projU \mu, \frac{1}{1-\eps}\mu_U\!\right)\! + \Wp\left(\frac{1}{1-\eps}\mu_U, \frac{1}{1-\eps}\nu_U\!\right) \!+ \Wp\left(\frac{1}{1-\eps}\nu_U,\projU \nu\right)\right]^p \!\!\dd \sigma_k(U)\!\\
    &\leq \int_{\stfl} \left[\Wp\left(\frac{1}{1-\eps}\mu_U, \frac{1}{1-\eps}\nu_U\right) + 2\rho \right]^p \dd \sigma_k(U) \tag*{\text{(Lemma~\ref{lem:MWpk-resilience})}}\\
    &= \left(\frac{1}{1-\eps}\right) \int_{\stfl} \left[\RWp(\projU \mu, \projU \nu) + 2\rho \right]^p \dd \sigma_k(U)\\
    &\lesssim \left(\RSWpk(\mu,\nu) + 2\rho\right)^p. \tag*{\text{($L^p(\sigma_k)$ triangle inequality)}}
\end{align*}

For the max-sliced distance, let $\mu_U,\nu_U \in (1-\eps)\cP(\R^k)$ with $\mu_U \leq \projU \mu$ and $\nu_U \leq \projU \nu$ be minimizers for the $\RWp(\projU \mu,\projU \nu)$ problem, for any $U \in \stfl$. We then bound
\begin{align*}
    \MWpk(\mu,\nu) &= \sup_{U \in \stfl} \Wp\left(\projU \mu, \projU \nu\right)\\
    &\leq \sup_{U \in \stfl} \Wp\left(\projU \mu, \frac{1}{1-\eps} \mu_U\right) + \Wp\left(\frac{1}{1-\eps}\mu_U,\frac{1}{1-\eps}\nu_U\right) + \Wp\left(\frac{1}{1-\eps}\nu_U, \projU \nu\right)\\
    &\leq \sup_{U \in \stfl} 
    \Wp\left(\frac{1}{1-\eps}\mu_U,\frac{1}{1-\eps}\nu_U\right) + 2 \rho \tag*{(Lemma~\ref{lem:MWpk-resilience})}\\
    &= \left(\frac{1}{1-\eps}\right)^{\!\frac{1}{p}} \RMWpk(\mu,\nu) + 2 \rho\\
    &\lesssim \RMWpk(\mu,\nu) + \rho.
\end{align*}
Together, these two inequalities imply the lemma.
\end{proof}

Finally, we replicate the proof of \cref{thm:finite-sample-minimax-risk}. For the lower bound, the exact same argument applies up to the substitution of $\Wp$ with $\mathsf{D} \in \{\SWpk,\MWpk\}$. To prove the upper bound, let $\cG \subseteq \overline{\cW}_{p,k}(\rho,2\eps)$,
fix $\mu \in \cG$ with empirical measure $\hat{\mu}_n$, and consider any distribution $\tilde{\mu}_n$ with $\|\tilde{\mu}_n - \hat{\mu}_n\|_\tv \leq \eps$. Then, for $\mathsf{T} = \mathsf{T}_{[\cG,\mathsf{D}^\eps]}^\delta$, we have
\begin{align*}
    \mathsf{D}^{2\eps}(\mu,\mathsf{T}(\tilde{\mu}_n)) &\leq \mathsf{D}^\eps(\mu,\tilde{\mu}_n) + \mathsf{D}^\eps\bigl(\tilde{\mu}_n,\mathsf{T}(\tilde{\mu}_n)\bigr) \tag{Lemma~\ref{lem:sliced-Wp-triangle-inequality}}\\
    &\leq 2\mathsf{D}^\eps(\mu,\tilde{\mu}_n) + \delta \tag{MDE guarantee and $\mu \in \cG$}\\
    &\leq 2\mathsf{D}^\eps(\tilde{\mu}_n,\hat{\mu}_n) + 2\mathsf{D}(\hat{\mu}_n,\mu) + \delta. \tag{Lemma~\ref{lem:sliced-Wp-triangle-inequality}}\\
    &\leq 2\mathsf{D}(\hat{\mu}_n,\mu) + \delta. \tag{$\|\tilde{\mu}_n - \hat{\mu}_n\|_\tv \leq \eps$}
\end{align*}
Writing $\lambda_n = 20\E[\mathsf{D}(\hat{\mu}_n,\mu)] + \delta$, Markov's inequality and Lemma~\ref{lem:sliced-Wp-modulus} yield
\begin{equation*}
    \mathsf{D}\bigl(\mu,\mathsf{T}(\tilde{\mu}_n)\bigr) %
    \leq  \sup_{\substack{\mu,\nu \in \cG \\ \mathsf{D}^{2\eps}(\mu,\nu) \leq \lambda_n}} \Wp(\mu,\nu) \lesssim \rho + \lambda_n
\end{equation*}
with probability at least $9/10$. Thus, $R_{p,n}(\mathsf{D},\cG,\eps) \lesssim \rho + \delta + \E[\Wp(\hat{\mu}_n,\mu)]$, as desired. Plugging in the appropriate $\rho$ for the families of interest gives the upper bounds of the theorem when $\delta = 0$ (we present the argument for $\delta > 0$ to emphasize that exact optimization is not necessary).\qed

\subsection{Proof of \cref{thm:robust-sliced-distance-estimation}}

The proof mimcs that of  \cref{thm:robust-distance-estimation}. We assume a Huber contamination model, i.e., when $\tilde{\mu}_n = (1-\eps)\hat{\mu}_n + \eps \alpha$ and $\tilde{\nu}_n = (1-\eps)\hat{\nu}_n + \eps \beta$ for some $\alpha,\beta \in \cP(\cX)$ (the extension to the full setting is identical to that for \cref{thm:robust-distance-estimation}).
More concisely, we write the Huber condition as $\hat{\mu}_n \leq \frac{1}{1-\eps}\tilde{\mu}_n, \hat{\nu}_n \leq \frac{1}{1-\eps}\tilde{\nu}_n$. 
We now have\newpage\noindent
\begin{align*}
    \RMWpk(\tilde{\mu}_n,\tilde{\nu}_n) &= \sup_{U \in \stfl} \RWp(\projU \tilde{\mu}_n, \projU \tilde{\nu}_n)\\
    &= (1-\eps)^{\frac{1}{p}}\sup_{U \in \stfl} \inf_{\substack{\mu',\nu' \in \cP(\cX)\\ \mu' \leq \frac{1}{1-\eps}\projU\tilde{\mu}_n,\, \nu' \leq \frac{1}{1-\eps}\projU\tilde{\nu}_n}} \Wp(\mu',\nu')\\
    &< \sup_{U \in \stfl} \Wp(\projU \hat{\mu}_n,\projU \hat{\nu}_n)\\
    &= \MWpk(\hat{\mu}_n,\hat{\nu}_n)\\
    &\leq \MWpk(\mu,\nu) + \MWpk(\mu,\hat{\mu}_n) + \MWpk(\nu,\hat{\nu}_n).
\end{align*}
The same argument holds when $\mathsf{D} = \RSWpk$. 

\medskip

In the other direction, Lemma~\ref{lem:sliced-Wp-triangle-inequality} gives
\begin{align*}
    \mathsf{D}^{3\eps}(\mu,\nu) &\leq \mathsf{D}(\mu,\hat{\mu}_n) + \mathsf{D}^\eps(\hat{\mu}_n,\tilde{\mu}_n) + \mathsf{D}^\eps(\tilde{\mu}_n,\tilde{\nu}_n) + \mathsf{D}^\eps(\tilde{\nu}_n,\hat{\nu}_n) + \mathsf{D}(\hat{\nu}_n,\nu)\\
    &= \mathsf{D}(\mu,\hat{\mu}_n) + \mathsf{D}^\eps(\tilde{\mu}_n,\tilde{\nu}_n) + \mathsf{D}(\hat{\nu}_n,\nu).
\end{align*}
Moreover, the resilience of $\mu$ and $\nu$ w.r.t.\ $\MWpk$ implies that
\begin{align*}
    \MWpk^{3\eps}(\mu,\nu) &=(1-3\eps)^{\frac{1}{p}}\sup_{U \in \stfl} \inf_{\substack{\mu',\nu' \in \cP(\cX)\\ \mu' \leq \frac{1}{1-3\eps}\projU\mu,\, \nu' \leq \frac{1}{1-3\eps}\projU\nu}} \Wp(\mu',\nu')\\
    &\geq(1-3\eps)^{\frac{1}{p}}\sup_{U \in \stfl} \left[\Wp(\projU\mu,\projU\nu) - 2\rho\right]\\
    &= (1-3\eps)^{\frac{1}{p}} \left(\MWpk(\mu,\nu) - 2\rho \right)\\
    &\geq (1-3\eps)^{\frac{1}{p}} \MWpk(\mu,\nu) - 2\rho.
\end{align*}
The same argument holds when $\mathsf{D} = \RSWpk$. Combining the pieces, we obtain
\begin{equation*}
    |\mathsf{D}^\eps(\tilde{\mu}_n,\tilde{\nu}_n) - \mathsf{D}(\mu,\nu)| \leq 2\rho + \tau\mathsf{D}(\mu,\nu) + \mathsf{D}(\mu,\hat{\mu}_n) + \mathsf{D}(\nu,\hat{\nu}_n). \qed
\end{equation*}

\section{Full Experiment Details}
\label{app:experiment-details}
Full code is available on GitHub at \url{https://github.com/sbnietert/robust-OT}. We stress that we never had to adjust hyper-parameters or loss computations from their defaults, only adding our robust objective modification and procedures to corrupt the original datasets (MNIST \citep{lecun1998recognition}, CelebA-HQ \citep{karras2018}, and CIFAR-10 \citep{krizhevsky2009learning}). The implementation of WGAN-GP used for the robust GAN experiments in the main text was based on a standard PyTorch implementation \citep{cao2017}, as was that of StyleGAN 2 \citep{labml}. The images presented in \cref{fig:generated-samples} were generated without any manual filtering after a predetermined number of batches (125k, batch size 64, for WGAN-GP; 100k, batch size 32, for StyleGAN 2). Training for the WGAN-GP and StyleGAN 2 experiments took 5 hours and 20 hours of compute, respectively, on a cluster machine equipped with a NVIDIA Tesla V100 and 14 CPU cores. 

For the comparison to \citet{balaji2020} and \citet{staerman21}, we used the 50k CIFAR-10 training set of $32\times32$ images contaminated with 2632 images of uniform random noise. Each GAN was trained for 100k batches of 128 images, taking approximately 12 hours of compute on the same cluster machine. For the robust GAN of \citet{balaji2020}, we used their continuous weighting scheme, recommended robustness strength, and DCGAN architecture, along with hyperparameters suggested in their appendix.
\vspace{-1mm}

\end{appendices}

\end{document}